\documentclass{article} % For LaTeX2e
\usepackage{iclr2024_conference,times}

\usepackage{amsmath,amsfonts,bm}

\def\eqref#1{equation~\ref{#1}}
\def\1{\bm{1}}

\DeclareMathAlphabet{\mathsfit}{\encodingdefault}{\sfdefault}{m}{sl}
\SetMathAlphabet{\mathsfit}{bold}{\encodingdefault}{\sfdefault}{bx}{n}

\usepackage{hyperref}
\usepackage{url}

\usepackage[utf8]{inputenc} % allow utf-8 input
\usepackage[T1]{fontenc}    % use 8-bit T1 fonts
\usepackage{hyperref}       % hyperlinks
\usepackage{url}            % simple URL typesetting
\usepackage{booktabs}       % professional-quality tables
\usepackage{amsfonts}       % blackboard math symbols
\usepackage{nicefrac}       % compact symbols for 1/2, etc.
\usepackage{microtype}      % microtypography
\usepackage{xcolor}         % colors

\usepackage{xcolor, colortbl}
\definecolor{asparagus}{rgb}{0.0, 0.0, 0.0}

\title{Beyond Spatio-Temporal Representations: Evolving Fourier Transform for Temporal Graphs}

\author{Anson Bastos$^{1,2}$, \textbf{Kuldeep Singh$^{4}$, Abhishek Nadgeri$^{3}$, Manish Singh$^{2}$, Toyotaro Suzumura$^{5}$}\\
$^{1}$HERE Technologies, India $^{2}$Indian Institute of Technology Hyderabad, India\\ 
$^{3}$RWTH Aachen, Germany $^{4}$Cerence Gmbh, Germany $^{5}$The University of Tokyo, Japan\\
\texttt{ansonbastos@gmail.com}, ~\texttt{kuldeep.singh1@cerence.com}, ~\texttt{abhishek.nadgeri@rwth-aachen.de}\\
~\texttt{msingh@cse.iith.ac.in}, ~\texttt{suzumura@acm.org}
}

\usepackage{xcolor}
\usepackage[shortlabels]{enumitem}
\setlist[enumerate]{nosep}
\usepackage{multirow}
\usepackage{float}
\usepackage{graphicx}
\usepackage{float}
\usepackage{array}
\usepackage{amsmath,lipsum}
\usepackage{graphicx}
\graphicspath{ {./images/} }
\usepackage[ruled]{algorithm2e}
\usepackage{subcaption} 

\usepackage{adjustbox}
\usepackage{multicol}
\usepackage{multirow}
\usepackage{tabularx}

\usepackage{xcolor}
\usepackage{color}
\usepackage{colortbl} %use in the preamble
\usepackage{inconsolata}
\usepackage{xcolor}
\definecolor{asparagus}{rgb}{0.0, 0.0, 0.0}

\usepackage{physics}
\usepackage{mathtools}
\usepackage{amsmath}
\usepackage{amssymb}

\usepackage{amsthm}
\newtheorem{theorem}{Theorem}

\newtheorem{lemma}{Lemma}

\newtheorem{property}{Property}
\newtheorem{remark}{Remark}
\makeatletter
\@addtoreset{theorem}{section}
\@addtoreset{definition}{section}
\@addtoreset{corollary}{section}
\@addtoreset{lemma}{section}
\@addtoreset{claim}{section}
\@addtoreset{property}{section}
\@addtoreset{remark}{section}
\makeatother

\newcommand{\ftmat}{\mathbf{\Psi}}
\newcommand{\invftmat}{\mathbf{\Phi}}
\newcommand{\eft}{\emph{EFT }}

\newcommand{\ljdmat}{$L_{\mathcal{J_D}}$}
\newcommand{\aeft}{\emph{AD }}
\newcommand{\dlambdamin}{\left( \Delta \lambda \right)_{min}}
\newcommand{\dlambdagmin}{\left( \Delta \lambda_G \right)_{min}}
\newcommand{\dlambdajmin}{\left( \Delta \lambda_J \right)_{min}}

\usepackage{caption}
\usepackage{subcaption}

\usepackage{multicol}
\usepackage{wrapfig}
\usepackage{blindtext}

\usepackage{mdframed}

\iclrfinalcopy % Uncomment for camera-ready version, but NOT for submission.
\begin{document}

\maketitle

\begin{abstract}

%Transforming a graph from the vertex domain to spectral space and back is a cornerstone for applications performing operations in the frequency domain.
%We present a scalable approach for semi-supervised learning on graph-structured data that is based on an efficient variant of convolutional neural networks which operate directly on graphs. We motivate the choice of our convolutional architecture via a localized first-order approximation of spectral graph convolutions. Our model scales linearly in the number of graph edges and learns hidden layer representations that encode both local graph structure and features of nodes. In a number of experiments on citation networks and on a knowledge graph dataset we demonstrate that our approach outperforms related methods by a significant margin.
%We present an invertible spectral transform to capture evolving representations on temporal graphs. We introduce the Evolving Graph Fourier Transform (\eft), which enables the transformation of temporal graphs into their frequency domain.
We present the Evolving Graph Fourier Transform (\emph{EFT}), the first invertible spectral transform that captures evolving representations on temporal graphs. 
%and enables the transformation of these graphs into their frequency domain.
% In this vein, researchers have generalized the Fourier transform for static graphs. 
% However, real-world graphs are rarely static, and existing methods are inadequate for capturing the evolving graph spectra. 
We motivate our work by the inadequacy of existing methods for capturing the evolving graph spectra, which are also computationally expensive due to the temporal aspect along with the graph vertex domain.
We view the problem as an optimization over the Laplacian of the continuous time dynamic graph. Additionally, we propose pseudo-spectrum relaxations that decompose the transformation process, making it highly computationally efficient. 
%We provide the conditions under which the simplified form approximates the exact transform and error bounds. 
% The \eft method adeptly captures the evolving graph's structural and positional properties, making it ideal for link prediction in dynamic graphs. 
The \eft method adeptly captures the evolving graph's structural and positional properties, making it effective for downstream tasks on evolving graphs. Hence, as a reference implementation, we develop a simple neural model induced with \eft for capturing evolving graph spectra. 
We empirically validate our theoretical findings
on a number of large-scale and standard temporal graph benchmarks and demonstrate that our model
achieves state-of-the-art performance.
\end{abstract}

% \maketitle

\section{Introduction} \label{sec:introduction}
%In 1822 Joseph Fourier studied motion of heat through solid bodies  in which he proposed the integral transform, famously known as the Fourier Transform. 
%Since the inception of Fourier Transform \citep{fourier_1822}, it has been fundamental in analyzing differential equations, quantum mechanics, signal processing etc \citep{bracewell1986fourier,osgood2019lectures}. Essentially, the transform provides a method to convert from a time series to the frequency domain. For graph signal processing, the theory of Fourier Transform has been generalized \citep{hammond2011wavelets} from time series to irregular graphs using eigenvalue decomposition of the graph Laplacian. Furthermore, this theory has been successfully applied in spectral graph neural networks (GNN) \citep{kipf2016gcn,san2021,defferrard2016convolutional,levie_cayleynets_2018}. 
%The Fourier Transform  \citep{fourier_1822} has played a critical role in various fields such as differential equations, quantum mechanics, and signal processing by enabling the conversion of time series data into frequency domain. 

In numerous practical situations, graphs exhibit temporal characteristics, as seen in applications like social networks, citation graphs, and bank transactions, among others \citep{kazemi2020representation}. These temporal graphs can be divided into two types: 1) temporal graphs with constant graph structure \citep{grassi2017time,cao2020spectral}, and 2) temporal graphs with dynamic structures \citep{FMLPRec,bastos2023learnable,Xu2020Inductive}. Our focus in this work is the latter case. 

The evolving graphs have been comprehensively studied from the spatio-temporal graph-neural network (GNN) perspective, focusing on propagating local information \citep{evolvegcn,shi2021gaen,ledg,Xu2020Inductive}. Albeit the success of spectral GNNs for static graphs for capturing non-local dependencies in graph signals \citep{wang2022powerful}, they have not been applied to temporal graphs with evolving structure. To make spectral GNN work for temporal graphs effectively and efficiently, there is a necessity for an invertible transform that collectively captures evolving spectra along the graph vertex and time domain. To the best of our knowledge, there exists no such transform in the spectral domain for temporal graphs with evolving structures. 
%In graph signal processing domain, the theory of Fourier Transform (FT) \citep{fourier1888theorie} has been extended from time series to irregular graphs using eigenvalue decomposition of the graph Laplacian \citep{hammond2011wavelets}. Furthermore, this theory has been successfully applied in graph neural networks (GNN) \citep{kipf2016gcn,san2021,defferrard2016convolutional,levie_cayleynets_2018}, which are widely used for static graphs. 
%In graph signal processing domain, the theory of Fourier Transform (FT) \citep{fourier1888theorie} has been extended from time series to irregular graphs using eigenvalue decomposition of the graph Laplacian \citep{hammond2011wavelets}. Furthermore, this theory has been successfully applied in graph neural networks (GNN) \citep{kipf2016gcn,san2021,defferrard2016convolutional,levie_cayleynets_2018}, which are widely used for static graphs. 
%are interested  in the latter case. In essence, these approaches endeavor to capture the evolving attributes of graphs by leveraging spatial features that rely on local neighborhood aggregation, thereby accounting for local dependencies (EvolveGCN, Goyal et al., 2020).

In the present literature, Graph Fourier Transform (GFT), which is a generalization of Fourier Transform, exists for static graphs but can not capture spectra of evolving graph structure \citep{shuman2013emerging}. Hence, it cannot be applied to temporal graphs due to the additional temporal aspect. One naive extension would be to treat the time direction as a temporal edge, construct a directed graph with newly added nodes at each timestep, and find the Eigenvalue Decomposition (EVD) of the joint graph. However, this would lose the distinction between variation along temporal and vertex domains. Moreover, such an approach would incur an added computational cost by a multiplicative factor of $\mathcal{O}(T^3)$, which would be prohibitively high for the temporal setting with a large number of timesteps. Thus, in this paper, we attempt to find an approximation to the dynamic graph transform that would capture its evolving spectra and be efficient to compute. 

%For proposing the first transform for dynamically evolving graphs, Similar to prior works \citep{hammond2011wavelets},
We aim to propose a novel transform for a temporal graph to its frequency domain. For this we consider the Laplacian of the dynamic graph and find the orthogonal basis of maximum variation to obtain the spectral transform \citep{hammond2011wavelets}. We view this as an optimization of the variational form of the Laplacian such that the optimal value is within the $\epsilon-$ pseudospectrum \citep{terry_pseudospectrum}. We then show that such optimization gives us a simple and efficient to compute solution while also being close to the exact solution of the variational form under certain conditions of Lipschitz continuous dynamic graphs. 
% Specifically, we provide conclusive evidences that the proposed transform for dynamic graphs can be computed simply by computing the Graph Fourier Transform of the graphs at each timestep followed by the Discrete Fourier Transform (DFT) along the time dimension or in the reverse order (i.e., collectively across vertex and time domain). 
Effectively, we propose a method to simultaneously perform spectral transform along both the time and vertex dimensions of a dynamic graph. 
This solves the following challenges over the natural extension of EVD over dynamic graphs: 1) The proposed transformation is computationally efficient compared to the direct eigendecomposition of the joint Laplacian. 
% 2) Distinguishing between low-pass and high-pass frequencies with the proposed transform provides interpretability to the
% frequency components, whether belonging to the time or vertex domain.
2) Distinction between time and vertex domain frequency components with the proposed transform provides interpretability to the transformed spectral domain.
We term the proposed concept as "Evolving Graph Fourier Transform" (\eft).

In summary, we make the following key contributions:

\begin{itemize}[resume, before = \vspace*{-\dimexpr\topsep+\partopsep\relax}, leftmargin=*]
    \item We propose $\eft$ (grounded on theoretical foundations), that transforms a temporal graph to its frequency domain for capturing evolving spectra.
    \item We provide the theoretical bounds of the difference between $\eft$ and the exact solution to the variational form and analyze its properties. 
    % This imparts the method with good interpretable properties. We develop a simple neural architecture, Graph Isomorphism Network (GIN), and show thatits discriminative/representational power is equal to the power of the WL test.
    \item As a reference implementation, we develop a simple neural model induced with the proposed transform
   to process and filter the signals on the dynamic graphs for downstream tasks. We perform extensive experimentation on large-scale and standard datasets for dynamic graphs to show that our method can effectively filter out the noise signals and enhance task performance against baselines.
\end{itemize}

\section{Related Work} \label{sec:related}
\textbf{Spectral Graph Transforms:}
Work by \citep{hammond2011wavelets} was among the first to propose a computationally efficient algorithm to compute the Fourier Transform for static graphs.
Loukas et al. \citep{loukas2016frequency} conceptualized Joint Fourier Transform (JFT) over graphs on which the signals change with time. 
% However, JFT does not consider graph structures evolving with time. 
\textcolor{asparagus}{JFT has been generalized in \citep{kartal2022joint} by proposing the Joint Fractional Fourier Transform (JFRT).}
However, JFT and JFRT does not consider graph structures evolving with time. 
\citep{cao2021spectral} apply JFT and propose a model for time series forecasting. 
% \citep{DGFT_icsap} finds an orthogonal matrix to summarize graph Laplacians over time but does not fully capture the varying graph information. 
% \citep{DGFT_icsap} finds an orthogonal matrix for graph Laplacian over time but does not fully capture the varying graph information. 
% \citep{cao2021spectral} have proposed  combine Fourier Transform for time series forecasting. 
%However authors do not propose a transform for dynamic graphs and only propose an architecture for time series forecasting.
\citep{DGFT_icsap} summarized graphs over time by using Tucker decomposition to the dynamic graph Laplacian in order to obtain an orthogonal matrix and further applies it to a cognitive control experiment. However, this method does not fully capture the varying graph information in a lossless sense. 
\textcolor{asparagus}{Researchers have also proposed spectral methods for spatio-temporal applications such as action recognition \citep{yan2018spatial,pan2020spatio}, traffic forecasting \citep{yu2017spatio} etc.}
Other works such as \citep{mahyari2014fourier,chen2022gc,sarkar2012nonparametric,kurokawa2017multidimensional,joint_tv_filter_9374709,svd_gft_10195957} also consider temporal graphs, but ignore the evolving structure. 
% We position our work as the novel spectral graph transform for continuous time dynamic graphs which is currently a gap in existing literature. 
We position our work as the novel spectral graph transform for temporal graphs which is currently a gap in existing literature. 

% \textbf{Link Prediction and Dynamic Graph Representation Learning:} 
\textbf{Temporal Graph Representation Learning:} 
% The link prediction approaches are broadly classified into three categories \citep{chamberlain2022graph} 1) heuristic-based \citep{zhou2009predicting}, 2) unsupervised node embeddings based \citep{chamberlain2020tuning}, and 3) GNN-based \citep{kipf2016gcn,hamilton2017inductive}. 
Since static graph methods do not work well with dynamic graphs \citep{evolvegcn}, researchers have proposed a slew of methods \citep{evolvegcn,DBLP:journals/kbs/GoyalCC20,ledg}, for learning on dynamic graphs for problems such as link prediction and node classification. 
One elementary way to adapt methods developed for static graphs on dynamic graphs is to use RNN modules in conjunction with GNN modules to capture the evolving graph dynamics. 
% This idea has been explored in works such as \citep{DBLP:journals/corr/SeoDVB16,narayan2018learning,DBLP:journals/pr/ManessiRM20}. 
Researchers \citep{DBLP:journals/corr/SeoDVB16,narayan2018learning,DBLP:journals/pr/ManessiRM20} have explored this idea extensively.
Some other recent approaches model several real world phenomena, however, these methods rely on an RNN for encoding temporal information such as \citet{bastas2019evolve2vec}, \citet{Xu2020Inductive}, \citet{Ma2020streaming}, etc. Most generic among these works is TGN (Temporal Graph Networks) \citep{rossi2020temporal} that remembers nodes and connections it has seen in the past, and then uses that memory to update new nodes and connections that it hasn't seen before. However, the memory updater uses GRU which may have issues such as vanishing gradient limiting the ability to capture long term information. Also, these models have been studied for small-graphs spread over limited time duration (e.g., one month). \\ 
% For large graphs with decades of temporal information (1996-2018) \citep{DGSR}, bastas2019evolve2vec,ma2020streaming,kumar2019predicting
Considering large scale temporal graphs with evolving structures, one such application is that of sequential recommendation (SR) with decades of temporal information (1996-2018) \citep{DGSR, huang2023temporal}. 
Researchers \citep{DGCF,DGSR,TKDD-dynamic} have attempted to model the sequential recommendation task as a link prediction over dynamic graphs. 
%PDGCN \citep{TKDD-dynamic} creates item-item, user-item, and user-sub sequence graphs and then applies two layers of graph convolutions. The generic graph convolutions are inherently low pass \citep{balcilar2020analyzing}, and the hop limit restricts the ability to capture complex interactions.
% DGCF \citep{DGCF} at most considers 2-hop relations instead of the generic graph, as the authors argue that higher-order relations may induce noise. 
%In this vein, DGCF \citep{DGCF} considers 2-hop neighborhood graphs, as the authors argue that higher-order relations may induce noise. 
DGSR \citep{DGSR} is a work that considers generic dynamic graphs over user-item interactions. However, the GNN-based methods described in this section including DGSR majorly employ low pass GNNs that limit the ability to model complex relations and are fundamentally restricted to focus on local neighborhood interactions \citep{balcilar2020analyzing}. 
\begin{figure*}
    \centering
    \includegraphics[width=0.77\textwidth]{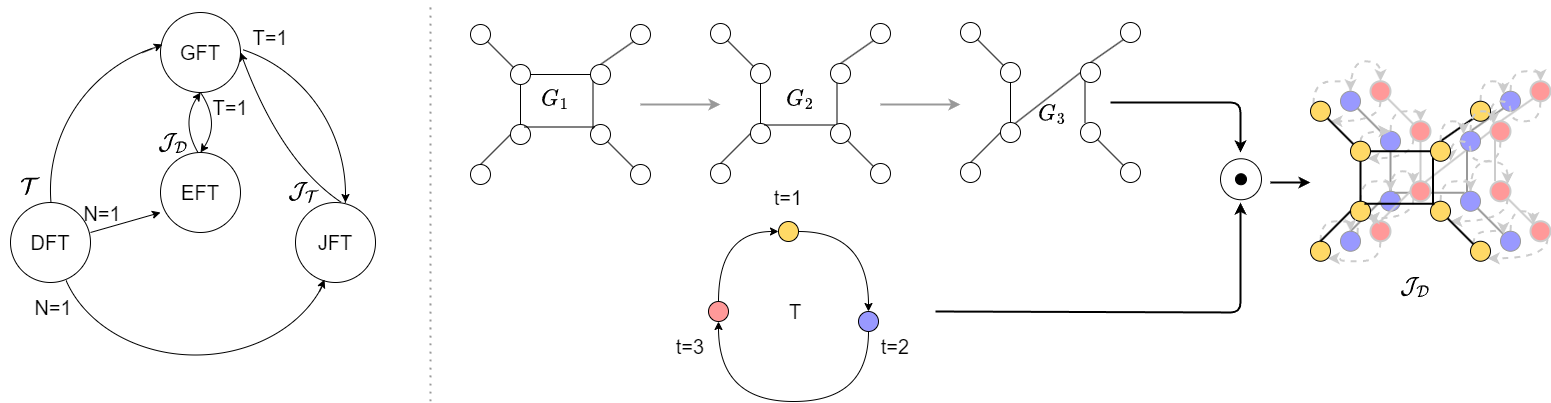}
    \caption{Left circular figure shows equivalence between \emph{EFT} and existing transformations (DFT \citep{sundararajan2023discrete}, JFT \citep{loukas2016frequency}, GFT \citep{ortega2018graph}). Each directed arrow (e.g, A to B), interprets as a transform simulation (transform A can be simulated by B using edge annotations). Right part shows timestamp-wise product between signals and graph structure. Here, nodes of next timestep are connected by dotted arrows to obtain the graph $\mathcal{J_D}$ which can be used by GFT to simulate \emph{EFT} (if graph is static).}

    \label{fig:ctdft_motivation}
\end{figure*}

\section{Preliminaries} \label{sec:preliminaries}
% Graph Fourier Transform
%Below, we summarize the preliminary concepts used in the paper.\\
\textbf{Discrete Fourier Transform}\label{subsec:dft}
 % Discrete Fourier Transform 
 (DFT) \citep{sundararajan2023discrete} is employed to obtain the frequency representation of a sequence of signal values sampled at equal intervals of time. 
% The magnitude and phase of the frequency components are obtained by multiplying the signal values by complex sinusoids of the respective frequencies. 
Consider a signal $x$ sampled at $N$ intervals of time $t \in [0, N-1]$ to obtain the sequence $\{x_t\}$. The DFT of $x_t$ is then given by $X_k = \sum_{t=0}^{N-1} x_t e^{-i \omega_t k}$ with $\quad \omega_t=\frac{2 \pi t}{N}$.
% % \small{
% \begin{equation} 
%     X_k = \sum_{t=0}^{N-1} x_t e^{-i \omega_t k}, \quad \omega_t=\frac{2 \pi t}{N}
% \end{equation}
% % }
The transformed sequence $X_k$ gives the values of the signal in the frequency domain. If we represent $X$ as the vector form of the signal, we can define the DFT matrix $\ftmat_T$ such that $X_k = \ftmat_T X$. 
% Thus $X_k$ is the complex valued spectrum of $\{x_t\}$ at frequency $\omega_t$. We can perform filtering by removal of noisy frequencies in this spectral domain. As required by signal processing applications we can then obtain the signal sequence in the time domain from the frequency domain $\{X_k\}$ using the Inverse Discrete Fourier Transform (IDFT) as, 
% \begin{equation}
%     x_t = \frac{1}{N} \sum_{k=0}^{N-1} X_k e^{i \omega_k t}, \quad \omega_k=\frac{2 \pi k}{N}
% \end{equation}

% Fourier transform
\textbf{Graph Fourier Transform }\label{subsec:gft}
% Graph Fourier Transform 
($\emph{GFT}$) \citep{ortega2018graph} is a generalization of the Discrete Fourier Transform (DFT) to graphs. We represent a graph as $(\mathcal{V},\mathcal{E})$ where $\mathcal{V}$ is the set of $N$ nodes and $\mathcal{E}$ represents the edges between them. Denote the adjacency matrix by $A$. 
% In the setting of an undirected graph $A$ would be a symmetric matrix. 
$D$ is the degree matrix, defined as $(D)_{i}^{i}=\sum_{j}(A)_{ij}$, which is diagonal. The graph Laplacian graph is given by $\hat{L}=D-A$ and the normalized Laplacian $L$ is defined as $L=I-D^{-\frac{1}{2}}AD^{-\frac{1}{2}}$. The Laplacian($L$) has the eigendecomposition as: $L = \ftmat_G^{*} \Lambda \ftmat_G$.
% %\[ L = U \Lambda U^{*} \] 
% where $\ftmat_G^{*}$ is an $N \times N$ matrix whose columns are the eigenvectors corresponding to the eigenvalues $\lambda_1, \lambda_2, \dots, \lambda_N$ and $\Lambda = {\rm diag}([\lambda_1, \lambda_2, \dots, \lambda_n])$. 
Let $X \in R^{N \times d}$ be the signal on the nodes of the graph. The Graph Fourier Transform $\hat{X}$ of $X$ is then given as: $ \hat{X} = \ftmat_G X$. 

\textbf{Pseudospectrum}:
The spectrum of a graph (of N nodes) is a finite set consisting of N points $\lambda$ that form the eigenvalues of the graph's matrix representation $M$ i.e. $\{ \lambda \in \mathbb{C} \quad | \quad \norm{(M-\lambda I)^{-1}} = \infty \}$. Similarly we can think of the ($\epsilon$-)pseudospectrum of a graph to be the larger set (containing these N points) such that $A-\lambda I$ has the least singular value at most $\epsilon$. Formally the pseudospectrum can be defined by the set $\{ \lambda \in \mathbb{C} \quad | \quad \norm{(M-\lambda I)^{-1}} \geq \frac{1}{\epsilon} \}$.

{
\color{asparagus}
\textbf{Common Notations}:
We denote by $\oplus$, $\otimes$ the Kronecker sum, product respectively. $(M)_{i}^{j}$ refers to the $i$-th row and $j$-th column of matrix $M$. $\{.\}$ refers to a sequence, of elements, in time. $\boxtimes, \boxplus$ refer to the Kronecker product and sum respectively, applied timestep wise.
}

% \section{Problem Statement and Method} \label{sec:problem_approach}
% \section{Proposed Transform:\emph{EFT}} \label{sec:problem_approach}
\section{Theoretical Framework: An Optimization Perspective} \label{sec:problem_approach}
We begin by striving for a physical interpretation of frequency for dynamic graph systems. For this, we draw inspiration from energy diffusion processes and establish similarities with the variation of signals on static graphs. 
Consider graph $G_t$ at time $t$ with node $n_{i} \in V_t$ and $n_j \overset{G_t}{\thicksim} n_i$ denoting the neighbors of $n_i$ at time $t$. We define a directed graph $\mathcal{J_D}$ with the graphs at all timesteps taken as is and a directed edge added from a node at time $t-1$ (modulo $T$) to its corresponding node at time $t$. For continuous time dynamic graph the previous time would be represented by $t-dt$ (modulo $T$). Let $X_{n_{i},t}$ represent the energy of the signal on node $n_i$ at time $t$. The flow of energy to the node $n_i$ at time $t$ can be represented by the divergence of the gradient ($\Delta_{n_{i},t} X$) of the energy. We define the variation of the signals at time $t$ and node $n_i$ as follows:  $\small{\left\lVert \Delta_{n_{i},t} X \right\rVert_2 = \left[ \sum_{n_j \overset{\mathcal{J_D}}{\thicksim} n_i} \left( \frac{\partial X}{\partial e_{n_i n_j}} \right)^2 \right]^{\frac{1}{2}} = \left[ \sum_{n_j \overset{G_t}{\thicksim} n_i} \left( X_{n_j,t} - X_{n_i,t} \right)^2 + \left( \frac{\partial X_{n_i,t}}{\partial t} dt \right)^2 \right]^{\frac{1}{2}}}$
% \begin{align}
%     &\left\lVert \Delta_{n_{i},t} X \right\rVert_2 = \left[ \sum_{n_j \overset{\mathcal{J_D}}{\thicksim} n_i} \left( \frac{\partial X}{\partial e_{n_i n_j}} \right)^2 \right]^{\frac{1}{2}} \\
%     % &= \left[ \sum_{n_j \overset{G_t}{\thicksim} n_i} \left( X_{n_j,t} - X_{n_i,t} \right)^2 + \left( X_{n_i,t-1} - X_{n_i,t} \right)^2 \right]^{\frac{1}{2}}
%     &= \left[ \sum_{n_j \overset{G_t}{\thicksim} n_i} \left( X_{n_j,t} - X_{n_i,t} \right)^2 + \left( \frac{\partial X_{n_i,t}}{\partial t} \right)^2 \right]^{\frac{1}{2}}
% \end{align}
% \begin{align}
%     &\left\lVert \Delta_{n_{i},t} X \right\rVert_2 = \left[ \sum_{n_j \overset{\mathcal{J_D}}{\thicksim} n_i} \left( \frac{\partial X}{\partial e_{n_i n_j}} \right)^2 \right]^{\frac{1}{2}} = \left[ \sum_{n_j \overset{G_t}{\thicksim} n_i} \left( X_{n_j,t} - X_{n_i,t} \right)^2 + \left( \frac{\partial X_{n_i,t}}{\partial t} \right)^2 \right]^{\frac{1}{2}}
% \end{align}
, where $\frac{\partial X}{\partial e_{n_i n_j}}$ is the discrete edge derivative on the collective dynamic graph $\mathcal{J_D}$. Considering $\Delta$ to be the finite difference between neighboring nodes in the joint graph, the global notion of variation ($S_p(X)$) can be given by the $p$-Dirichlet form as follows
% \begin{align}
%     S_p(X) &= \frac{1}{p} \sum_{n=1}^{N} \sum_{t=1}^{T}  \left\lVert \Delta_{n_{i},t} X \right\rVert_2^p \\
%     &= \frac{1}{p} \sum_{t=1}^{T} \sum_{n=1}^{N} \left[ \sum_{n_j \overset{G_t}{\thicksim} n_i} \left( X_{n_j,t} - X_{n_i,t} \right)^2 + \left( X_{n_i,t-1} - X_{n_i,t} \right)^2 \right]^{\frac{p}{2}}
% \end{align}
% \begin{align}
%     S_p(X) &= \frac{1}{p} \sum_{n=1}^{N} \int_{t=0}^{T}  \left\lVert \Delta_{n_{i},t} X \right\rVert_2^p \\
%     &= \frac{1}{p} \int_{t=0}^{T} \sum_{n=1}^{N} \left[ \sum_{n_j \overset{G_t}{\thicksim} n_i} \left( X_{n_j,t} - X_{n_i,t} \right)^2 + \left( \frac{\partial X_{n_i,t}}{\partial t} \right)^2 \right]^{\frac{p}{2}}
% \end{align}
\begin{align*}
    \small
    S_p(X) = \frac{1}{p} \sum_{n=1}^{N} \int_{t=0}^{T}  \left\lVert \Delta_{n_{i},t} X \right\rVert_2^p dt = \frac{1}{p} \int_{t=0}^{T} \sum_{n=1}^{N} \left[ \sum_{n_j \overset{G_t}{\thicksim} n_i} \left( X_{n_j,t} - X_{n_i,t} \right)^2 + \left( \delta X_{n_i,t} \right)^2 \right]^{\frac{p}{2}} dt
\end{align*}

Define $L_T$ to be the Laplacian of the continuous ring graph representing the nodes at each timestep $t \in [0,T]$ and connecting consequent nodes. Let $L_{G_t}$ be the Laplacian of the sampled graph at time $t$. In the discrete case the Laplacian $L_{\mathcal{J_D}}$ of $\mathcal{J_D}$ can be shown to be 
\begin{equation}
    (L_{\mathcal{J_D}})_{i}^{j} = (L_{T} \otimes I_{N})_{i}^{j} + (I_{T} \otimes \{L_{G_t}\})_{i}^{j \left\lfloor \frac{j}{N} \right\rfloor} = (L_T \oplus {L_{G_t}})_{i}^{j \left\lfloor \frac{j}{N} \right\rfloor} \\
\end{equation}
% where $\oplus$, $\otimes$ is the Kronecker sum, product respectively and $(M)_{i}^{j}$ refers to the $i$-th row and $j$-th column of matrix $M$. 
For the case of continuous time, this can be generalized to 
\begin{equation}
    (L_{\mathcal{J_D}}) = L_{T} \otimes I_{N} + [I_{T} \boxtimes \{L_{G_t}\}] = [L_T \boxplus {L_{G_t}}] \\
\end{equation}
where $\boxtimes, \boxplus$ refers to the timestep wise Kronecker product and sum respectively and $[.]$ refers to the matricization operation. In the discrete case this operation would convert $R^{NT \times T \times N} \xrightarrow[]{} R^{NT \times NT}$, ordering from the last dimension first. We can now characterize the variation of signals on $J_D$ similar to static graphs by the following result:
% \begin{lemma}\label{lemma_variational_characterization}(Variational Characterization of $\mathcal{J_D}$)
%     The $2$-Dirichlet $S_2(X)$ of the signals $X$ on $\mathcal{J_D}$ is the quadratic form of the Laplacian $L_{\mathcal{J_D}}$ of $\mathcal{J_D}$ i.e. \\
%     \[ S_2(X) = vec(X)^T L_{\mathcal{J_D}} vec(X) \]
% \end{lemma}

% \begin{lemma}\label{lemma_variational_characterization}(Variational Characterization of $\mathcal{J_D}$)
%     The $2$-Dirichlet $S_2(X)$ of the signals $X$ on $\mathcal{J_D}$ is the quadratic form of the Laplacian $L_{\mathcal{J_D}}$ of $\mathcal{J_D}$ i.e. \\
%     \[ S_2(X) = \int_{i=0}^{NT} \text{vec}(X)(i) \int_{j=0}^{NT} L_{\mathcal{J_D}}(i,j) \text{vec}(X)(j) di dj = \text{vec}(X)^T L_{\mathcal{J_D}} \text{vec}(X) \]
% \end{lemma}
\begin{lemma}\label{lemma_variational_characterization}(Variational Characterization of $\mathcal{J_D}$)
    The $2$-Dirichlet $S_2(X)$ of the signals $X$ on $\mathcal{J_D}$ is the quadratic form of the Laplacian $L_{\mathcal{J_D}}$ of $\mathcal{J_D}$ i.e. \\
    \color{asparagus}{
    \begin{equation}
        S_2(X) = \int_{i=0}^{NT} \text{vec}(X)(i) \int_{j=0}^{NT} L_{\mathcal{J_D}}(i,j) \text{vec}(X)(j) di dj = \text{vec}(X)^T L_{\mathcal{J_D}} \text{vec}(X)
    \end{equation}
    }
\end{lemma}
This implies that $L_{\mathcal{J_D}} \succeq 0$ since $S_2(X) \geq 0$, which assures us of the existence of the eigenvalue decomposition. Additionally, the value of $S_2(X)$ is lower when the signal changes slower along the dynamic graph and higher when the signal changes faster. Hence, we can define a notion of signal variation on the dynamic graph that is similar to the variation of signals on static graphs.
Consequently, the eigendecomposition of \ljdmat characterizes signals on the dynamic graph by projecting them onto the optimizers of $S_2(X)$. This means that high-frequency components of the evolving dynamic graph represent sharply varying signals, whereas smoother signals will have a higher magnitude in the low-frequency components. From an optimization perspective, we can view the maximum frequency as the optimal value for the below equation, i.e.,
%Moreover, we can see that the slower the signal changes along the dynamic graph, the lower the value of $S_2(X)$ and vice versa. Thus we have a notion of variation of signals on the dynamic graph similar to the case of static graphs. Therefore, the eigendecomposition of \ljdmat characterizes signals on the dynamic graph by its projection to the optimizers of $S_2(X)$. In other words, high collective dynamic graph frequency components inform of sharply varying signals, and smoother signals will have a higher magnitude in the low-frequency components. From an optimization perspective, we can view the maximum frequency as the optimal value for the below equation, i.e.,
% Now we look at the equivalence between the eigen decomposition of \ljdmat and $\ftmat_D$.
% \begin{align}\label{eq_variational_char_LD}
%     f_{\text{max}} &= \underset{x, \norm{x} \leq 1}{max} \int_{i=0}^{NT} x(i) \int_{j=0}^{NT} L_{\mathcal{J_D}}(i,j) x(j) di dj \\
%     &= \underset{x, \norm{x} \leq 1}{max} \quad x^T L_{\mathcal{J_D}}(i,j) x
% \end{align}
\begin{align}\label{eq_variational_char_LD}
    f_{\text{max}} &= \underset{x, \norm{x} \leq 1}{\text{max}} \int_{i=0}^{NT} x(i) \int_{j=0}^{NT} L_{\mathcal{J_D}}(i,j) x(j) di dj = \underset{x, \norm{x} \leq 1}{\text{max}} \quad x^T L_{\mathcal{J_D}}(i,j) x
\end{align}
The optimal solution $x$ provides the basis for transforming a dynamic graph signal to obtain its maximum frequency component, denoted by $f_{\text{max}}$. We can obtain the next frequency values by optimizing equation \ref{eq_variational_char_LD} in orthogonal directions. However, this approach has an issue - the eigenvalue decomposition would have to be performed over a large number of nodes. In a real world setting of temporal graphs with $T$ timesteps, this method would have a complexity of $\mathcal{O}((NT)^3)$, which would be prohibitive considering large number of timesteps. 
To address this issue, we relax the objective in equation \ref{eq_variational_char_LD} to include solutions in the pseudospectrum. The solution is presented in the following result, upon which we can formulate a transformation method for temporal graphs.
\begin{lemma}\label{lemma_ctdft_der}
    Consider the variational form $ x^T L_{\mathcal{J_D}} x = \int_{i=0}^{NT} x(i) \int_{j=0}^{NT} L_{\mathcal{J_D}}(i,j) x(j) di dj $. The optimization problem $f = \underset{x, \norm{x} \leq 1}{\text{min}} [ | x^T L_{\mathcal{J_D}} x - \lambda_s | - \epsilon ]_{+}$ has the optimal solution as $y_{\omega} \otimes {z_{l}^{\omega}} $, where $\lambda_s$ is the optimal value of equation \ref{eq_variational_char_LD}, $y_{\omega}$ is the $\omega$-th optimal solution of the variational form of the ring graph, $z_{l}^{t}$ is the $l$-th optimal solution to the variational form of the graph at time t, $[s]_+=\text{max}(s,0)$ and $\epsilon = \mathcal{O}(\delta)$.
\end{lemma}
% TODO: define x_t and z_{l}^{t} in a better way

% \section{Towards Evolving Graph Fourier Transform}
\section{Constructing an Evolving Graph Fourier Transform}
In the previous section, we have outlined the theoretical framework for the evolving graph Fourier transform. We also obtained a sketch of the transform as a solution to the optimization problem of the variational characterization with pseudospectrum relaxations. This enables us to obtain a simple and efficient form to compute. In this section, building upon the theoretical framework, we propose our formulation of the \emph{Evolving Graph Fourier Transform} (\emph{EFT}).
From lemma \ref{lemma_ctdft_der}, we obtain the orthogonal basis vectors of the desired transform matrix in terms of the kronecker product of the basis vectors of the Fourier Transform ($\ftmat_T$) and Graph Fourier Transform ($\ftmat_G$). Thus, lemma \ref{lemma_ctdft_der} helps us to define the \emph{EFT} in terms of the graph and time Fourier transforms:
% \begin{equation}
%     EFT(f_g,\omega) = \sum_n U^T(f_g,n) \int_{t=0}^{T} f(t) e^{-j \omega t} dt
% \end{equation}
% \vspace{-2mm}
\begin{equation}
    \small
    EFT(f_g,\omega) = \sum_n \ftmat_G(f_g,n) \int_{t=0}^{T} f_s(n,t) e^{-j \omega t} dt
\end{equation}
% \vspace{-4.2mm}
{\color{asparagus}  where $f_g, \omega$ are the graph and temporal frequency components respectively, $f_s(n,t)$ is the signal at node $n$ and time $t$.}
In terms of the matrix representation, the \eft could be expressed, using the Einstein notation \citep{albert1916foundation}, as a Kronecker product of DFT and GFT as
% $\ftmat_D = \ftmat_T \otimes \ftmat_G$ 
$(\ftmat_{D})_i^j = (\ftmat_T \otimes \{\ftmat_{G_t}\})_{i}^{j \left\lfloor \frac{j}{N} \right\rfloor}$
, which when applied to the columnwise vectorized signal $f_s$ gives the transform in the spectral space.
% for the dynamic graph. 

\eft is one of the solutions in the pseudospectrum of $L_{\mathcal{J_D}}$ as shown in lemma \ref{lemma_ctdft_der}. There also exists other solutions and specifically considering the case where $\epsilon=0$ we obtain the solution to the exact EVD of \ljdmat. 
Let $\ftmat_{AD}$ be the matrix whose rows form the right eigenvectors of $L_{\mathcal{J_D}}$. 
Since $\ftmat_{AD}$ is the absolute decomposition of \ljdmat we term this as \emph{AD} for brevity. We now define error bounds between $\ftmat_{D}$ and $\ftmat_{AD}$. 
\begin{theorem}\label{theorem_bounds}
    Considering bounded changes in a graph $G$ with $N$ nodes over time $T$, the norm of the difference between \eft ($\ftmat_D$) and \aeft ($\ftmat_{AD}$) is bounded as follows: 
    $\norm{\ftmat_D-\ftmat_{AD}} \leq \mathcal{O}\left( N^{\frac{3}{2}} T \varepsilon(\omega_{max},(\Delta \lambda_G)_{min}, (\Delta \lambda_J)_{min}) \right) \left( \norm{\dot{L}_G} \right)_{max}$
    where $(\Delta \lambda_J)_{min}$ and $(\Delta \lambda_J)_{min}$ refer to the minimum difference between the eigenvalues of matrices $L_G$ and \ljdmat respectively, $\dot{L}_G$ is the rate of change of $L_G$ and $\omega_{max}=2\pi$ and $\varepsilon(\omega_{max},\Delta \lambda_G,\Delta \lambda_J) = \frac{ \omega_{max}^{\frac{1}{2}}}{\Delta \lambda_G} +  \frac{\omega_{max}^{2}}{\Delta \lambda_J}$.
\end{theorem}
The above theorem states that as the structure on the graph evolves infinitesimally, the difference between $\ftmat_D$ and $\ftmat_{AD}$ is bounded from above by the change in the graph matrix representation (laplacian/adjacency). This property is desirable since it allows us to approximate $\ftmat_{AD}$, which is formed by the eigendecomposition of \ljdmat\ and has a physical interpretation, using the defined $\ftmat_D$ that is easy to compute. 
% Above approximation is valid when the rate of change of the graph structure is small. 
{\color{asparagus} The above bound is finite if 1) The rate of change of the graph with time is bounded. 2) The eigenvalues have a multiplicity of 1.}
In such cases, \emph{EFT} characterizes signals on the dynamic graph by their proximity (projection) to the optimizers of $S_2(X)$ defined in lemma \ref{lemma_variational_characterization}. The physical implication of this is that applying \eft, the high-frequency components correspond to sharply varying signals on a dynamic graph, while low-frequency components correspond to smoother signals.
%We thus see that as the graph evolves infinitesimally the difference between $\ftmat_D$ and $\ftmat_{AD}$ is bounded from above by the change in the graph matrix representation. This is desirable since it allows us to approximate $\ftmat_{AD}$ (formed by the eigendecomposition of \ljdmat) which has a physical interpretation using the defined $\ftmat_D$ which is simple to compute, when the graph changes in a stable manner.In such cases, \emph{EFT} therefore characterizes signals on the dynamic graph by its proximity (projection) to the optimizers of $S_2(X)$ meaning high (collective dynamic graph) frequency components correspond to sharply varying signals and low frequency components to smoother signals.
Hence, the norm of the difference between \eft and \aeft are bounded from above by the rate of evolution of the graphs. 
%In the special case where the graph structure remains the same, but only the signals change, we shall show that the two transforms are exactly the same.
%in the discrete setting.
%and in the special case that the graph structure remains same but only the signals change, the two are exactly the same as we shall show for the discrete setting.

For \textit{computational purpose in real-world applications}, the sampled form of \eft can be obtained by sampling $T$ snapshots of the dynamic graph signal at uniform time intervals. We now get a dynamic graph $\{(\mathcal{V}_t,\mathcal{E}_t)\}, t \in \{0, T\}$ the edges ($\mathcal{E}_t$) of which by definition evolves with time. We consider the node set $\mathcal{V}$ to be fixed, i.e., no new nodes are added. All the nodes ($|\mathcal{V}|=N$) are known from the start, and the graph may contain isolated nodes.
% {
% \color{asparagus} Note however this is not a limitation of the method as even if nodes appear and disappear in the dataset, we could create dummy isolated nodes with varying node signals and edge connectivity information. 
% }
{\color{asparagus} In case of node editions, we could create dummy isolated nodes with varying node signals and edge connectivity information. }
Without loss of generality, consider a 1-dimensional temporal signal, uniformly sampled at $T$ intervals, residing on the graph nodes. 
% \footnote{The analysis holds for higher dimensions too by performing \emph{EFT} dymension wise}
% The analysis holds for higher dimensions too by performing \emph{EFT} dimension wise.
Let $X \in R^{N \times T}$ represent the temporal signal on the graph nodes. The Fourier transform (DFT) (with DFT matrix $\ftmat_T$) independently for each node is $DFT(X) = X \ftmat_T^{\top}$.
% \begin{equation}\label{dft_X}
%     DFT(X) = X \ftmat_T^{\top}
% \end{equation}
Further, the $\emph{GFT}$ for the graph $G_t \equiv (\mathcal{V}_t,\mathcal{E}_t)$ at time $t$ is given as $GFT(X_t) = \ftmat_{G_t} X_t$, 
% \begin{equation}
%     GFT(X_t) = \ftmat_{G_t} X_t
% \end{equation}
where $X_t \in R^{N}$ is the signal on the graph nodes at time $t$. 
% The above equation computes the transform of each graph independently and does not consider the temporal dimension.
In order to compute the dynamic graph transform along the graph domain as well as the temporal dimension, we can \emph{collectively} perform both the operations. 

Consider $\{\ftmat_{G_t}\} \in R^{N \times N \times T}$ as the tensor containing the graph Fourier basis at each timestep. Then using Einstein notation \citep{albert1916foundation}, we write \emph{EFT} as
% Consider $\{\ftmat_{G_t}\} \in R^{N \times N \times T}$ as the tensor containing the graph Fourier basis at each timestep. Then using Einstein notation \citep{albert1916foundation} and the Kronecker product ($\otimes$) between tensors, we write \emph{EFT} as
\begin{equation}\label{jgft_einsum}
    \boxed{\left( \mathbf{EFT}(\{G_t\}; X) \right)_{i}^{j} = \left( \ftmat_{G_t} X \right)_{i}^{kk} \left( \ftmat_T^{\top} \right)_{k}^{j}  }
\end{equation}
% \begin{equation}\label{jgft_einsum}
%     \boxed{\left( \mathbf{EFT}(\{G_t\}; X) \right)_{i}^{j} = \left( \ftmat_{G_t} X \right)_{i}^{kk} \left( \ftmat_T^{\top} \right)_{k}^{j} = \left( \ftmat_T \otimes \{\ftmat_{G_t}\} \right)_{(j*N + i)}^{km} x_{k}}
% \end{equation}
% % \end{mdframed}

% TODO: write in matrix form using einsum and Kronecker
% Write EFT in matrix form...
where $i,j,k$ are tensor indices. 
Next, we aim to define a transformation matrix for \emph{EFT} as in DFT and GFT.
%Next, we would like to express the transformation in matrix form so we can represent the \emph{EFT} as a matrix-vector multiplication as in DFT. 
For this we make use of the Kronecker product ($\otimes$) between tensors. 
% With a slight permutation of dimensions in $\{\ftmat_{G_t}\}$ consider $\ftmat_G \in R^{N \times N \times T}$ to be the tensor containing the basis vectors (GFT matrix) at each timestep. 
% We then get the following expressions:
We then get the matrix form of \eft as the following expression:
% \begin{align} \label{ctdft_XG}
%     \left( \mathbf{EFT(\{G_t\}; X)} \right)_{i}^{j} &= \left( \ftmat_{G_t} X \right)_{ki}^{k} \left( \ftmat_T^T \right)_{j}^{k}  \\
%     \hat{X}_G &= \left( \ftmat_{G_t} X \right)_{ki}^{k} \left( \ftmat_T^T \right)_{j}^{k}  \\
%     \left( \hat{X}_G \right)_{ij} = \hat{x}_{k}  &= \left( \ftmat_T \otimes \ftmat_{G} \right)_{(i*j)}^{km} x_{k}  \\
%     \hat{x} &= \ftmat_D x
% \end{align}
% \begin{align} \label{ctdft_XG}
%     \left( \mathbf{EFT(\{G_t\}; X)} \right)_{i}^{j} = \left( \hat{X}_G \right)_i^j &= \left( \ftmat_{G_t} X \right)_{ki}^{k} \left( \ftmat_T^T \right)_{j}^{k}  \\
%     % \hat{X}_G &= \left( \ftmat_{G_t} X \right)_{ki}^{k} \left( \ftmat_T^T \right)_{j}^{k}  \\
%     \left( \hat{X}_G \right)_{i}^{j} = \hat{x}_{j*N + i}  &= \left( \ftmat_T \otimes \ftmat_{G} \right)_{(i*j)}^{km} x_{k}  \\
%     \hat{x} &= \ftmat_D x
% \end{align}
% \begin{align} \label{ctdft_XG}
%     \left( \mathbf{EFT(\{G_t\}; X)} \right)_{i}^{j} = \left( \hat{X}_G \right)_i^j &= \left( \ftmat_{G_t} X \right)_{i}^{kk} \left( \ftmat_T^{\top} \right)_{k}^{j}  \\
%     % \hat{X}_G &= \left( \ftmat_{G_t} X \right)_{ki}^{k} \left( \ftmat_T^T \right)_{j}^{k}  \\
%     \left( \hat{X}_G \right)_{i}^{j} = \hat{x}_{j*N + i}  &= \left( \ftmat_T \otimes \{\ftmat_{G_t}\} \right)_{(j*N + i)}^{km} x_{k}  \\
%     \hat{x} &= \ftmat_D x
% \end{align}
\begin{align} \label{ctdft_XG}
\small
    \left( \mathbf{EFT}(\{G_t\}; X) \right)_{i}^{j} = \left( \hat{X}_G \right)_i^j &= \left( \ftmat_{G_t} X \right)_{i}^{kk} \left( \ftmat_T^{\top} \right)_{k}^{j} = \left( \ftmat_T \otimes \{\ftmat_{G_t}\} \right)_{(j*N + i)}^{km} x_{k}
\end{align}
Thus, we have $ \hat{x}_{j*N + i} =  \left( \ftmat_T \otimes \{\ftmat_{G_t}\} \right)_{(j*N + i)}^{km} x_{k}$ or $\hat{x} = \ftmat_D x$. In the above equations, $\hat{X}_G$ is the \emph{EFT} of signal $X$ over dynamic graph $\{G_t\}$, $x, \hat{x} \in R^{NT}$ are the columnwise vectorized form of $X, \hat{X}_G \in R^{N \times T}$ and 
% $m = \floor{\frac{k}{N}}{}$.
$m = \left\lfloor \frac{k}{N} \right\rfloor$.
$\ftmat_D \in R^{NT \times NT}$ is the \emph{EFT} matrix over dynamic graph $\{G_t\}$ with 
% $\left( \ftmat_D \right)_{ij} = \left( \ftmat_T \otimes \ftmat_{G} \right)_{i}^{j \floor{\frac{j}{N}}$.
$\left( \ftmat_D \right)_{i}^{j} = \left( \ftmat_T \otimes \ftmat_{G} \right)_{i}^{j \left\lfloor \frac{j}{N} \right\rfloor}$.

% Thus, we get the \emph{EFT} as   
%For simplicity of implementation and as another form of proposed transformation, we can decompose the \emph{EFT} into time and vertex frequency transforms as 
%\begin{align}\label{jgft}
 %   \mathbf{EFT}(\{G_t\}; X) = \mathbf{\hat{G}} \mathbf{\ftmat_T^{\top}}, \quad \mathbf{\hat{G}^t} = \mathbf{\ftmat_{G_t}} X_t
%\end{align}
%$\mathbf{\hat{G}} \in R^{N \times T}$ and $\mathbf{\hat{G}^t}$ represents the $t-$th column of $\mathbf{\hat{G}}$. $G_t$ is the graph at time $t$ and $X_t$ is the signal on $G_t$. 
%To illustrate, one can imagine \emph{EFT} as performing the GFT of the graph at each timestep. 
%Then, the DFT is applied node-wise along the time dimension to the result.

%We remark that in the special case where the graph structure remains invariant with time (but signals on the graph may vary), \eft is same as \aeft.

We remark from equation \ref{jgft_einsum} of \eft, that the following desirable properties (over the exact eigendecompostion of the joint laplacian) are satisfied: 1) The equation imparts interpretibility to the frequency components, whether belonging to the time or vertex domain, as compared to the exact eigendecomposition. This is possible because we are able to decompose the transform into the individual transforms of each domain. 2) The transform equation is computationally efficient as compared to the exact eigendecomposition of the joint laplacian. 
% The computational complexity of \eft for the dynamic graph ($T$ timesteps) is $\mathcal{O}(T + T \log(T))$ as compared to $\mathcal{O}(T^3)$.
Specifically \eft reduces the computational complexity for the dynamic graph ($T$ timesteps) from a factor of $\mathcal{O}(T^3)$ to $\mathcal{O}(T + T \log(T))$.

\textcolor{asparagus}{Having derived the \eft transform, we state and prove its properties in the appendix \ref{sec:properties}.}
% We see in property \ref{prop_equivalence_dft}, that $\emph{EFT}$ can be simulated by $\emph{GFT}$ in the special case that the graph structure does not change with time.
% The illustration between other transforms is in Figure \ref{fig:ctdft_motivation}. 
The illustration between \eft and other transforms is in Figure \ref{fig:ctdft_motivation}. 
%The above as well as additional equivalences between various transforms as noted by \citep{loukas2016frequency} are given in the figure \ref{fig:ctdft_motivation}. 
The figure shows transforms (\emph{GFT, JFT, DFT, EFT)} in a circle, and arrows from one transform to the next indicate that the source transform can be obtained by the destination transform using the simulation annotated on the edges.
{
\color{asparagus}
For example, the \emph{GFT} of a ring graph ($\mathcal{T}$) gives the \emph{DFT} and thus the DFT can be simulated by \emph{GFT} using graph $\mathcal{T}$. Similarly \emph{DFT} can be simulated by \eft when the number of nodes $N=1$. 
Also the \emph{GFT} of the temporal ring of a static graph (topologically equivalent to a torus), where the nodes and edges remain constant with time, gives the \eft and vice versa (when time $T=1$). However when the graph structure changes with time \emph{GFT} cannot be used to simultae \eft.
% We also remark that the \eft of the ring of ring graphs (topologically equivalent to a torus) gives the \emph{GFT}. 
Thus, we can also look at the \eft as a generalization of the previous transforms.
}
% We explain the task specific implementation of these modules in the appendix \ref{sec:architecture_apndx} and focus more on the representations and results in the following sections. 
We briefly explain the task specific implementation of these modules in the below subsection and focus more on the representations and results in the following sections. 

\subsection{Implementation Details}
{\color{asparagus}
Having obtained the representations using the proposed transform, we intend to perform filtering in spectral space for dynamic graphs. Since our idea is to perform collective filtering along the vertex and temporal domain in \emph{EFT}, we need two modules to compute $\ftmat_{G_t}$ (vertex aspect) and $\ftmat_T$ (temporal aspect), respectively, in equation \ref{jgft_einsum} of \emph{EFT}. 
We now briefly explain these modules with details in appendix \ref{sec:architecture_apndx}.

\noindent \textbf{Filtering along the Vertex Domain:}
This module computes the convolution matrix $\ftmat_{G_t}$ in equation \ref{jgft_einsum}.
% Consider the filter response $\hat{\Lambda}_l$ which is a diagonal matrix with diagonal values representing the magnitude of the corresponding frequency(eigenvalue).
% In order to avoid the computational cost of the eigendecomposition, we choose to approximate the it using polynomials. In this work, we use the Chebyshev polynomials \citep{defferrard2016convolutional}. 
% Specifically, 
The frequency response of the desired filter is approximated as $\hat{\Lambda}_l = \sum_{k=0}^{\textcolor{black}{O_f}} c_k T_k(\Tilde{\Lambda})$,
where $O_f$ is the polynomial/filter order, $T_k$ is the Chebyshev polynomial basis, $\Tilde{\Lambda} = \frac{2\Lambda}{\lambda_{max}} - I$, $\lambda_{max}$ is the maximum eigenvalue and $c_k$ is the corresponding \textit{filter coefficients}. 
% Thus, we can approximate the filtering operation as:
% $X \ast \Lambda_l \approx U \left( \sum_{k=0}^{\textcolor{black}{O_f}} c_k T_k(\Tilde{\Lambda_l}) \right) U^*X = \sum_{k=0}^{\textcolor{black}{O_f}} c_k T_k(U \Tilde{\Lambda_l} U^*) X = \sum_{k=0}^{\textcolor{black}{O_f}} c_k T_k(\Tilde{L}) X$.
% Having the filter coefficients $c_k$ as learnable parameters enables learning of filter for the task. 
% % The convolution $X \ast \Lambda_l$ gives the desired filtered response in the vertex domain. 
The convolution of the graph signal $X$ with the filter ($X \ast \Lambda_l$) gives the desired filter response in the vertex domain. 

\noindent \textbf{Filtering along the Temporal Domain:}
After performing filtering in the vertex domain, we aim to filter over the temporal signals using  $\ftmat_T$ as in equation \ref{jgft_einsum}.
% To apply the $\ftmat_T$ (Fourier transform), we must first ensure that the signals in sequences are sampled at uniform intervals. 
% In the continuous time setting, interactions between nodes could occur at anytime or the sampling could be non-uniform, 
% Thus, we perform a mapping from $R^{T \times d} \xrightarrow{} R^{T \times d}$ that aims to map the input space to a uniformly sampled space. For computational reasons, we select the current and next embeddings (with positional information) along with the timestamp information ($E_t(t) \in R^{d}$) for getting the mapped embedding akin to interpolation. 
Formally, let $X_{t} \in R^{d}$ be the signal of a node at time $t$. 
% This is first mapped to the interpolated space using a universal approximator: $X_{t} = W^i_2 \sigma^i (W^i_1 [X_{t}^i;X_{t+1}^i;E_t(t)] + b^i_1) + b^i_2$, 
% where $W^i_1$, $W^i_2$, $b^i_1$, $b^i_1$ are learnable parameters and $\sigma^i$ is a non-linearity. We call this module the \emph{time encoding layer} 
% , which is essential for applying Fourier transform along the temporal dimension. 
Let $X = \{X_t\} \in R^{T \times d}$ be the time ordered matrix of embeddings of the node. This is converted to the frequency domain ($\hat{X} \in R^{T \times d}$) using the matrix $\ftmat_T$ as $\hat{X} = \ftmat_T X$.
Then we multiply $\hat{X}$ element-wise by a temporal filter $F_T \in R^{T \times d}$ to obtain the filtered signal $\hat{X}_f = F_T \odot \hat{X}$ which is then converted back to the temporal domain by using the inverse transform $\ftmat_T^*$ to get $X_f = \ftmat_T^* \hat{X}_f$.
% $X_f$ is the equivalent of $\hat{X}_G$ in equation \ref{ctdft_XG} that is the output of \emph{EFT}. 
$X_f$ is the filtered signal in the time-vertex domain of the dynamic graph. 
% In practice, the fast Fourier transform is used that can perform the computations in order $\mathcal{O}(T log(T))$. 
% Hence, overall time complexity of the architecture is $O((N+E)T + NTlogT)$.
% % To map the output back to the original space from the interpolated space we would need further mapping layers. Similar to \citep{FMLPRec}, we use the standard layer normalization (LN) and feedforward (FFN) layers: $X_{F} = \text{LN}\left(\text{LN}\left(X_t + \text{D}(X_f)\right) + \text{D}\left(\text{FFN}\left(\text{LN}\left(X_t + \text{D}(X_f)\right)\right)\right)\right)$, where 
% % $W^f_2, W^f_1, b^f_1, b^f_2$ are learnable parameters and D(.) represents dropout. 
% We could stack filter layers with the node embeddings obtained from previous layers as inputs. $X_{F}$ is the final filtered signal that is used in the downstream tasks. 
% % For the concerned node $n$ we denote this as $X_{F}^n$. 

}

\section{Experimental Setup}\label{sec:experiment}
%In this section we present the results of the experiments performed over three real world sequential recommendation datasets. Furthermore to show the extensibility of our method we also run experiments on three session based recommendation datasets where the users are anonymous. We 
 % For a fair comparison, we inherit experiment settings from the previous best dynamic graph method DGSR \citep{DGSR}. 
\textbf{Model Implementation and Datasets:} Considering \emph{EFT} is a spectral transform, we need a base model to induce \emph{EFT} in it. We select transformer as the base model inspired from \citep{FMLPRec,bastos2022how} that induce learnable filters into a vanilla transformer for downstream tasks (implementation is inspired from \citep{FMLPRec}, hence, details are in appendix). To illustrate the efficacy of the representations obtained from \emph{EFT}, we consider eight datasets. We name our model \emph{EFT-T}. The first three (Amazon Beauty, Games, CD in Table \ref{tab:dataset_stats}) are large continuous time dynamic graph datasets from sequential recommendation (SR) setting \citep{huang2023temporal}, spread over two decades. We inherit these datasets, dynamic graph construction process in SR setting, and metric from \citep{DGSR}. Other datasets \citep{evolvegcn} (UCI, AS, SBM, Elliptic, Brain) are standard (discrete) dynamic graph datasets to understand the generalizability of our method and contain a sequence of time-ordered graphs. 
% SBM is a synthetic dataset to simulate evolving community structures. UCI dataset is a student community network where nodes represent the students, and the edges represent the messages exchanged between them. AS dataset summarizes a temporal communication network indicating traffic flow between routers. 
% The Elliptic (Ell) dataset delineates legitimate versus unlawful transactions within the elliptic network of Bitcoin transactions. 
% In this context, nodes symbolize individual transactions, while edges correspond to the pathways of monetary transfers. 
% The Brain (Brn) dataset focuses on nodes representing minuscule cerebral regions or cubes, with the edges signifying their interconnections. 
Details on datasets, metrics, and experiment settings are in Appendix (cf., Table  \ref{tab:dataset}). 
% Experiment code and associated datasets are on Github: \url{https://github.com/nadgeri14/ICLR_EFT}. 
Experiment code and associated datasets are on Github: \url{https://github.com/ansonb/EFT}. 

\textbf{Baselines:} We use baselines depending on the experiment setting for fairness. For SR link prediction, we use strong baselines from previous best \citep{DGSR}: BPR-MF \citep{rendle2009bpr}, FPMC \citep{rendle2010factorizing}, GRU4Rec+ \citep{GRUrec+}, Caser \citep{tang2018personalized}, SASRec \citep{Sasrec}, HGN \citep{ma2019hierarchical}, TiSASRec \citep{TiSASRec}, SRGNN \citep{SRGNN}, HyperRec \citep{Hyperrec}, FMLPRec \citep{FMLPRec}, and DGSR \citep{DGSR}. For link prediction, node classification on discrete dynamic graph datasets, we rely on state of the art approaches of this setting \citep{ledg}: GCN \citep{kipf2016gcn}, GAT \citep{velivckovic2018graph}, GCN-GRU \citep{evolvegcn}, DynGEM \citep{DBLP:journals/corr/abs-1805-11273}, GAEN \citep{shi2021gaen}, EvolveGCN \citep{evolvegcn}, dyngraph2vec (dg2vec) \citep{DBLP:journals/kbs/GoyalCC20}. 

\section{Results and Discussion} \label{sec:evaluation}
This section reports the various experiment results, supporting our theoretical contributions.

% \begin{wrapfigure}{r}{0.42\textwidth}
%     \begin{center}\vspace{-5mm}
%     \includegraphics[width=0.4\textwidth]{images/recon_error_plot.png}
%     \end{center}\vspace{-2mm}
%     \caption{Reconstruction error on the synthetic datasets comparing EFT with the existing methods with varying fractions of the filtered frequency components. We can clearly see EFT outperforms the existing methods with a less reconstruction error.\vspace{-5mm}}
%     \label{fig:recon_error}
% \end{wrapfigure}

\begin{wrapfigure}{l}{0.28\textwidth}
    % \begin{center}\vspace{-5mm}
    \begin{center}\vspace{-6mm}
    \includegraphics[width=0.27\textwidth]{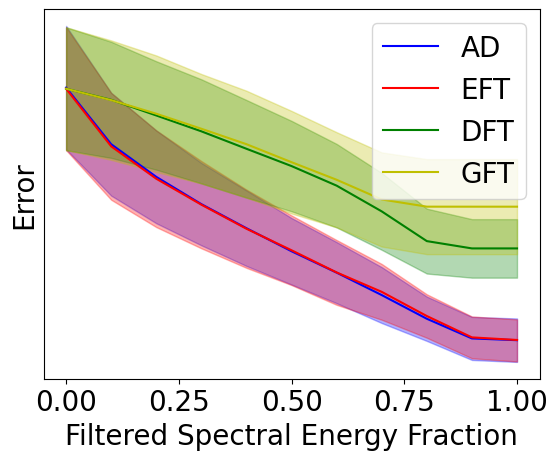}
    \end{center}\vspace{-2mm}
    % \caption{Reconstruction error with perturbation on the synthetic dataset.\vspace{-7mm}}
    % \caption{Reconstruction error with noisy synthetic dataset.\vspace{-2mm}}
    % \caption{Reconstruction error on noisy synthetic data.\vspace{-7mm}}
    \caption{Reconstruction error on noisy synthetic data.\vspace{-2mm}}
    \label{fig:recon_error}
\end{wrapfigure}
\noindent \textbf{Denoising and reconstruction on synthetic dataset with perturbation:}
Here, we aim to study whether \eft can better filter out noise from a dynamic graph than
DFT \citep{sundararajan2023discrete} and GFT \citep{ortega2018graph}.
% \begin{wrapfigure}{l}{0.42\textwidth}
%     \begin{center}\vspace{-5mm}
%     \includegraphics[width=0.4\textwidth]{images/recon_error_plot.png}
%     \end{center}\vspace{-2mm}
%     \caption{Reconstruction error on the synthetic datasets comparing EFT with the existing methods with varying fractions of the filtered frequency components. We can clearly see EFT outperforms the existing methods with a less reconstruction error.\vspace{-5mm}}
%     \label{fig:recon_error}
% \end{wrapfigure}
% \begin{wrapfigure}{l}{0.28\textwidth}
%     \begin{center}\vspace{-5mm}
%     \includegraphics[width=0.27\textwidth]{images/recon_error_plot.png}
%     \end{center}\vspace{-2mm}
%     \caption{Reconstruction error with perturbation on the synthetic dataset.\vspace{-2mm}}
%     \label{fig:recon_error}
% \end{wrapfigure}
The graphs are generated by sampling edge weights from a random normal distribution and evolved by perturbing the edge weights from the previous timestep. The graph signals are sampled from the eigenvectors of the graphs at each timestep, while the temporal signals are sampled from a sinusoidal signal. To add an element of complexity and realism, noise is induced along both the graph vertex and time signals (details in appendix \ref{sec:synthetic_dataset_apndx}). As a result, the dynamic graph signals evolve with time while being induced with noise along both dimensions. 
%Our hypothesis suggests that \eft transformative approach will result in superior denoising and signal reconstruction capabilities compared to conventional methods such as GFT or DFT, which are limited to filtering in just one dimension.
We hypothesize that using \eft, which transforms collectively across time and vertex dimensions, will result in better denoising and signal reconstruction compared to using GFT or DFT, which only performs filtering in one dimension. 
Our hypothesis is confirmed in Figure \ref{fig:recon_error}, 
which shows a decrease in error as the spectral energy of the signal is preserved while noise is filtered. 
Moreover, \eft yields comparable results to absolute transform (\aeft) while requiring less computational resources.
\begin{table*}[!t] 
% \small{

\caption{For link prediction on large temporal graphs of sequential recommendation setting, \label{tab:SR_results} table shows our model comparison (EFT-T) on the metrics \textit{Recall@10} and \textit{NDCG@10}. The best results are shown in boldface. The second best result has been underlined. 
% $ * $, $ ** $, $ *** $ indicate the statistical significance for $ p <= 0.05 $, $ p <= 0.01 $, and $ p <= 0.001 $, respectively, compared to the best baseline method based on the paired t-test. 
% \textit{Gain} shows the improvement of our method over the best-performing baseline which is statistically significant with p < 0.05.
The improvement of our method over the best-performing baseline is statistically significant with p < 0.05.
}
% \begin{tabular}{|c c c| c c c c c c| c c c|l| c|}
\adjustbox{max width=\textwidth}{
\begin{tabular}{|c c c c c c c c c c|l|}
\hline
& \textbf{GRU4Rec+} & \textbf{Caser} & \textbf{SASRec} & \textbf{HGN} & \textbf{TiSASRec} & \textbf{FMLPRec} & \textbf{SRGNN} & \textbf{HyperRec} & \textbf{DGSR} & \textbf{EFT-T} \\\hline
\multicolumn{10}{|c|}{Recall@10} \\
\hline
\textit{Beauty} & 43.98 & 42.64 & 48.54 & 48.63 & 46.87 & 47.47 & 48.62 & 34.71 & \underline{52.40} & \textbf{53.23} \\
% \hline
\textit{Games} & 67.15 & 68.83 & 73.98 & 71.42 & 71.85 & 73.62 & 73.49 & 71.24 & \underline{75.57} & \textbf{77.78} \\
% \hline
\textit{CDs} & 67.84 & 61.65 & 71.32 & 71.42 & 71.00 & 72.41 & 69.63 & 71.02 & \underline{72.43} & \textbf{75.42}  \\
% \hline
\hline
\multicolumn{10}{|c|}{NDCG@10} \\ 
\hline
\textit{Beauty} & 26.42 & 25.47 & 32.19 & 32.47 & 30.45 & 32.38 & 32.33 & 23.26 & \underline{35.90} & \textbf{37.10}\\
% \hline
\textit{Games} & 45.64 & 45.93 & 53.60 & 49.34 & 50.19 & 51.26 & 53.35 & 48.96 & \underline{55.70} & \textbf{58.65}\\
% \hline
\textit{CDs} & 44.52 & 45.85 & 49.23 & 49.34 & 48.97 & 53.31 & 48.95 & 47.16 & \underline{51.22} & \textbf{54.99}\\
\hline
\end{tabular}
}
% \vspace{-0.64cm}
% \vspace{-0.9cm}

% }
\end{table*}

\noindent \textbf{Compactness of \eft}:
Compaction refers to the ability of the transform to summarize the data compactly. A transform with good compaction is desirable as it summarizes the signals well in the frequency components, which can be used for efficient processing by downstream models. In this experiment, we verify the compaction properties of the proposed transform for the time-vertex frequencies on the temporal mesh graphs \citep{grassi2017time} concerning GFT and DFT. In order to test this, we remove varying percentile of the frequency components from the transformed frequency domain of signal $X$. We then apply the inverse transform to obtain the signal $X_r$. We plot the error $\frac{\norm{X-X_r}_F}{\norm{X}_F}$ vs the percentile of components removed. 
% From figure \ref{fig:recon_error_compactness} we can see that \eft has a lower error and better compaction and thus is able to summarize the data better than the baselines that only transform along a single dimension of vertex or time. 
From figure \ref{fig:compactness_dog}, \ref{fig:compactness_dancer} we can see that \eft has a lower error and better compaction and thus is able to summarize the data better than the baselines that only transform along a single dimension of vertex or time. 
\begin{wrapfigure}{l}{0.77\textwidth}
     \begin{subfigure}[b]{0.25\textwidth}
         \centering
         \includegraphics[width=\textwidth]{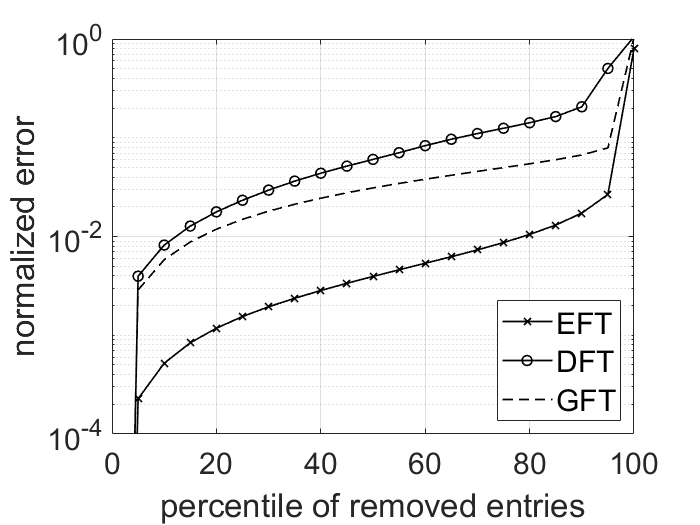}
         \caption{$Dog$}
         \label{fig:compactness_dog}
     \end{subfigure}
     \hfill
     \begin{subfigure}[b]{0.25\textwidth}
         \centering
         \includegraphics[width=\textwidth]{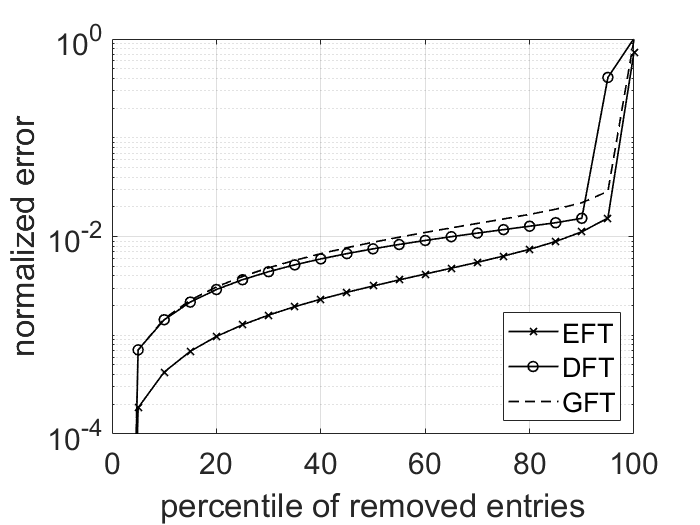}
         \caption{$Dancer$}
         \label{fig:compactness_dancer}
     \end{subfigure}
     \hfill
     \begin{subfigure}[b]{0.25\textwidth}
         \centering
         \includegraphics[width=\textwidth]{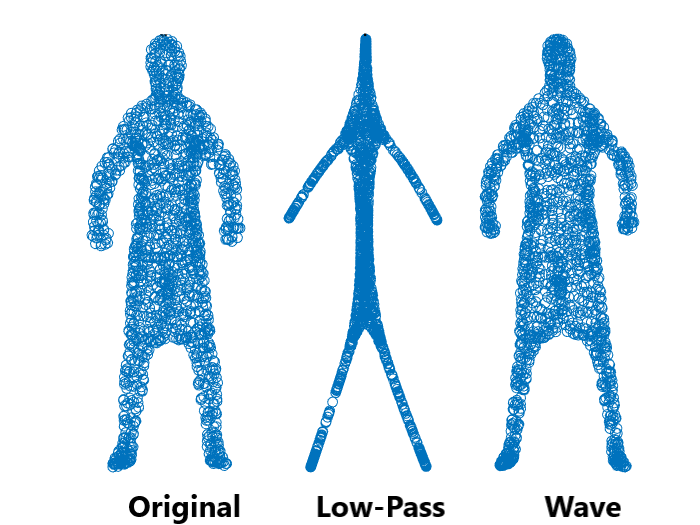}
         \caption{$Dancer$}
         \label{fig:hand_wavy}
     \end{subfigure}
    % \begin{center}\vspace{-5mm}
    % \includegraphics[width=0.27\textwidth]{images/ctdft_fig6_dog.png}
    % \end{center}\vspace{-2mm}
    \caption{Representations on dynamic mesh datasets. Left (a,b): Reconstruction error on the datasets illustrating the compactness of \eft. Right (c): Illustration of filtering using \eft on the dynamic mesh of a Dancer.
    \vspace{-3mm}
    }
    \label{fig:recon_error_compactness}
\end{wrapfigure}

% % \noindent \textbf{Illustration: filtering of dynamic mesh}
% \begin{wrapfigure}{l}{0.28\textwidth}
%     \begin{center}\vspace{-5mm}
%     \includegraphics[width=0.27\textwidth]{images/ctdft_fig7.png}
%     \end{center}\vspace{-2mm}
%     \caption{Illustration of filtering using \eft on a dynamic mesh dataset.\vspace{-2mm}}
%     \label{fig:hand_wavy}
% \end{wrapfigure}
\noindent \textbf{Illustration of filtering on temporal mesh}
Figure \ref{fig:hand_wavy} shows an example of collective filtering of a dynamic mesh representing a dancer \citep{grassi2017time}.
% \footnote{http://research.microsoft.com/en-us/um/redmond/events/geometrycompression/data/default.html}
Similar to \citet{grassi2017time}, we implement the following filters: (a) a low-pass filter that jointly attenuates high frequency components of the dynamic graph, and (b) a wave filter whose frequency response is described in Eq. (19) of \citet{grassi2017time}. The former filter gives us the frame of the mesh with stiff manoeuvers, whereas the fluid filter produces fluid movements. This experiment shows that \eft can enhance the frequency components non-linearly. This also hints towards why \eft performs better on evolving temporal graphs in subsequent experiments. 

%%%%%%%%%%%%%%%%%% deft table %%%%%%%%%%%%%%%%%%%%%%%%%%%%%%%%%%%%%%%%%%%%%%%
\begin{wraptable}{r}{0.64\textwidth}
% \begin{table}[ht] \small
% \vspace{-5pt}
    % \vspace{-24pt}
	\centering
	% \caption{Link prediction results where mean average precision (MAP) and mean reciprocal rank (MRR) are displayed. Best values are bold, second bests are underlined.}
    % \caption{Link prediction results where MAP and MRR are displayed. Best values are in bold and second bests are underlined.}
    % \caption{Results for Link Prediction (UCI, SBM, AS) on MAP and MRR; F1 for Node Classification (Brn, Ell) tasks. Best values are in bold and second bests are underlined.}
    \caption{Results for Link Prediction (UCI, SBM, AS) and Node Classification (Brn, Ell) tasks. Best values are in bold and second bests are underlined.}
    \vspace{-7pt}
	% \scalebox{0.8}{
 \resizebox{0.62\textwidth}{!}{
 % \resizebox{0.6\textwidth}{!}{
	\begin{tabular}{c|cc|cc|cc|c|c}
	\hline
	\multirow{2}{*}{}Datasets&
	\multicolumn{2}{c|}{SBM}&
	\multicolumn{2}{c|}{UCI}&
	\multicolumn{2}{c|}{AS}&Ell&Brn\cr\cline{2-9}
	Metrics&MAP&MRR&MAP&MRR&MAP&MRR&F1&F1
	\cr
% 	\hline
	\hline
 % 	Transformer & 0.2052 & 0.0174 & 0.0308 &  0.1441 & 0.4844 & 0.6439 \cr
	% GT & 0.2166 & 0.01805 & 0.0310 &  0.1414 & 0.4169 & 0.6216\cr
	% SAN & 0.2143 & 0.0180 & 0.0388 &  0.1822 & 0.4650 & 0.6308 \cr
	GCN  & 0.189 & 0.014 & 0.000 &  0.047 & 0.002 & 0.181 & 0.434& 0.232 \cr
	GAT  & 0.175 & 0.013 & 0.000 &  0.047 & 0.020 & 0.139 &0.451 &0.121  \cr
	% \hline
% 	\hline
	DynGEM & 0.168 & 0.014 & 0.021 &  0.106 & 0.053 & 0.103 &0.502 &0.225  \cr
	GCN-GRU  & 0.190 & 0.012 & 0.011& 0.098 & 0.071 & 0.339 &\underline{0.575} &0.186  \cr
	dg2vec \tiny{V1} & 0.098 & 0.008 & 0.004 &  0.054 & 0.033 & 0.070 &0.464 &0.191  \cr
	dg2vec \tiny{V2} & 0.159 & 0.012 & 0.020 &  0.071 & 0.071 & 0.049 &0.442 &0.215   \cr
	% dySAT & \cr
    GAEN & 0.1828	& 0.008 & 0.000 &0.049 & 0.130 & 0.051 &0.492 &0.205  \cr
	EGCN-H & 0.195 & 0.014 & 0.013 &  0.090 & 0.153 & 0.363 &0.391 &0.225 \cr
	EGCN-O  & \underline{0.200} & 0.014 & 0.027 & 0.138 & 0.114 & 0.275 &0.544 &0.192  \cr
	LED-GCN & 0.196 & \underline{0.015} & \underline{0.032}   & \underline{0.163} & 0.193 & \underline{0.469} &0.471 &\underline{0.261}  \cr
	LED-GAT &0.182  & 0.012     & 0.026 & 0.149  & \underline{0.233} & 0.384 &0.503 &0.150        \cr
% 	\hline
	\hline
	% \textbf{DEFT-MLP\tiny{(ours)}} &0.1658  & 0.0124     & \textbf{0.0543} & \underline{0.1715}  & 0.1737	&0.3569      \cr 
	% \textbf{DEFT-GAT\tiny{(ours)}} &0.0966  & 0.0081     & \underline{0.0502} & 0.1702  & 0.0308 & 0.0928      \cr
	% \textbf{DEFT-T\tiny{(ours)}} &\textbf{0.2421} & \textbf{0.0220} & 0.0501   & \textbf{0.2007} & \textbf{0.5879} & \textbf{0.6471}\cr
       % EFT-MLP & \underline{0.2080} & \underline{0.0174} & \underline{0.0450} & \textbf{0.1898} & 0.0981 & 0.3383  \cr 
      %  EFT-GAT & 0.1933 & 0.0141 & 0.0429 & 0.1639 & \underline{0.3468} & \underline{0.4810}  \cr 
        % EFT-T & \textbf{0.2501} & \textbf{0.0237} & \textbf{0.0547} & \underline{0.1813} & \textbf{0.6724} & \textbf{0.6892}  \cr 
        EFT-T & \textbf{0.250} & \textbf{0.024} & \textbf{0.055} & \textbf{0.181} & \textbf{0.672} & \textbf{0.689} &\textbf{0.616} &\textbf{0.308}   \cr 
	\hline
	\end{tabular}
 }
	% \vspace{-7pt}
        \vspace{-4pt}
        % \vspace{10pt}
	\label{tab:performance_lp}
% \end{table}
\end{wraptable}
\noindent \textbf{Performance comparison on (continuous) large-scale temporal graph datasets:}
The results on the large-scale SR datasets are in Table \ref{tab:SR_results} and \emph{EFT-T} outperforms baselines on all datasets. We note that our gains to the best baseline are higher in CDs, followed by the Games and Beauty dataset. We observe that as the density of the graph and length of sequences in the data increases (e.g., CD dataset), the performance of \emph{EFT-T} enhances. We believe that as graph density increases, higher-order connections may encompass noisy relations, a challenge conventional baselines struggle to filter out, whereas our method effectively handles this noise. Also, \eft effectively captures global interactions as it considers the temporal aspect in the collective filtering module. 
% Furthermore, compared to the FMLPRec model that induces DFT into a transformer, \emph{EFT-T} performs significantly better, concluding the necessity of capturing evolving spectra of temporal graphs. 
% The pure sequence methods(TiSASRec, HGN) do not perform well in this setting, demonstrating the graph-based methods' effectiveness.
% \noindent \textbf{Performance comparison on (continuous) large-scale temporal graph datasets:}
% The results on the large-scale SR datasets are in Table \ref{tab:SR_results} and \emph{EFT-T} outperforms baselines on all datasets. We note that our gains to the best baseline are higher in CDs, followed by the Games and Beauty dataset. We observe that as the density of the graph and length of sequences in the data increases (e.g., CD dataset), the performance of \emph{EFT-T} enhances. We believe that with increasing graph density, higher-order connections may contain noisy relations, and the baselines cannot filter out this noise, which our method does successfully. Also, \eft effectively captures global interactions as it considers the temporal aspect in the collective filtering module. 
Furthermore, compared to the FMLPRec model that induces DFT into a transformer, \emph{EFT-T} performs significantly better, concluding the necessity of capturing evolving spectra of temporal graphs. 
We also note that among the graph-based methods, SRGNN only considers connectivity information from the sequence graph, whereas HyperRec uses higher-order connectivity information. This indicates that not using the graph information effectively hampers performance but using higher-order connectivity without filtering to remove noise also degrades the results.

\noindent \textbf{Performance comparison on discrete temporal graph datasets:}
Table \ref{tab:performance_lp} summarizes link prediction and node classification results. Across datasets, our model significantly outperforms all baselines, which focus on learning local dependencies. 
It illustrates our framework's effectiveness in filtering noise and amplifying useful signals in evolving temporal graphs. 
%2) Interestingly, our model with a simple MLP predictor (EFT-MLP) on the filtered signals already achieves better results than most of the baselines. 3) The transformer variant (DEFT-T) gives consistently better results than the GAT and MLP variants indicating the aggregation using self-attention of transformers is helpful on these datasets. 
%These results verifies that collective filtering benefits performance on other settings of link prediction on the benchmark datasets (answering \textbf{RQ2}). 

\begin{figure}{l}[]
    \begin{center}\vspace{-5mm}
     \begin{subfigure}[b]{0.2\textwidth}
         \centering
         \includegraphics[width=\textwidth]{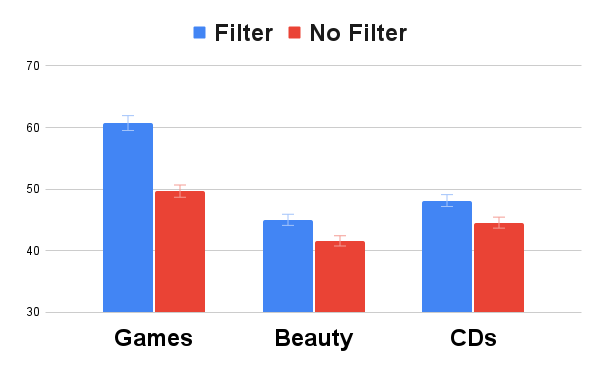}
         \caption{$Recall@10$}
         \label{fig:fee_r10}
     \end{subfigure}
     \hfill
     \begin{subfigure}[b]{0.2\textwidth}
         \centering
         \includegraphics[width=\textwidth]{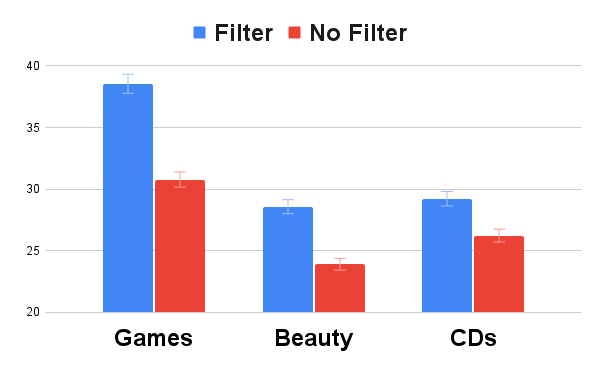}
         \caption{$NDCG@10$}
         \label{fig:fee_n10}
     \end{subfigure}
     \hfill
     \begin{subfigure}[b]{0.2\textwidth}
         \centering
         \includegraphics[width=\textwidth]{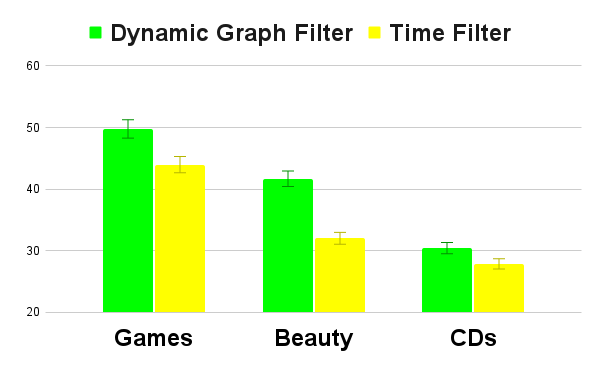}
         \caption{$Recall@10$}
         \label{fig:fee_r10}
     \end{subfigure}
     \hfill
     \begin{subfigure}[b]{0.2\textwidth}
         \centering
         \includegraphics[width=\textwidth]{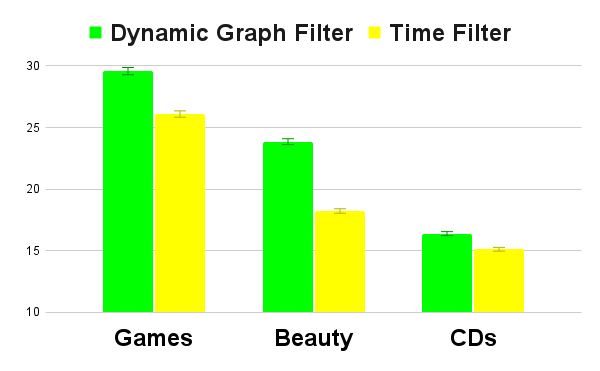}
         \caption{$NDCG@10$}
         \label{fig:fee_n10}
     \end{subfigure}
    \end{center}
    \vspace{-2mm}
    \caption{Effect of inducing 1) semantic noise in embeddings with and without filters (a-b) 2) structural noise in the form of graph perturbations with and without graph filters (c-d), on the performance of \eft. We consider large-scale SR setting.\vspace{-5mm}}
    \label{fig:filter_effectiveness}
\end{figure}
% \subsection{Effectiveness of filtering module (RQ3)}
\noindent \textbf{Effectiveness of filtering module (Figure \ref{fig:filter_effectiveness}):}
%We have seen that our method can perform well on the task of sequential recommendation. The learnable filters seem to boost the performance in the presence of implicit noise compared to the static filters, as is evident from table \ref{tab:ablation_results}. 
% Our approach focuses on capturing \emph{distant collaborative signals} while attenuating the noise. Hence, in this section, we aim to understand the effectiveness of the filters along both graphs (vertex) and time dimension in the presence of \textit{explicitly added noise}. 
Our approach focuses on capturing useful frequencies along vertex and time dimensions collectively while filtering the noise. Hence, in this experiment, we aim to understand the effectiveness of the filters along both graphs (vertex) and time dimension in the presence of \textit{explicitly added noise}. 
%Specifically, we intend to gauge the benefit of filters in the presence of semantic noise in the embeddings and structural noise induced by perturbing the graph.
% \begin{wrapfigure}{r}{0.42\textwidth}
\begin{wrapfigure}{r}{0.52\textwidth}
    \begin{center}
    % \vspace{-7mm}
    % \includegraphics[width=0.3\textwidth]{images/recon_error_plot_motivation.png}
     % \begin{subfigure}[b]{0.2\textwidth}
     \begin{subfigure}[b]{0.24\textwidth}
         \centering
         \includegraphics[width=\textwidth]{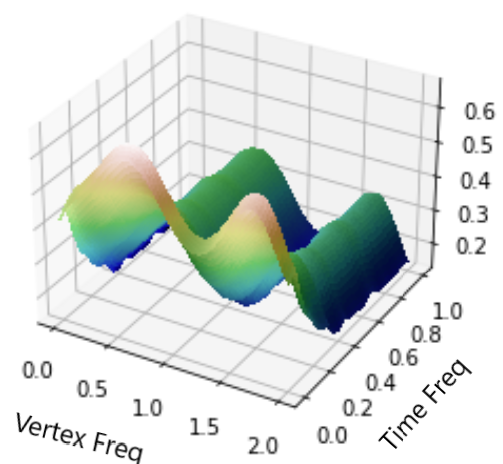}
         \caption{$Games$}
         \label{fig:filter_vis_games}
     \end{subfigure}
     \hfill
     % \begin{subfigure}[b]{0.2\textwidth}
     \begin{subfigure}[b]{0.24\textwidth}
         \centering
         \includegraphics[width=\textwidth]{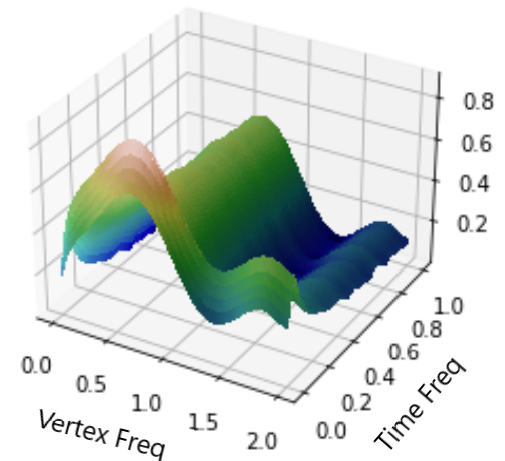}
         \caption{$CDs$}
         \label{fig:filter_vis_cd}
     \end{subfigure}
    \end{center}
    \vspace{-2mm}
    % \caption{Filter frequency responses learnt by \eft on dynamic graph datasets. The x-axis shows the graph frequency (0-2), y axis shows the normalized temporal frequency and z axis shows the magnitudes of the normalized frequency response. Extended view of these graphs are in appendix Figure \ref{fig:filter_vis_apndx}.
    % % \caption{Filter frequency responses 
    % \vspace{-5mm}}
    \caption{Filter frequency responses learnt by \eft on dynamic graph datasets. The x-axis shows the vertex frequency (0-2), y axis shows the normalized temporal frequency and z axis shows the magnitudes of the normalized frequency response.
    % \caption{Filter frequency responses 
    \vspace{-7mm}
    }
    \label{fig:filter_vis}
\end{wrapfigure}

Firstly, we induce \emph{semantic noise} into the system by adding a random vector (sampled from a normal distribution) to the node embeddings. 
% This simulates the scenario in sequential recommendation where a user randomly clicks a few items unrelated to the primary intention. 
Then, we run experiments on our model with and without learnable collective graph-time filters. 
To ensure a fair comparison, we keep the parameters in both models the same and simulate the no-filter configuration by using a uniform distribution for the frequency response (all-pass filter).
In the presence of noise, the performance of configuration with filters is much better ($p < 0.01$) than that without any filtering. 
Next, we induce \emph{structural noise} into the system by adding random nodes/edges.
We observe that on inducing structural noise, the performance of the configuration with graph filters is statistically better ($p<0.01$ using a paired t-test) compared to the one without, confirming that collective filtering is needed to be robust to structural noise in dynamic graphs. Additionally, we plotted the filter frequency responses of \eft on the Games and CDs datasets in Figure \ref{fig:filter_vis}. The figure shows dominating low-frequency response and higher-frequency components, indicating global aggregation for the long-range interactions.

\noindent

\section{Conclusion} \label{sec:conclusion}
In this paper, we introduce a novel approach to transform temporal graphs into the frequency domain, grounded on theoretical foundations. 
% Our method uses pseudo-spectrum relaxations to simplify the transformation, 
We propose pseudospectrum relaxations to the variational objective obtaining a simplified transformation, 
making it computationally efficient for real-world applications. 
% We demonstrate the effectiveness of our approach by providing error bounds for the exact solution and thoroughly studying its properties and application, specifically in the context of dynamic graph link prediction. 
We show that the error between the proposed transform and the exact solution to the variational objective is bounded from above and study its properties. We further demonstrate the practical effectiveness for temporal graphs. 
% {\color{\asparagus} 
% % Additionally, we recognize a known limitation related to invariance of the eigendecomposition to the direction of the basis vector. 
% In the current scope, we do not consider generic signed and directed graphs. To mitigate this, we suggest future works explore generalizing the Laplacian and the resulting transform to such graphs, leveraging techniques proposed in \citep{mercado2016clustering,cucuringu2021regularized}.}
{\color{asparagus} In the current scope, we do not consider generic signed and directed graphs. To mitigate this, we suggest future works explore generalizing the Laplacian and the resulting transform to such graphs, leveraging techniques proposed in \citep{mercado2016clustering,cucuringu2021regularized}.}
Our work opens up new possibilities for dynamic graph analysis and representation learning, and we encourage researchers to explore potential of \eft as a spectral representation of the evolving graph in downstream graph representation learning models.

% \bibliography{iclr2024_conference}
\bibliography{bibliography}

\begin{thebibliography}{62}
\providecommand{\natexlab}[1]{#1}
\providecommand{\url}[1]{\texttt{#1}}
\expandafter\ifx\csname urlstyle\endcsname\relax
  \providecommand{\doi}[1]{doi: #1}\else
  \providecommand{\doi}{doi: \begingroup \urlstyle{rm}\Url}\fi

\bibitem[Albert et~al.(1916)Albert, Perrett, and Jeffery]{albert1916foundation}
Einstein Albert, W~Perrett, and G~Jeffery.
\newblock The foundation of the general theory of relativity.
\newblock \emph{Ann. Der Phys}, 49:\penalty0 769--822, 1916.

\bibitem[Balcilar et~al.(2020)Balcilar, Renton, H{\'e}roux, Ga{\"u}z{\`e}re, Adam, and Honeine]{balcilar2020analyzing}
Muhammet Balcilar, Guillaume Renton, Pierre H{\'e}roux, Benoit Ga{\"u}z{\`e}re, S{\'e}bastien Adam, and Paul Honeine.
\newblock Analyzing the expressive power of graph neural networks in a spectral perspective.
\newblock In \emph{International Conference on Learning Representations}, 2020.

\bibitem[Bastas et~al.(2019)Bastas, Semertzidis, Axenopoulos, and Daras]{bastas2019evolve2vec}
Nikolaos Bastas, Theodoros Semertzidis, Apostolos Axenopoulos, and Petros Daras.
\newblock evolve2vec: Learning network representations using temporal unfolding.
\newblock In \emph{MultiMedia Modeling: 25th International Conference, MMM 2019, Thessaloniki, Greece, January 8--11, 2019, Proceedings, Part I 25}, pp.\  447--458. Springer, 2019.

\bibitem[Bastos et~al.(2022)Bastos, Nadgeri, Singh, Kanezashi, Suzumura, and Mulang']{bastos2022how}
Anson Bastos, Abhishek Nadgeri, Kuldeep Singh, Hiroki Kanezashi, Toyotaro Suzumura, and Isaiah~Onando Mulang'.
\newblock How expressive are transformers in spectral domain for graphs?
\newblock \emph{Transactions on Machine Learning Research}, 2022.
\newblock ISSN 2835-8856.
\newblock URL \url{https://openreview.net/forum?id=aRsLetumx1}.

\bibitem[Bastos et~al.(2023)Bastos, Nadgeri, Singh, Suzumura, and Singh]{bastos2023learnable}
Anson Bastos, Abhishek Nadgeri, Kuldeep Singh, Toyotaro Suzumura, and Manish Singh.
\newblock Learnable spectral wavelets on dynamic graphs to capture global interactions.
\newblock In \emph{Proceedings of the AAAI Conference on Artificial Intelligence}, volume~37, pp.\  6779--6787, 2023.

\bibitem[Blackledge(2005)]{conv_theorem}
Jonathan~M. Blackledge.
\newblock Chapter 2 - 2d fourier theory.
\newblock In \emph{Digital Image Processing}, Woodhead Publishing Series in Electronic and Optical Materials, pp.\  30--49. Woodhead Publishing, 2005.

\bibitem[Cao et~al.(2020)Cao, Wang, Duan, Zhang, Zhu, Huang, Tong, Xu, Bai, Tong, et~al.]{cao2020spectral}
Defu Cao, Yujing Wang, Juanyong Duan, Ce~Zhang, Xia Zhu, Congrui Huang, Yunhai Tong, Bixiong Xu, Jing Bai, Jie Tong, et~al.
\newblock Spectral temporal graph neural network for multivariate time-series forecasting.
\newblock \emph{Advances in neural information processing systems}, 33:\penalty0 17766--17778, 2020.

\bibitem[Cao et~al.(2021)Cao, Wang, Duan, Zhang, Zhu, Huang, Tong, Xu, Bai, Tong, and Zhang]{cao2021spectral}
Defu Cao, Yujing Wang, Juanyong Duan, Ce~Zhang, Xia Zhu, Conguri Huang, Yunhai Tong, Bixiong Xu, Jing Bai, Jie Tong, and Qi~Zhang.
\newblock Spectral temporal graph neural network for multivariate time-series forecasting, 2021.

\bibitem[Chen et~al.(2022)Chen, Wang, and Xu]{chen2022gc}
Jinyin Chen, Xueke Wang, and Xuanheng Xu.
\newblock Gc-lstm: Graph convolution embedded lstm for dynamic network link prediction.
\newblock \emph{Applied Intelligence}, pp.\  1--16, 2022.

\bibitem[Cheng et~al.(2023)Cheng, Chen, Lee, and Sun]{svd_gft_10195957}
Cheng Cheng, Yang Chen, Yeon~Ju Lee, and Qiyu Sun.
\newblock Svd-based graph fourier transforms on directed product graphs.
\newblock \emph{IEEE Transactions on Signal and Information Processing over Networks}, 9:\penalty0 531--541, 2023.
\newblock \doi{10.1109/TSIPN.2023.3299511}.

\bibitem[Cucuringu et~al.(2021)Cucuringu, Singh, Sulem, and Tyagi]{cucuringu2021regularized}
Mihai Cucuringu, Apoorv~Vikram Singh, D{\'e}borah Sulem, and Hemant Tyagi.
\newblock Regularized spectral methods for clustering signed networks.
\newblock \emph{The Journal of Machine Learning Research}, 22\penalty0 (1):\penalty0 12057--12135, 2021.

\bibitem[da~Xu et~al.(2020)da~Xu, chuanwei ruan, evren korpeoglu, sushant kumar, and kannan achan]{Xu2020Inductive}
da~Xu, chuanwei ruan, evren korpeoglu, sushant kumar, and kannan achan.
\newblock Inductive representation learning on temporal graphs.
\newblock In \emph{International Conference on Learning Representations}, 2020.

\bibitem[Defferrard et~al.(2016)Defferrard, Bresson, and Vandergheynst]{defferrard2016convolutional}
Micha{\"e}l Defferrard, Xavier Bresson, and Pierre Vandergheynst.
\newblock Convolutional neural networks on graphs with fast localized spectral filtering.
\newblock \emph{Advances in neural information processing systems}, 29:\penalty0 3844--3852, 2016.

\bibitem[Gass \& Fu(2013)Gass and Fu]{KKTref1}
Saul~I. Gass and Michael~C. Fu (eds.).
\newblock \emph{Karush-Kuhn-Tucker (KKT) Conditions}, pp.\  833--834.
\newblock Springer US, Boston, MA, 2013.
\newblock ISBN 978-1-4419-1153-7.

\bibitem[Goyal et~al.(2017)Goyal, Kamra, He, and Liu]{DBLP:journals/corr/abs-1805-11273}
Palash Goyal, Nitin Kamra, Xinran He, and Yan Liu.
\newblock Dyngem: Deep embedding method for dynamic graphs.
\newblock \emph{IJCAI Workshop on Representation Learning for Graphs}, 2017.

\bibitem[Goyal et~al.(2020)Goyal, Chhetri, and Canedo]{DBLP:journals/kbs/GoyalCC20}
Palash Goyal, Sujit~Rokka Chhetri, and Arquimedes Canedo.
\newblock dyngraph2vec: Capturing network dynamics using dynamic graph representation learning.
\newblock \emph{Knowl. Based Syst.}, 187, 2020.

\bibitem[Grassi et~al.(2017)Grassi, Loukas, Perraudin, and Ricaud]{grassi2017time}
Francesco Grassi, Andreas Loukas, Nathana{\"e}l Perraudin, and Benjamin Ricaud.
\newblock A time-vertex signal processing framework: Scalable processing and meaningful representations for time-series on graphs.
\newblock \emph{IEEE Transactions on Signal Processing}, 66\penalty0 (3):\penalty0 817--829, 2017.

\bibitem[Hammond et~al.(2011)Hammond, Vandergheynst, and Gribonval]{hammond2011wavelets}
David~K Hammond, Pierre Vandergheynst, and R{\'e}mi Gribonval.
\newblock Wavelets on graphs via spectral graph theory.
\newblock \emph{Applied and Computational Harmonic Analysis}, 30\penalty0 (2):\penalty0 129--150, 2011.

\bibitem[Hidasi \& Karatzoglou(2018)Hidasi and Karatzoglou]{GRUrec+}
Bal{\'a}zs Hidasi and Alexandros Karatzoglou.
\newblock Recurrent neural networks with top-k gains for session-based recommendations.
\newblock In \emph{Proceedings of the 27th ACM international conference on information and knowledge management}, pp.\  843--852, 2018.

\bibitem[Huang et~al.(2023)Huang, Poursafaei, Danovitch, Fey, Hu, Rossi, Leskovec, Bronstein, Rabusseau, and Rabbany]{huang2023temporal}
Shenyang Huang, Farimah Poursafaei, Jacob Danovitch, Matthias Fey, Weihua Hu, Emanuele Rossi, Jure Leskovec, Michael Bronstein, Guillaume Rabusseau, and Reihaneh Rabbany.
\newblock Temporal graph benchmark for machine learning on temporal graphs.
\newblock \emph{arXiv preprint arXiv:2307.01026}, 2023.

\bibitem[Jiang et~al.(2021)Jiang, Feng, Tay, and Xu]{joint_tv_filter_9374709}
Junzheng Jiang, Hairong Feng, David~B. Tay, and Shuwen Xu.
\newblock Theory and design of joint time-vertex nonsubsampled filter banks.
\newblock \emph{IEEE Transactions on Signal Processing}, 69:\penalty0 1968--1982, 2021.
\newblock \doi{10.1109/TSP.2021.3064984}.

\bibitem[Jing et~al.(2022)Jing, Zhu, Xu, Liu, Zang, Wang, and Yu]{TKDD-dynamic}
Mengyuan Jing, Yanmin Zhu, Yanan Xu, Haobing Liu, Tianzi Zang, Chunyang Wang, and Jiadi Yu.
\newblock Learning shared representations for recommendation with dynamic heterogeneous graph convolutional networks.
\newblock \emph{ACM Transactions on Knowledge Discovery from Data (TKDD)}, 2022.

\bibitem[Kang \& McAuley(2018)Kang and McAuley]{Sasrec}
Wang-Cheng Kang and Julian McAuley.
\newblock Self-attentive sequential recommendation.
\newblock In \emph{2018 IEEE international conference on data mining (ICDM)}, pp.\  197--206. IEEE, 2018.

\bibitem[Kartal et~al.(2022)Kartal, Özgünay, and Koç]{kartal2022joint}
Bünyamin Kartal, Eray Özgünay, and Aykut Koç.
\newblock Joint time-vertex fractional fourier transform, 2022.

\bibitem[Kazemi et~al.(2020)Kazemi, Goel, Jain, Kobyzev, Sethi, Forsyth, and Poupart]{kazemi2020representation}
Seyed~Mehran Kazemi, Rishab Goel, Kshitij Jain, Ivan Kobyzev, Akshay Sethi, Peter Forsyth, and Pascal Poupart.
\newblock Representation learning for dynamic graphs: A survey.
\newblock \emph{J. Mach. Learn. Res.}, 21\penalty0 (70):\penalty0 1--73, 2020.

\bibitem[Kipf \& Welling(2017)Kipf and Welling]{kipf2016gcn}
Thomas~N. Kipf and Max Welling.
\newblock Semi-supervised classification with graph convolutional networks.
\newblock In \emph{5th International Conference on Learning Representations, {ICLR} 2017}, 2017.

\bibitem[Kurokawa et~al.(2017)Kurokawa, Oki, and Nagao]{kurokawa2017multidimensional}
Takashi Kurokawa, Taihei Oki, and Hiromichi Nagao.
\newblock Multi-dimensional graph fourier transform, 2017.

\bibitem[Lancaster \& Farahat(1972)Lancaster and Farahat]{tensor_product_norms}
P.~Lancaster and H.~K. Farahat.
\newblock Norms on direct sums and tensor products.
\newblock \emph{Mathematics of Computation}, 26\penalty0 (118):\penalty0 401--414, 1972.
\newblock ISSN 00255718, 10886842.

\bibitem[Li et~al.(2020{\natexlab{a}})Li, Wang, and McAuley]{TiSASRec}
Jiacheng Li, Yujie Wang, and Julian McAuley.
\newblock Time interval aware self-attention for sequential recommendation.
\newblock In \emph{Proceedings of the 13th international conference on web search and data mining}, pp.\  322--330, 2020{\natexlab{a}}.

\bibitem[Li et~al.(2020{\natexlab{b}})Li, Zhang, Wu, Liu, Wang, and Philip]{DGCF}
Xiaohan Li, Mengqi Zhang, Shu Wu, Zheng Liu, Liang Wang, and S~Yu Philip.
\newblock Dynamic graph collaborative filtering.
\newblock In \emph{2020 IEEE International Conference on Data Mining (ICDM)}, pp.\  322--331. IEEE, 2020{\natexlab{b}}.

\bibitem[Loukas \& Foucard(2016)Loukas and Foucard]{loukas2016frequency}
Andreas Loukas and Damien Foucard.
\newblock Frequency analysis of time-varying graph signals.
\newblock In \emph{2016 IEEE Global Conference on Signal and Information Processing (GlobalSIP)}, pp.\  346--350. IEEE, 2016.

\bibitem[Ma et~al.(2019)Ma, Kang, and Liu]{ma2019hierarchical}
Chen Ma, Peng Kang, and Xue Liu.
\newblock Hierarchical gating networks for sequential recommendation.
\newblock In \emph{Proceedings of the 25th ACM SIGKDD international conference on knowledge discovery \& data mining}, pp.\  825--833, 2019.

\bibitem[Ma et~al.(2020)Ma, Guo, Ren, Tang, and Yin]{Ma2020streaming}
Yao Ma, Ziyi Guo, Zhaocun Ren, Jiliang Tang, and Dawei Yin.
\newblock Streaming graph neural networks.
\newblock In \emph{Proceedings of the 43rd international ACM SIGIR conference on research and development in information retrieval}, pp.\  719--728, 2020.

\bibitem[Mahyari \& Aviyente(2014)Mahyari and Aviyente]{mahyari2014fourier}
Arash~Golibagh Mahyari and Selin Aviyente.
\newblock Fourier transform for signals on dynamic graphs.
\newblock In \emph{2014 48th Asilomar Conference on Signals, Systems and Computers}, pp.\  2001--2004. IEEE, 2014.

\bibitem[Manessi et~al.(2020)Manessi, Rozza, and Manzo]{DBLP:journals/pr/ManessiRM20}
Franco Manessi, Alessandro Rozza, and Mario Manzo.
\newblock Dynamic graph convolutional networks.
\newblock \emph{Pattern Recognit.}, 97, 2020.

\bibitem[Mercado et~al.(2016)Mercado, Tudisco, and Hein]{mercado2016clustering}
Pedro Mercado, Francesco Tudisco, and Matthias Hein.
\newblock Clustering signed networks with the geometric mean of laplacians.
\newblock \emph{Advances in neural information processing systems}, 29, 2016.

\bibitem[Narayan \& Roe(2018)Narayan and Roe]{narayan2018learning}
Apurva Narayan and Peter~HO’N Roe.
\newblock Learning graph dynamics using deep neural networks.
\newblock \emph{IFAC-PapersOnLine}, 51\penalty0 (2):\penalty0 433--438, 2018.

\bibitem[Ortega et~al.(2018)Ortega, Frossard, Kova{\v{c}}evi{\'c}, Moura, and Vandergheynst]{ortega2018graph}
Antonio Ortega, Pascal Frossard, Jelena Kova{\v{c}}evi{\'c}, Jos{\'e}~MF Moura, and Pierre Vandergheynst.
\newblock Graph signal processing: Overview, challenges, and applications.
\newblock \emph{Proceedings of the IEEE}, 106\penalty0 (5):\penalty0 808--828, 2018.

\bibitem[Pan et~al.(2020)Pan, Chen, and Ortega]{pan2020spatio}
Chao Pan, Siheng Chen, and Antonio Ortega.
\newblock Spatio-temporal graph scattering transform.
\newblock In \emph{International Conference on Learning Representations}, 2020.

\bibitem[Pareja et~al.(2020)Pareja, Domeniconi, Chen, Ma, Suzumura, Kanezashi, Kaler, Schardl, and Leiserson]{evolvegcn}
Aldo Pareja, Giacomo Domeniconi, Jie Chen, Tengfei Ma, Toyotaro Suzumura, Hiroki Kanezashi, Tim Kaler, Tao~B. Schardl, and Charles~E. Leiserson.
\newblock Evolvegcn: Evolving graph convolutional networks for dynamic graphs.
\newblock In \emph{The Thirty-Fourth {AAAI} Conference on Artificial Intelligence, {AAAI} 2020}, pp.\  5363--5370. {AAAI} Press, 2020.

\bibitem[Paszke et~al.(2017)Paszke, Gross, Chintala, Chanan, Yang, DeVito, Lin, Desmaison, Antiga, and Lerer]{paszke2017pytorch}
Adam Paszke, Sam Gross, Soumith Chintala, Gregory Chanan, Edward Yang, Zachary DeVito, Zeming Lin, Alban Desmaison, Luca Antiga, and Adam Lerer.
\newblock Automatic differentiation in pytorch.
\newblock 2017.

\bibitem[Rendle et~al.(2009)Rendle, Freudenthaler, Gantner, and Schmidt-Thieme]{rendle2009bpr}
Steffen Rendle, Christoph Freudenthaler, Zeno Gantner, and Lars Schmidt-Thieme.
\newblock Bpr: Bayesian personalized ranking from implicit feedback.
\newblock In \emph{Proceedings of the Twenty-Fifth Conference on Uncertainty in Artificial Intelligence}, pp.\  452--461, 2009.

\bibitem[Rendle et~al.(2010)Rendle, Freudenthaler, and Schmidt-Thieme]{rendle2010factorizing}
Steffen Rendle, Christoph Freudenthaler, and Lars Schmidt-Thieme.
\newblock Factorizing personalized markov chains for next-basket recommendation.
\newblock In \emph{Proceedings of the 19th international conference on World wide web}, pp.\  811--820, 2010.

\bibitem[Rossi et~al.(2020)Rossi, Chamberlain, Frasca, Eynard, Monti, and Bronstein]{rossi2020temporal}
Emanuele Rossi, Ben Chamberlain, Fabrizio Frasca, Davide Eynard, Federico Monti, and Michael Bronstein.
\newblock Temporal graph networks for deep learning on dynamic graphs.
\newblock \emph{arXiv preprint arXiv:2006.10637}, 2020.

\bibitem[Sarkar et~al.(2012)Sarkar, Chakrabarti, and Jordan]{sarkar2012nonparametric}
Purnamrita Sarkar, Deepayan Chakrabarti, and Michael Jordan.
\newblock Nonparametric link prediction in dynamic networks.
\newblock \emph{arXiv preprint arXiv:1206.6394}, 2012.

\bibitem[Seo et~al.(2016)Seo, Defferrard, Vandergheynst, and Bresson]{DBLP:journals/corr/SeoDVB16}
Youngjoo Seo, Micha{\"{e}}l Defferrard, Pierre Vandergheynst, and Xavier Bresson.
\newblock Structured sequence modeling with graph convolutional recurrent networks.
\newblock \emph{CoRR}, abs/1612.07659, 2016.
\newblock URL \url{http://arxiv.org/abs/1612.07659}.

\bibitem[Shi et~al.(2021)Shi, Huang, Zhu, Tang, Zhuang, and Liu]{shi2021gaen}
Min Shi, Yu~Huang, Xingquan Zhu, Yufei Tang, Yuan Zhuang, and Jianxun Liu.
\newblock Gaen: Graph attention evolving networks.
\newblock In \emph{IJCAI}, pp.\  1541--1547, 2021.

\bibitem[Shuman et~al.(2013)Shuman, Narang, Frossard, Ortega, and Vandergheynst]{shuman2013emerging}
David~I Shuman, Sunil~K Narang, Pascal Frossard, Antonio Ortega, and Pierre Vandergheynst.
\newblock The emerging field of signal processing on graphs: Extending high-dimensional data analysis to networks and other irregular domains.
\newblock \emph{IEEE signal processing magazine}, 30\penalty0 (3):\penalty0 83--98, 2013.

\bibitem[Sundararajan(2023)]{sundararajan2023discrete}
Duraisamy Sundararajan.
\newblock The discrete fourier transform.
\newblock In \emph{Signals and Systems}, pp.\  125--160. Springer, 2023.

\bibitem[Tang \& Wang(2018)Tang and Wang]{tang2018personalized}
Jiaxi Tang and Ke~Wang.
\newblock Personalized top-n sequential recommendation via convolutional sequence embedding.
\newblock In \emph{Proceedings of the eleventh ACM international conference on web search and data mining}, pp.\  565--573, 2018.

\bibitem[Tao(2008)]{terry_pseudospectrum}
Terrence Tao.
\newblock When are eigenvalues stable?, 2008.
\newblock URL \url{https://terrytao.wordpress.com/2008/10/28/when-are-eigenvalues-stable/}.

\bibitem[Veli{\v{c}}kovi{\'c} et~al.(2018)Veli{\v{c}}kovi{\'c}, Cucurull, Casanova, Romero, Li{\`o}, and Bengio]{velivckovic2018graph}
Petar Veli{\v{c}}kovi{\'c}, Guillem Cucurull, Arantxa Casanova, Adriana Romero, Pietro Li{\`o}, and Yoshua Bengio.
\newblock Graph attention networks.
\newblock In \emph{International Conference on Learning Representations}, 2018.

\bibitem[Villafañe-Delgado \& Aviyente(2017)Villafañe-Delgado and Aviyente]{DGFT_icsap}
Marisel Villafañe-Delgado and Selin Aviyente.
\newblock Dynamic graph fourier transform on temporal functional connectivity networks.
\newblock In \emph{2017 IEEE International Conference on Acoustics, Speech and Signal Processing (ICASSP)}, pp.\  949--953, 2017.

\bibitem[Wang et~al.(2020)Wang, Ding, Hong, Liu, and Caverlee]{Hyperrec}
Jianling Wang, Kaize Ding, Liangjie Hong, Huan Liu, and James Caverlee.
\newblock Next-item recommendation with sequential hypergraphs.
\newblock In \emph{Proceedings of the 43rd international ACM SIGIR conference on research and development in information retrieval}, pp.\  1101--1110, 2020.

\bibitem[Wang et~al.(2019)Wang, Zheng, Ye, Gan, Li, Song, Zhou, Ma, Yu, Gai, Xiao, He, Karypis, Li, and Zhang]{wang2019dgl}
Minjie Wang, Da~Zheng, Zihao Ye, Quan Gan, Mufei Li, Xiang Song, Jinjing Zhou, Chao Ma, Lingfan Yu, Yu~Gai, Tianjun Xiao, Tong He, George Karypis, Jinyang Li, and Zheng Zhang.
\newblock Deep graph library: A graph-centric, highly-performant package for graph neural networks.
\newblock \emph{arXiv preprint arXiv:1909.01315}, 2019.

\bibitem[Wang \& Zhang(2022)Wang and Zhang]{wang2022powerful}
Xiyuan Wang and Muhan Zhang.
\newblock How powerful are spectral graph neural networks.
\newblock In \emph{International Conference on Machine Learning}, pp.\  23341--23362. PMLR, 2022.

\bibitem[Wu et~al.(2019)Wu, Tang, Zhu, Wang, Xie, and Tan]{SRGNN}
Shu Wu, Yuyuan Tang, Yanqiao Zhu, Liang Wang, Xing Xie, and Tieniu Tan.
\newblock Session-based recommendation with graph neural networks.
\newblock In \emph{Proceedings of the AAAI conference on artificial intelligence}, volume~33, pp.\  346--353, 2019.

\bibitem[Xiang et~al.(2022)Xiang, Huang, and Wang]{ledg}
Xintao Xiang, Tiancheng Huang, and Donglin Wang.
\newblock Learning to evolve on dynamic graphs (sa).
\newblock In \emph{Thirty-Sixth {AAAI} Conference on Artificial Intelligence, {AAAI} 2022}. {AAAI} Press, 2022.
\newblock URL \url{https://arxiv.org/pdf/2111.07032.pdf}.

\bibitem[Yan et~al.(2018)Yan, Xiong, and Lin]{yan2018spatial}
Sijie Yan, Yuanjun Xiong, and Dahua Lin.
\newblock Spatial temporal graph convolutional networks for skeleton-based action recognition.
\newblock In \emph{Proceedings of the AAAI conference on artificial intelligence}, volume~32, 2018.

\bibitem[Yu et~al.(2017)Yu, Yin, and Zhu]{yu2017spatio}
Bing Yu, Haoteng Yin, and Zhanxing Zhu.
\newblock Spatio-temporal graph convolutional networks: A deep learning framework for traffic forecasting.
\newblock \emph{arXiv preprint arXiv:1709.04875}, 2017.

\bibitem[Zhang et~al.(2022)Zhang, Wu, Yu, Liu, and Wang]{DGSR}
Mengqi Zhang, Shu Wu, Xueli Yu, Qiang Liu, and Liang Wang.
\newblock Dynamic graph neural networks for sequential recommendation.
\newblock \emph{IEEE Transactions on Knowledge and Data Engineering}, 2022.

\bibitem[Zhou et~al.(2022)Zhou, Yu, Zhao, and Wen]{FMLPRec}
Kun Zhou, Hui Yu, Wayne~Xin Zhao, and Ji-Rong Wen.
\newblock Filter-enhanced mlp is all you need for sequential recommendation.
\newblock In \emph{Proceedings of the ACM Web Conference 2022}, pp.\  2388--2399, 2022.

\end{thebibliography}
\bibliographystyle{iclr2024_conference}

% \appendix
% \section{Appendix}

\newpage
\begin{appendix}

\section{Preliminaries} \label{sec:preliminaryappendix}
% Graph Fourier Transform
Here we give an extended discussion of the preliminaries which could not be accommodated in the main paper due to space constraints.
%Experiment code and associated datasets are on Github: \url{https://github.com/nadgeri14/ICLR_EFT}. 

\subsection{Discrete Fourier Transform}\label{subsec:dft}
The Discrete Fourier Transform (DFT) is used to obtain the frequency representation of a sequence of signal values sampled at equal intervals of time. The magnitude and phase of the frequency components are obtained by multiplying the signal values by complex sinusoids of the respective frequencies. Consider a signal $x$ sampled at $N$ intervals of time $t \in [0, N-1]$ to obtain the sequence $\{x_t\}$. The DFT of $x_t$ is then given by,
\begin{equation}
    X_k = \sum_{t=0}^{N-1} x_t e^{-i \omega_t k}, \quad \omega_t=\frac{2 \pi t}{N}
\end{equation}
The transformed sequence $X_k$ gives the values of the signal in the frequency domain. If we represent $X$ as the vector form of the signal we can define the DFT matrix $\ftmat_T$ such that $X_k = \ftmat_T X$. Thus $X_k$ is the complex valued spectrum of $\{x_t\}$ at frequency $\omega_t$. We can perform filtering by removal of noisy frequencies in this spectral domain. As required by signal processing applications we can then obtain the signal sequence in the time domain from the frequency domain $\{X_k\}$ using the Inverse Discrete Fourier Transform (IDFT) as, 
\begin{equation}
    x_t = \frac{1}{N} \sum_{k=0}^{N-1} X_k e^{i \omega_k t}, \quad \omega_k=\frac{2 \pi k}{N}
\end{equation}

% Fourier transform
\subsection{Graph Fourier Transform}\label{subsec:gft}
Graph Fourier Transform (GFT) is a generalization of the Discrete Fourier Transform to graphs. We represent a graph as $(\mathcal{V},\mathcal{E})$ where $\mathcal{V}$ is the set of $N$ nodes and $\mathcal{E}$ represents the edges between them. Denote the adjacency matrix by $A$. In the setting of an undirected graph $A$ would be a symmetric matrix. $D$ is the degree matrix, defined as $(D)_{ii}=\sum_{j}(A)_{ij}$, which is diagonal. The Laplacian of the graph is given by $\hat{L}=D-A$ and the normalized Laplacian $L$ is defined as $L=I-D^{-\frac{1}{2}}AD^{-\frac{1}{2}}$. The Laplacian($L$) can be decomposed into its orthogonal basis, namely the eigenvectors and eigenvalues as:$L = \ftmat_G^{*} \Lambda \ftmat_G$,
%\[ L = U \Lambda U^{*} \] 
where U is an $N \times N$ matrix whose columns are the eigenvectors corresponding to the eigenvalues $\lambda_1, \lambda_2, \dots, \lambda_N$ and $\Lambda = {\rm diag}([\lambda_1, \lambda_2, \dots, \lambda_n])$. 
Let $X \in R^{N \times d}$ be the signal on the nodes of the graph. The Fourier Transform $\hat{X}$ of $X$ is then given as: $ \hat{X} = \ftmat_G X$. 
%\begin{align*}
 %   \hat{X} = U^{*} X 
%\end{align*}
Similarly, the inverse Fourier Transform is defined as: $ X = \ftmat_G^{*} \hat{X}$.
%\begin{align*}
 %   X = U \hat{X} 
%\end{align*}
Note $\ftmat_G^{*}$ is the transposed conjugate of $\ftmat_G$. By the convolution theorem \citep{conv_theorem}, the convolution of the signal $X$ with a filter G having its frequency response as $\hat{G}$ is given by 
(below, $m$ represents the $m^{th}$ node in the graph, $\ftmat_{G_k}$ represents the $k^{th}$ eigenvector or column of $\ftmat_G$):
% (below, $v_m$ is the $m^{th}$ node in the graph):
% \begin{equation}\label{eq_xg_conv}
% % \small 
% \begin{aligned}
%     (X \ast G)(v_m) &= \sum_{k=1}^{n} \hat{X}(\lambda_k) \hat{G}(\lambda_k) U(v_m) \\
%     (X \ast G)(v_m) &= \sum_{k=1}^{n} (U^*X)(\lambda_k) \hat{G}(\lambda_k) U(v_m) \\
%     X \ast G &= U\hat{G}(\Lambda)U^*X
% \end{aligned}
% \end{equation}
\begin{equation}\label{eq_xg_conv}
% \small 
\begin{aligned}
    (X \ast G)(m) &= \sum_{k=1}^{N} \hat{X}(k) \hat{G}(k) \ftmat_{G_k}(m) \\
    (X \ast G)(m) &= \sum_{k=1}^{N} (\ftmat_G X)(k) \hat{G}(k) \ftmat_{G_k}^{*}(m) \\
    X \ast G &= \ftmat_G^{*} \hat{G} \ftmat_G X
\end{aligned}
\end{equation}
Note that a sequence can be considered as a grid graph and for this graph the GFT specializes to the DFT i.e $\ftmat_T=\ftmat_G$.

% \section{Theoretical discussions and Proofs}
\section{Theoretical Proofs}\label{sec:proofs_apndx}
In this section we outline the proofs stated in the main paper and briefly discuss the implications etc. that could not be accommodated in the main paper due to space constraints. For completeness we restate the results.

\begin{lemma}\label{lemma_variational_characterization_apndx}(Variational Characterization of $\mathcal{J_D}$)
    The $2$-Dirichlet $S_2(X)$ of the signals $X$ on $\mathcal{J_D}$ is the quadratic form of the Laplacian $L_{\mathcal{J_D}}$ of $\mathcal{J_D}$ i.e. \\
    \[ S_2(X) = vec(X)^T L_{\mathcal{J_D}} vec(X) \]
\end{lemma}
\begin{proof}
The $p$-Dirichlet form is given by
\begin{align*}
    S_p(X) &= \frac{1}{p} \sum_{n=1}^{N} \int_{t=0}^{T}  \left\lVert \Delta_{n_{i},t} X \right\rVert_p^2 \\
    &= \frac{1}{p} \sum_{t=1}^{T} \sum_{n=1}^{N} \left[ \sum_{n_j \overset{G_t}{\thicksim} n_i} \left( X_{n_j,t} - X_{n_i,t} \right)^2 + \left( \frac{\partial X_{n_i,t}}{\partial t} dt \right)^2 \right]^{\frac{p}{2}}
\end{align*}
Thus the $2$-Dirichlet form is
\begin{align*}
    S_2(X) &= \frac{1}{2} \sum_{n=1}^{N} \int_{t=0}^{T}  \left\lVert \Delta_{n_{i},t} X \right\rVert_2^2 \\
    &= \frac{1}{2} \int_{t=0}^{T} \sum_{n=1}^{N} \left[ \sum_{n_j \overset{G_t}{\thicksim} n_i} \left( X_{n_j,t} - X_{n_i,t} \right)^2 + \left( \frac{\partial X_{n_i,t}}{\partial t} dt \right)^2 \right] \\ 
    &= \frac{1}{2} \left( \int_{t=0}^{T} \sum_{n=1}^{N} \sum_{n_j \overset{G_t}{\thicksim} n_i} \left( \delta X_{n_i,t} \right)^2 \right)  \\
    & + \frac{1}{2} \left( \int_{t=0}^{T} \sum_{n=1}^{N} \left(  X_{n_i,t-dt} - X_{n_i,t} \right)^2  \right)
\end{align*}
We consider the above sum in two parts. Taking the first part we have
\begin{align*}
    & \frac{1}{2} \int_{t=0}^{T} \sum_{n=1}^{N} \sum_{n_j \overset{G_t}{\thicksim} n_i} \left( X_{n_j,t} - X_{n_i,t} \right)^2  \\ 
    &= \frac{1}{2} \int_{t=0}^{T} \sum_{n=1}^{N} \sum_{n_j \overset{G_t}{\thicksim} n_i} \left( X_{n_j,t}^2 - 2*X_{n_j,t}*X_{n_i,t} + X_{n_i,t}^{2} \right)  \\
    &= \frac{1}{2} \int_{t=0}^{T} \sum_{n=1}^{N} \sum_{n_j \overset{G_t}{\thicksim} n_i} \left( X_{n_j,t}^2 - X_{n_j,t}*X_{n_i,t} - X_{n_j,t}*X_{n_i,t} + X_{n_i,t}^{2} \right)  \\
\end{align*}
\begin{align*}
    &= \frac{1}{2} \int_{t=0}^{T} \sum_{n=1}^{N} \sum_{n_j \overset{G_t}{\thicksim} n_i} \left( X_{n_j,t} \left( X_{n_j,t} - X_{n_i,t} \right) + X_{n_i,t} \left( X_{n_i,t} - X_{n_j,t} \right) \right)  \\
    &= \frac{1}{2} \int_{t=0}^{T} \sum_{n=1}^{N} 2* \sum_{n_j \overset{G_t}{\thicksim} n_i} \left( X_{n_j,t} \left( X_{n_j,t} - X_{n_i,t} \right) \right)  \\
    &= \int_{t=0}^{T} \sum_{n=1}^{N} \sum_{n_j \overset{G_t}{\thicksim} n_i} \left( X_{n_j,t} \left( X_{n_j,t} - X_{n_i,t} \right) \right)  \\
    &= vec(X)^{\top} (I_T \otimes {L_{G_t}} ) vec(X)
\end{align*}

Considering the second part which is the ring graph along the time dimension we have
\begin{align*}
    & \frac{1}{2} \left( \int_{t=0}^{T} \sum_{n=1}^{N} \left( X_{n_i,t-dt} - X_{n_i,t} \right)^2  \right) \\
    &= \frac{1}{2} \left( \int_{t=0}^{T} \sum_{n=1}^{N} X_{n_i,t-dt}^2 - 2*X_{n_i,t-dt}*X_{n_i,t} + X_{n_i,t}^2  \right) \\
    &= \frac{1}{2} \left( \int_{t=0}^{T} \sum_{n=1}^{N} 2X_{n_i,t-dt}^2 - 2*X_{n_i,t-dt}*X_{n_i,t}  \right) \text{\quad\dots Redistributing terms} \\
    &= \int_{t=0}^{T} \sum_{n=1}^{N} X_{n_i,t-dt}^2 - X_{n_i,t-dt}*X_{n_i,t}  \\
    &= \int_{t=0}^{T} \sum_{n=1}^{N} X_{n_i,t-dt} \left( X_{n_i,t-dt} - X_{n_i,t} \right) \\
    &= vec(X)^{\top} (L_T \otimes I_N ) vec(X)
\end{align*}

Combining the results of the 2 parts we get the below result
\begin{align*}
    S_2(X) &= vec(X)^{\top} (I_T \otimes {L_{G_t}} ) vec(X) + vec(X)^{\top} (L_T \otimes I_N ) vec(X) \\
    &= vec(X)^{\top} \left( I_T \otimes {L_{G_t}} + L_T \otimes I_N \right) vec(X) \\
    &= vec(X)^{\top} L_{\mathcal{J_D}} vec(X)
\end{align*}
as required.

\end{proof}

This implies that $L_{\mathcal{J_D}} \succeq 0$. We can see that slower the changes in the signals along the dynamic graph smaller the value of $S_2(X)$ and vice versa. Thus we have a notion of variation of signals on the dynamic graph similar to the case of static graphs. 
% TODO: mention about the eigenvectors and eigenvalues
% \emph{EFT} therefore characterizes signals on the dynamic graph by its proximity (projection) to the optimizers of $S_2(X)$ meaning high (collective dynamic graph) frequency components correspond to sharply varying signals and low frequency components to smoother signals.
The eigen decomposition of \ljdmat therefore characterizes signals on the dynamic graph by its projection to the optimizers of $S_2(X)$. In other words, high collective dynamic graph frequency components inform of the presence of sharply varying signals and smoother signals will have higher magnitude in the low frequency components. 
% Now we look at the equivalence between the eigen decomposition of \ljdmat and $\ftmat_D$.
Next we provide a solution to the relaxed pseudospectrum objective in \ref{lemma_ctdft_der}.

% \begin{lemma}\label{lemma_ctdft_der_apndx}
%     Consider the variational form $ x^T L_{\mathcal{J_D}} x = \int_{i=0}^{NT} x(i) \int_{j=0}^{NT} L_{\mathcal{J_D}}(i,j) x(j) di dj $. The optimization problem $f = \underset{x, \norm{x} \leq 1}{max} [ | x^T L_{\mathcal{J_D}} x - \lambda_s | - \epsilon ]_{+}$ has the optimal solution as $y_{\omega} \otimes {z_{l}^{\omega}} $, where $\lambda_s$ is the optimal value of equation \ref{eq_variational_char_LD}, $y_{\omega}$ is the $\omega$-th optimal solution of the ring graph, $z_{l}^{t}$ is the $l$-th optimal solution of the graph at time t and $\epsilon = \mathcal{O}(\delta)$.
% \end{lemma}
\begin{lemma}\label{lemma_ctdft_der_apndx}
    Consider the variational form $ x^T L_{\mathcal{J_D}} x = \int_{i=0}^{NT} x(i) \int_{j=0}^{NT} L_{\mathcal{J_D}}(i,j) x(j) di dj $. The optimization problem $f = \underset{x, \norm{x} \leq 1}{min} [ | x^T L_{\mathcal{J_D}} x - \lambda_s | - \epsilon ]_{+}$ has the optimal solution as $y_{\omega} \otimes {z_{l}^{\omega}} $, where $\lambda_s$ is the optimal value of equation \ref{eq_variational_char_LD}, $y_{\omega}$ is the $\omega$-th optimal solution of the ring graph, $z_{l}^{t}$ is the $l$-th optimal solution of the graph at time t and $\epsilon = \mathcal{O}(\delta)$.
\end{lemma}
\begin{proof}
    We begin by considering the variational characterization of $L_{\mathcal{J_D}}$ which is given by the below equation
    \begin{equation}
        \lambda_s = \underset{x, \norm{x} \leq 1}{max} \int_{i=0}^{NT} x(i) \int_{j=0}^{NT} L_{\mathcal{J_D}}(i,j) x(j) di dj \underset{x, \norm{x} \leq 1}{max} x^T L_{\mathcal{J_D}} x 
    \end{equation}
    Note that in the objective, $x^T L_{\mathcal{J_D}} x$ is convex since $L_{\mathcal{J_D}} \succeq 0$ and thus $\nabla^2 \left( x^T L_{\mathcal{J_D}} x \right) = 2 L_{\mathcal{J_D}} \succeq 0$. Also, we can check that $\norm{x}^2 \leq 1$ is convex. Thus applying the KKT conditions \citep{KKTref1} to the lagrangian $L = x^T L_{\mathcal{J_D}} x + \lambda (\norm{x}^2 - 1)$, we get the below equation
    \begin{equation}
        L_{\mathcal{J_D}} x = \lambda_s x
    \end{equation}
    We recognize from the above equation that $\lambda_s$ is the eigenvalue of $L_{\mathcal{J_D}}$ and x is the corresponding eigenvector. However the computation of this exact solution is computationally costly and here we are ineterested in finding an efficient form of the solution to the objective with the pseudospectrum relaxation. As already seen, the pseudospectrum can be defined by the set $\{ \lambda \in \mathbb{C} \quad | \quad \norm{(L_{\mathcal{J_D}}-\lambda I)^{-1}} \geq \frac{1}{\epsilon} \}$ or equivalently $\{ \lambda \in \mathbb{C} \quad | \quad \norm{(L_{\mathcal{J_D}}-\lambda I)} \leq \epsilon \}$, where $\norm{.}$ is the operator norm. Thus we have that for the pseudospectrum, there exists a unit vector $v$ such that $| (L_{\mathcal{J_D}}-\lambda I)v | \leq \epsilon$ and so $|\lambda_s-\lambda| \leq \epsilon$. 
    % This shows that the $\epsilon-$ neighborhood of the spectrum of $L_{\mathcal{J_D}}$ is contained in the pseudospectrum i.e. if a $\lambda$ is in the $\epsilon-$ neighborhood of $\lambda_s$, it is in the pseudospectrum. 
    This shows that the $\epsilon-$ neighborhood of the spectrum of $L_{\mathcal{J_D}}$ is contained in the pseudospectrum i.e. if $\lambda$ is in the pseudospectrum of $L_{\mathcal{J_D}}$ it is in the $\epsilon-$ neighborhood of $\lambda_s$. 
    We would now like to find a solution residing in the pseudospectrum of $L_{\mathcal{J_D}}$. 
    % In order to show that a vector $y$ is a solution to the relaxed objective it suffices to show that the corresponding eigenvalue lies in the $\epsilon-$ neighborhood of the spectrum.

We have $\{L_{G_t}\} \in R^{N \times N \times T}$ to be the Laplacian of the graphs at each timestep with eigenvalues $\lambda_i^{t}$ where $i \in {N}, t \in [0,T]$. $L_T \in R^{T \times T}$ be the Laplacian of the time adjacency matrix with eigenvalues $\mu_j$ where $j \in {T}$.
The Laplacian of the collective graph $\mathcal{J_D}$ is expressed as 
% \begin{equation*}
%     (L_{\mathcal{J_D}})_i^j = (L_T \oplus \{L_{G_t}\})_i^{j \left\lfloor \frac{j}{N} \right\rfloor} = L_T \otimes I_N + (I_T \otimes \{ L_{G_t} \})_{i}^{j \left\lfloor \frac{j}{N} \right\rfloor}
% \end{equation*}
\begin{equation*}
    L_{\mathcal{J_D}} = L_T \oplus \{L_{G_t}\} = L_T \otimes I_N + [ I_T \boxtimes \{ L_{G_t} \} ]
\end{equation*}
In the above equation, $\boxtimes$ is the timestep wise Kronecker product and operator $[.]$ represents the vectorization. If $T$ is discrete this vectorization can be thought of as a reordering from $R^{NT \times T \times N} \xrightarrow[]{} R^{NT \times NT}$. Consider $a_1, a_2, \dots a_p$ to be the linearly independent right eigenvectors of $L_T$ and $b_1^t, b_2^t, \dots b_{q_t}^t$ to be the linearly independent right eigenvectors of $L_{G_t}$.  
% Consider the vector $y_j = (a_k \otimes {b_l^t})_j^{\left\lfloor \frac{j}{N} \right\rfloor}, y \in R^{NT}$. 
Consider the vector $y = [ a_k \boxtimes \{b_l^t\} ]$, where $\{b_l^t\}$ represents the set of eigenvectors of Laplacian at time $t$ i.e. $L_{G_t}$ and the operator $\boxtimes$ is again timestep wise followed by vectorization. 
Then we have
\begin{align*}
    & L_{\mathcal{J_D}} y = L_T \otimes I_N y + [ I_T \boxtimes \{ L_{G_t} \} ] y \\
    &= (L_T \otimes \{I_N\}) [ a_k \boxtimes \{b_l^t\} ] + [ I_T \boxtimes \{ L_{G_t} \} ] [ a_k \boxtimes \{b_l^t\} ] \\
    &= (L_T \otimes \{I_N\} \square [ a_k \boxtimes \{b_l^t\} ]) + (I_T \otimes \{ L_{G_t} \} \square [ a_k \boxtimes \{b_l^t\} ]) \\
    &= [L_T a_k \boxtimes \{I_N\} \square \{b_l^t\}] + [I_T a_k \boxtimes \{ L_{G_t} \square \{b_l^t\}\}] \\
    &= (\mu_k [ a_k \boxtimes \{b_l^t\} ]) + [a_k \boxtimes \{ \lambda_{l}^t \{b_l^t\}\}] \\
    &= (\mu_k [ a_k \boxtimes \{b_l^t\} ] + [a_k \boxtimes \{ \lambda_{l}^t \{b_l^t\}\}] \\
    &= ( [ a_k \boxtimes \{b_l^t\} ] diag(\{\mu_k\}) + [a_k \boxtimes \{b_l^t\} diag(\{\lambda_{l}^t\})] \\
    &= ( [ a_k \boxtimes \{b_l^t\} ] diag(\{\mu_k + \lambda_{l}^t\})) \\
\end{align*}
where $\square$ indicates timestep (column) wise product and $diag(.)$ operator converts a vector to a diagonal matrix. 
Now considering the graph at the 0-th timestep having eigenvalue $\lambda_{l}^0$, we are interested in verifying the pseudospectrum condition for $\mu_k + \lambda_{l}^0\}$. We thus have to find the upper bound for $\norm{L_{\mathcal{J_D}} - (\mu_k + \lambda_{l}^0)I}$.
% In the special case where $G_{ti}=G_{tj} \forall i,j \in {T}$ we have  $\lambda_{l}^{ti}=\lambda_{l}^{tj}$. Thus we get
% \begin{align*}
%     (L_{\mathcal{J_D}} y)_i &= ( (a_k \otimes {b_l^t}) diag(\{\mu_k + \lambda_{l} I_T\}))_{i}^{\left\lfloor \frac{i}{N} \right\rfloor} \\
%     &= ( \mu_k + \lambda_{l} (a_k \otimes {b_l^t}) diag(\{ I_T\}))_{i}^{\left\lfloor \frac{i}{N} \right\rfloor} \\
%     &= ( \mu_k + \lambda_{l} (a_k \otimes {b_l^t}))_{i}^{\left\lfloor \frac{i}{N} \right\rfloor} \\
%     &= ( \mu_k + \lambda_{l}) y_{i} \\
% \end{align*}

    In order to bound the above expression we consider the vector $y = [ a_k \boxtimes \{b_l^t\} ]$. 
    We have from the above equations,
    \begin{equation}\label{eq_pseudo_bound_lhs}
        \norm{L_{\mathcal{J_D}} y - (\mu_k + \lambda_{l}^0)y} = ( [ a_k \boxtimes \{b_l^t\} ] diag(\{\mu_k + \lambda_{l}^t - (\mu_k + \lambda_{l}^0)\}))
    \end{equation}
    We would like to study the rate of change of the eigenvalues as the graph changes. Consider a normal matrix $A$ of which the eigenvectors $v_1,\ldots,v_n$ form a basis of ${\Bbb C}^n$. Also we consider $w_1,\ldots,w_n$ be the dual basis, i.e. $w_j^* v_k = \delta_{jk}$ for all $1 \leq j,k \leq n$, where $\delta_{jk}$ is the Kronecker delta and 
    \begin{equation*}
        \delta_{jk} = 
        \begin{dcases}
        1 ,& \text{if } \text{$j=k$}\\
        0,              & \text{otherwise}
        \end{dcases} \\
    \end{equation*}
    Since the eigenvectors form a basis we can represent any vector $u$ as a linear combination of $v_1,\ldots,v_n$ as $u = \sum_{j=1}^n a_j v_j$. Also we have $w_j^* u = \sum_{j=1}^n a_j w_j^* v_j = a_j$.
    We thus have the below equation
    \begin{equation}\label{thm1_eq1}
        u = \sum_{j=1}^n (w_j^* u) v_j    
    \end{equation}
    for any vector $u \in \Bbb{C}^n$.  
    We know the below relation due to $v_k$ being the eigenvector of $A$ with eigen value $\lambda_k$
    \begin{equation}\label{thm1_eq2}
        A v_k = \lambda_k v_k
    \end{equation}
    We also can write the following in terms of the dual basis (since $A$ is a normal matrix)
    \begin{equation}\label{thm1_eq3}
    \begin{aligned}
        w_k^* A &= \sum_{j} \lambda_j w_k^* v_j w_j^* \\
        w_k^* A &= \lambda_k w_k^* 
    \end{aligned}
    \end{equation}
    
    We now differentiate \ref{thm1_eq2} using the product rule of differentiation to get
    \begin{equation}\label{thm1_eq4}
        \dot A v_k + A \dot v_k = \dot \lambda_k v_k + \lambda_k \dot v_k
    \end{equation}
    
    Taking the inner product of the equation \ref{thm1_eq4} with $w_k^*$, and using \ref{thm1_eq3} we obtain:
    \begin{equation}
    \begin{aligned}
        \dot A v_k + A \dot v_k &= \dot \lambda_k v_k + \lambda_k \dot v_k \\
        w_k^* \dot A v_k + w_k^* A \dot v_k &= w_k^* \dot \lambda_k v_k + w_k^* \lambda_k \dot v_k \\
        w_k^* \dot A v_k + \lambda_k w_k^* \dot v_k &= \dot \lambda_k w_k^* v_k + \lambda_k w_k^* \dot v_k \\
        w_k^* \dot A v_k &= \dot \lambda_k \\
        \dot \lambda_k = w_k^* \dot A v_k
    \end{aligned}
    \end{equation}
    Assuming $\lambda_k^0$ to be the eigenvalue at the start, we can get the value after time $t$ by simply integrating as follows,
    \begin{equation}
        \lambda_k^t = \lambda_k^0 + \int_{0}^{t} \omega_k \dot A v_{k_0} dt
    \end{equation}

    Thus from the above result and equation \ref{eq_pseudo_bound_lhs} we have
    \begin{align*}
        \norm{L_{\mathcal{J_D}} y - (\mu_k + \lambda_{l}^0)y} &= \norm{ [ a_k \boxtimes \{b_l^t\} ] diag(\{\mu_k + \lambda_{l}^t - (\mu_k + \lambda_{l}^0)\})} \\
        &= \norm{ [ a_k \boxtimes \{b_l^t\} ] diag(\lambda_{l}^t - \lambda_{l}^0\}) } \\
        &= \norm{ [ a_k \boxtimes \{b_l^t\} ] diag(\int_{0}^{t} \omega_k \dot A v_{k_0} dt) } \\
        &\leq \norm{ [ a_k \boxtimes \{b_l^t\} ]} \norm{ diag(\int_{0}^{t} \omega_k \dot A v_{k_0} dt) } \\
        &\leq \norm{ \int_{0}^{T} \int_{0}^{t} \omega_k \dot A v_{k_0} dtdt } \\
        &\leq \norm{ \int_{0}^{T} \int_{0}^{t} \norm{\omega_k \dot A v_{k_0}} dtdt } \\
        &\leq \norm{ \int_{0}^{T} \int_{0}^{t} \norm{\dot A} dtdt } \\
        &\leq \norm{ \int_{0}^{T} \int_{0}^{t} \delta N dtdt } \\
         &\leq \norm{ \int_{0}^{T} \delta NT dt } \\
        & \leq \delta NT^2 \\
        & \leq \mathcal{O}(\delta) \\
         \norm{L_{\mathcal{J_D}} y - (\mu_k + \lambda_{l}^0)y} &\leq \epsilon \\
          \norm{L_{\mathcal{J_D}} - (\mu_k + \lambda_{l}^0)} &\leq \epsilon
    \end{align*}

    Thus $\mu_k + \lambda_{l}^0$ is in the pseudospectrum of $L_{\mathcal{J_D}}$ and so $y=[ a_k \boxtimes \{b_l^t\} ]$ is one of the solutions to the objective with the pseudospectrum relaxation. Thus it follows that $[ \ftmat_T \boxtimes \{\ftmat_{G_t}\} ]$ forms a basis of the solution to the defined objective, where $\ftmat_T$ and $\ftmat_{G_t}$ have $a_k^{*}$ and $b_l^t$ as their row spaces respectively.
    
\end{proof}

% The above result gives a special case when the graph is static. 
The above result gives us the definition of \eft in terms of the Kronecker product of the Time Fourier Transform and the Graph Fourier Transform of the graph at each time.
While both \eft and \aeft are solutions to the pseudospectrum relaxed objectives they are not equal in general.
% But does the equivalence between $\ftmat_D$ and eigen decomposition $L_{\mathcal{J_D}}$ hold in the general case when the graph structure changes with time?
% We answer this question in the negative.
% To answer this question, 
To see this, we first need to look at the eigenvectors of $L_{\mathcal{J_D}}$. Let $\ftmat_{AD}$ be the matrix whose rows form the right eigenvectors of $L_{\mathcal{J_D}}$. 
% Due to [cite >4] we know that a closed form solution may not exist.
Below we state and prove the result of equivalence between $\ftmat_D$ and $\ftmat_{AD}$ for the general case of dynamic graphs using a counter example

\begin{remark}
In general, the collective dynamic graph fourier transform as defined by the operator $\ftmat_D$ does not form the eigenspace of the 
% eigen decomposition 
spectrum 
of $L_{\mathcal{J_D}}$ i.e. $\ftmat_D \neq \ftmat_{AD}$.
\end{remark}
\begin{proof}
    It is sufficient to show a single counter example to conclude the statement.

    Consider the below weighted adjacency matrix for a certain graph at time $t_0$
    \begin{align*}
        G_0 = 
        \begin{bmatrix}
     	1 & 0.5 \\
            0.5 & 1
        \end{bmatrix}
    \end{align*}
    Let this change to the following in the next timestep $t_1$
    \begin{align*}
        G_1 = 
        \begin{bmatrix}
     	1 & 0.6 \\
            0.6 & 1
        \end{bmatrix}
    \end{align*}

    The Laplacian $L_{\mathcal{J_D}}$ is given by
    \begin{align*}
        L_{\mathcal{J_D}} = 
        \begin{bmatrix}
            1.5 & -0.5 & -1. & -0. \\
            -0.5 & 1.5 & -0. & -1. \\
            -1. & -0. & 1.6 & -0.6 \\
            -0. & -1. & -0.6 & 1.6
        \end{bmatrix}
    \end{align*}    

    The EFT matrix $\ftmat_D$ is 
    \begin{align*}
    \ftmat_D = 
    \begin{bmatrix}
        0.5 & -0.5 & -0.5 & 0.5  \\
        0.5 & 0.5 & -0.5 & -0.5 \\
        0.5 & -0.5 & 0.5 & -0.5 \\
        0.5 & 0.5 & 0.5 & 0.5
    \end{bmatrix}
    \end{align*}    

    Similarly the matrix $\ftmat_{AD}$ comes out to be the following (upto sign and row wise permutations)
    \begin{align*}
    \ftmat_{AD} = 
    \begin{bmatrix}
        0.47 & -0.47 & -0.52 & 0.52 \\
        0.5 & 0.5 & -0.5 & -0.5 \\
        0.52 & -0.52 & 0.47 & -0.47 \\
        0.5 & 0.5 & 0.5 & 0.5 & 
    \end{bmatrix}
    \end{align*}    

    From the above we can see the two matrices differ and so we have a counter example.
    
\end{proof}

From the above result we can see that in the general case of dynamic graphs the defined \emph{EFT} and the eigen decomposition of the defined Laplacian \ljdmat are not the same. Thus we can have an alternate definition of the collective dynamic graph fourier transform in terms of the decomposition of the joint Laplacian \ljdmat.
We term $\ftmat_{AD}$ as the \emph{Absolute Drcomposition} or \emph{AD} for brevity.

Both \eft and \aeft have their own advantages. \eft has a simple primal definition and is easy to compute whereas \aeft has a beautiful physical interpretation. 
Even though \eft and \aeft are not exactly the same, in order to have desirable properties of both we can define approximation bounds that inform under what conditions the two transforms can be used interchangeably upto the approximation error. We work under the below assumptions for weighted graphs in order to bound the two transforms:
\begin{enumerate}
    \item The rate of change of the graph with time is bounded
    \item The eigenvalues of the graph Laplacian at any given timestep and between timesteps has a multiplicity of 1
    % \item The eigenvalues of the graph Laplacian between timesteps has a multiplicity of 1
\end{enumerate}
The condition 2 is required for stability of the bound and can be enforced for example by adding random perturbations to the matrix.
% Before stating the error bounds we give a brief background of the pseudospectrum of a graph that is used in the analysis. For more details we refer the reader to the following \citep{terry_pseudospectrum}.
% % \url{https://terrytao.wordpress.com/2008/10/28/when-are-eigenvalues-stable/}.

% \nonindent \underline{\textbf{Pseudospectrum}}:
% The spectrum of a graph (of N nodes) is a finite set consisting of N points $\lambda$ that form the eigenvalues of the graph's matrix representation $A$ i.e. $\{ \lambda \in \mathbb{C} \quad | \quad \norm{(A-\lambda I)^{-1}} = \infty \}$. Similarly we can think of the ($\epsilon$-)pseudospectrum of a graph to be the larger set (containing these N points) such that $A-\lambda I$ has the least singular value at most $\epsilon$. Formally the pseudospectrum can be defined by the set $\{ \lambda \in \mathbb{C} \quad | \quad \norm{(A-\lambda I)^{-1}} \geq \frac{1}{\epsilon} \}$.
Based on these assumptions we state and prove the bounds between \eft and \aeft below
\begin{theorem}
    Considering bounded changes in a graph $G$ with $N$ nodes over time $T$, the norm of the difference between \eft ($\ftmat_D$) and \aeft ($\ftmat_{AD}$) is bounded as follows: 
    $\norm{\ftmat_D-\ftmat_{AD}} \leq \mathcal{O}\left(\frac{N^{\frac{3}{2}} T \omega_{max}^{\frac{1}{2}}}{(\Delta \lambda_G)_{min}} + \frac{N^{\frac{3}{2}} T \omega_{max}^{2}}{(\Delta \lambda_J)_{min}}\right) \left( \norm{\dot{L}_G} \right)_{max}$
    where $(\Delta \lambda_G)_{min}$ and $(\Delta \lambda_J)_{min}$ refer to the minimum difference between the eigenvalues of matrices $L_G$ and \ljdmat respectively, $\dot{L_G}$ is the rate of change of $L_G$ and $\omega_{max}=2\pi$.
\end{theorem}
\begin{proof}
    Consider $L_J$ to be the Laplacian of collective graph when the graphs are static with time. We can show this, in a similar manner as \ljdmat to be $L_J = L_T \otimes I_N + I_T \otimes L_G$, where $L_G$ is the Laplacian of the static graph. Let $\ftmat_J$ be the matrix whose rows form the left eigenvectors of $L_J$. Now we consider the graph to change infinitesimally with $L_G$ as the starting state of the Laplacian. We intend to bound the frobenius norm $\norm{\ftmat_D - \ftmat_{AD}}$. We can manipulate this as follows
    \begin{align*}
        \norm{\ftmat_D - \ftmat_{AD}} &= \norm{\ftmat_D - \ftmat_{AD} + \ftmat_J - \ftmat_J} \\
        &= \norm{\ftmat_D - \ftmat_J + \ftmat_J - \ftmat_{AD} } \\
        &\leq \norm{\ftmat_D - \ftmat_J}  + \norm{\ftmat_J - \ftmat_{AD} } \\
    \end{align*}
    We thus find the bound in two parts first for the error between $\ftmat_D, \ftmat_J$  and second for the error between $\ftmat_J, \ftmat_{AD}$. 

    In order to bound the matrices (which are formed by the eigenvectors) we first attempt to bound the vectors forming the matrix. For this we study the rate of change of the vectors with time (as the graph evolves) using the language of calculus. For deeper insights into this and the stability of eigenvectors/values we refer the interested reader to \citep{terry_pseudospectrum}. Consider a normal matrix $A$ of which the eigenvectors $v_1,\ldots,v_n$ form a basis of ${\Bbb C}^n$. Also we consider $w_1,\ldots,w_n$ be the dual basis, i.e. $w_j^* v_k = \delta_{jk}$ for all $1 \leq j,k \leq n$, where $\delta_{jk}$ is the Kronecker delta and 
    \begin{equation*}
        \delta_{jk} = 
        \begin{dcases}
        1 ,& \text{if } \text{$j=k$}\\
        0,              & \text{otherwise}
        \end{dcases} \\
    \end{equation*}
    Since the eigenvectors form a basis we can represent any vector $u$ as a linear combination of $v_1,\ldots,v_n$ as $u = \sum_{j=1}^n a_j v_j$. Also we have $w_j^* u = \sum_{j=1}^n a_j w_j^* v_j = a_j$.
    We thus have the below equation
    \begin{equation}\label{thm1_eq1}
        u = \sum_{j=1}^n (w_j^* u) v_j    
    \end{equation}
    for any vector $u \in \Bbb{C}^n$.  
    We know the below relation due to $v_k$ being the eigenvector of $A$ with eigen value $\lambda_k$
    \begin{equation}\label{thm1_eq2}
        A v_k = \lambda_k v_k
    \end{equation}
    We also can write the following in terms of the dual basis (since $A$ is a normal matrix)
    \begin{equation}\label{thm1_eq3}
    \begin{aligned}
        w_k^* A &= \sum_{j} \lambda_j w_k^* v_j w_j^* \\
        w_k^* A &= \lambda_k w_k^* 
    \end{aligned}
    \end{equation}
    
    We now differentiate \ref{thm1_eq2} using the product rule of differentiation to get
    \begin{equation}\label{thm1_eq4}
        \dot A v_k + A \dot v_k = \dot \lambda_k v_k + \lambda_k \dot v_k
    \end{equation}
    
    Taking the inner product of the equation \ref{thm1_eq4} with $w_k^*$, and using \ref{thm1_eq3} we obtain:
    \begin{equation}
    \begin{aligned}
        \dot A v_k + A \dot v_k &= \dot \lambda_k v_k + \lambda_k \dot v_k \\
        w_k^* \dot A v_k + w_k^* A \dot v_k &= w_k^* \dot \lambda_k v_k + w_k^* \lambda_k \dot v_k \\
        w_k^* \dot A v_k + \lambda_k w_k^* \dot v_k &= \dot \lambda_k w_k^* v_k + \lambda_k w_k^* \dot v_k \\
        w_k^* \dot A v_k &= \dot \lambda_k \\
        \dot \lambda_k = w_k^* \dot A v_k
    \end{aligned}
    \end{equation}
    
    In our case since A is normal, we have the eigenbasis $v_k$ as an orthonormal set, and the dual basis $w_k$ is identical to $v_k$. 
    % In particular we see that $|\dot \lambda_k| \leq \|\dot A\|_{op}$; the infinitesimal change of each eigenvalue does not exceed the infinitesimal size of the perturbation.  This is consistent with the stability of the spectrum for normal operators mentioned in the previous section.
    
    We are interested in how the eigenvectors change with time. Taking the inner product of equation \ref{thm1_eq4} with $w_j^*$ for $j \neq k$, we get
    % \begin{equation}
    %     w_j^* \dot A v_k + (\lambda_j - \lambda_k) w_j^* \dot v_k = 0
    % \end{equation}
    % Using \ref{thm1_eq1} we obtain the following
    % % conclude a first variation formula for the eigenvectors $v_k$, namely that
    \begin{equation}
    \begin{aligned}
        \dot A v_k + A \dot v_k &= \dot \lambda_k v_k + \lambda_k \dot v_k \\
        w_j^* \dot A v_k + w_j^* A \dot v_k &= w_j^* \dot \lambda_k v_k + w_j^* \lambda_k \dot v_k \\
        w_j^* \dot A v_k + \lambda_j w_j^* \dot v_k &= \dot \lambda_k w_j^* v_k + w_j^* \lambda_k \dot v_k \\
        w_j^* \dot A v_k + \lambda_j w_j^* \dot v_k &= w_j^* \lambda_k \dot v_k \\
        w_j^* \dot A v_k + \lambda_j w_j^* \dot v_k - w_j^* \lambda_k \dot v_k &= 0 \\
        w_j^* \dot A v_k + (\lambda_j  - \lambda_k) w_j^* \dot v_k &= 0 \\
        w_j^* \dot v_k &= \frac{w_j^* \dot A v_k}{(\lambda_k  - \lambda_j)}\\
    \end{aligned}
    \end{equation}
    Using the above in \ref{thm1_eq1} we obtain the following
    \begin{equation}
    \begin{aligned}
        \dot v_k &= \sum_{j=1}^n (w_j^* \dot v_k) v_j  \\
        \dot v_k &= \sum_{j \neq k} (w_j^* \dot v_k) v_j +  (w_k^* \dot v_k) v_k \\
        \dot v_k &= \sum_{j \neq k} \frac{w_j^* \dot A v_k}{\lambda_k - \lambda_j} v_j +  (w_k^* \dot v_k) v_k
    \end{aligned}
    \end{equation}
    We consider the change in $A$ so that the resulting matrix is also normal. Thus the eigenvectors of the resulting matrix will also be orthonormal. This imples all the vectors lie on the surface of the unit sphere in $\Bbb{C}^n$ and so the change in the eigenvectors should be along the surface of this sphere. As such $\dot A v_k$ will be tangential to the sphere at $v_k$ and so $\dot v_k^{\top} v_k = 0$. Note this need not be the case in general if we consider non-unit vectors (that could also be eigenvectors). Thus we can represent $\dot v_k = 0 v_k + \sum_{j \neq k} b_j v_j$. Thus we have
    \begin{align*}
        \dot v_k &= v_k + \sum_{j \neq k} b_j v_j \\
        w_k^* \dot v_k &= w_k^* \left( \sum_{j \neq k} b_j v_j \right) \\
        w_k^* \dot v_k &= \left( \sum_{j \neq k} b_j w_k^* v_j \right) \\
        w_k^* \dot v_k &= 0
    \end{align*}
    Using the above equations and the consideration that $\norm{v_j}=1$ we have
    % \begin{equation}
    % \begin{aligned}
    %     \dot v_k &= \sum_{j \neq k} \frac{w_j^* \dot A v_k}{\lambda_k - \lambda_j} v_j \\
    %     \norm{\dot v_k} &= \norm{\sum_{j \neq k} \frac{w_j^* \dot A v_k}{\lambda_k - \lambda_j} v_j} \\
    %     &\leq \sum_{j \neq k} \norm{\frac{w_j^* \dot A v_k}{\lambda_k - \lambda_j} v_j} \\
    %     &\leq \sum_{j \neq k} \norm{\frac{w_j^* \dot A v_k}{\lambda_k - \lambda_j}} \norm{v_j} \\ 
    %     \norm{\dot v_k} &\leq \sum_{j \neq k} \norm{\frac{w_j^* \dot A v_k}{\lambda_k - \lambda_j}} \\
    %     &\leq \sum_{j \neq k} \frac{ \norm{w_j^* \dot A v_k}}{\norm{\lambda_k - \lambda_j}} \\
    %     &\leq \sum_{j \neq k} \frac{ \sigma(\dot A)}{\norm{\lambda_k - \lambda_j}} \\
    %     &\leq \sum_{j \neq k} \frac{ \norm{\dot A}}{\norm{\lambda_k - \lambda_j}} \\
    %     &\leq \sum_{j \neq k} \frac{ \norm{\dot A}}{\dlambdamin} \\
    %     &\leq \frac{ N-1}{\dlambdamin} \norm{\dot A} \\
    % \end{aligned}
    % \end{equation}
    \begin{align*}
        \dot v_k &= \sum_{j \neq k} \frac{w_j^* \dot A v_k}{\lambda_k - \lambda_j} v_j \\
        \norm{\dot v_k} &= \norm{\sum_{j \neq k} \frac{w_j^* \dot A v_k}{\lambda_k - \lambda_j} v_j} \\
        &\leq \sum_{j \neq k} \norm{\frac{w_j^* \dot A v_k}{\lambda_k - \lambda_j} v_j} \\
        &\leq \sum_{j \neq k} \norm{\frac{w_j^* \dot A v_k}{\lambda_k - \lambda_j}} \norm{v_j} \\ 
    \end{align*}
    Since we consider orthonormal vectors $\norm{v_j}=1$
    \begin{align*} 
        \therefore \norm{\dot v_k} &\leq \sum_{j \neq k} \norm{\frac{w_j^* \dot A v_k}{\lambda_k - \lambda_j}} \\
        &\leq \sum_{j \neq k} \frac{ \norm{w_j^* \dot A v_k}}{\norm{\lambda_k - \lambda_j}} \\
        &\leq \sum_{j \neq k} \frac{ \sigma(\dot A)}{\norm{\lambda_k - \lambda_j}} \\
        &\leq \sum_{j \neq k} \frac{ \norm{\dot A}}{\norm{\lambda_k - \lambda_j}} \\
        &\leq \sum_{j \neq k} \frac{ \norm{\dot A}}{\dlambdamin} \\
        &\leq \frac{ N-1}{\dlambdamin} \norm{\dot A} \\
    \end{align*}
    where $\sigma(.)$ is the operator norm and $\dlambdamin$ is the absolute of the minimum difference between the eigenvalues of $A$. in the above we have seen how the change in eigenvectors is bounded by the change in the matrix. Using this result we now attempt to bound the change in the required transform matrices.

    For the first part we bound $\norm{\ftmat_D - \ftmat_J}$. Let $\Delta v$ represent the (infinitesimal) change in the eigenvectors of $L_G$ in time $t$ and let $\Delta v_i$ be the infinitesimal change per unit time in the vector at step $i$. 
    % Now using the triangle inequality we have $\Delta v \leq$
    Thus using the triangle inequality we have the below equations
    \begin{align*}
        \norm{\Delta v} \leq \sum_{i} \norm{\Delta v_i} \\ 
        % \norm{\Delta v} \leq \int_{t} \norm{\Delta v(t)} dt \\
        \norm{\Delta v} \leq \int_{t} \norm{\dot v(t)} dt \\
        % \norm{\Delta v} \leq \int_{t} \norm{\frac{ N-1}{\dlambdamin} \norm{\dot L_G(t)}} dt \\
        % \leq \int_{t} \frac{ N-1}{\dlambdamin} \norm{\left( \dot L_G \right)_{max} } dt \\
        % \norm{\Delta v} \leq \frac{ N-1}{\dlambdamin} \norm{\left( \dot L_G \right)_{max} }  \int_{t} dt \\
        % \norm{\Delta v} \leq \frac{ N-1}{\dlambdamin} \norm{\left( \dot L_G \right)_{max} }  t \\
        % \leq \frac{ \mathcal{O}(N-1)}{\dlambdamin} \norm{\left( \dot L_G \right)_{max} } \\
    \end{align*}
    Using the above derived inequality $ \norm{\dot v_k} \leq \frac{ N-1}{\dlambdamin} \norm{\dot A}$ for $L_G$ we have
    \begin{align*}
        % \norm{\Delta v} \leq \sum_{i} \norm{\Delta v_i} \\ 
        % \norm{\Delta v} \leq \int_{t} \norm{\Delta v(t)} dt \\ 
        \norm{\Delta v} \leq \int_{t} \norm{\frac{ N-1}{\dlambdamin} \norm{\dot L_G(t)}} dt \\
        \leq \int_{t} \frac{ N-1}{\dlambdamin} \left( \norm{ \dot L_G(t)  }\right)_{max} dt \\
        % \norm{\Delta v} \leq \frac{ N-1}{\dlambdamin} \norm{\left( \dot L_G \right)_{max} }  \int_{t} dt \\
        % \norm{\Delta v} \leq \frac{ N-1}{\dlambdamin} \norm{\left( \dot L_G \right)_{max} }  t \\
        % \leq \frac{ \mathcal{O}(N-1)}{\dlambdamin} \norm{\left( \dot L_G \right)_{max} } \\
    \end{align*}
    Finally taking the maximum norm of the rate of change in $L_G$ over the entire time duration we have the following
    \begin{align*}
        % \norm{\Delta v} \leq \sum_{i} \norm{\Delta v_i} \\ 
        % \norm{\Delta v} \leq \int_{t} \norm{\Delta v(t)} dt \\ 
        % \norm{\Delta v} \leq \int_{t} \norm{\frac{ N-1}{\dlambdamin} \norm{\dot L_G(t)}} dt \\
        % \leq \int_{t} \frac{ N-1}{\dlambdamin} \norm{\left( \dot L_G \right)_{max} } dt \\
        \norm{\Delta v} \leq \frac{ N-1}{\dlambdamin} \left( \norm{ \dot L_G  }\right)_{max}  \int_{t} dt \\
        \norm{\Delta v} \leq \frac{ N-1}{\dlambdamin} \left( \norm{ \dot L_G  }\right)_{max}  t \\
        \leq \frac{ \mathcal{O}(N-1)T}{\dlambdamin} \left( \norm{ \dot L_G  }\right)_{max} \\
    \end{align*}
    where $\left( \dot L_G \right)_{max}$ is the maximum of the norm of the rate of change of $L_G$ over all timesteps considered and we absorb the total time $t$ into the constant factor considering finite time. Thus we have,
    \begin{align*}
        \norm{\dot \ftmat_G} &= \sqrt{\sum_i \norm{v_i}^2} \\
        &\leq \frac{\sqrt{N}(N-1)T}{\dlambdagmin} \left( \dot L_G \right)_{max} \\
        &\leq \frac{\mathcal{O}(N^{\frac{3}{2}}T)}{\dlambdagmin} \left( \dot L_G \right)_{max} \\
    \end{align*}
    Thus using theorem 8 from \citep{tensor_product_norms} and the fact that $\norm{\ftmat_{T}^t} = 1$ we have,
    \begin{align*}
        \norm{\ftmat_D-\ftmat_J} &= \norm{\left( [\ftmat_{T} \boxtimes (\{\ftmat_{G_t}\} - \ftmat_G) ] \right)} \\
        &=  \left( \int_{\omega} \norm{\ftmat_{T}^{\omega} \otimes (\ftmat_{G_{\omega}} - \ftmat_G) }^2 \right)^{\frac{1}{2}} \\
        &\leq  \left( \int_{\omega} \norm{\ftmat_{T}^{\omega}}^2 \norm{ (\ftmat_{G_{\omega}} - \ftmat_G) }^2 \right)^{\frac{1}{2}} \\
        &\leq  \left( \int_{\omega} \norm{ (\ftmat_{G_{\omega}} - \ftmat_G) }^2 \right)^{\frac{1}{2}} \\
        &\leq  \left( \int_{\omega} \norm{ \Delta \ftmat_{G_{\omega}} }^2 \right)^{\frac{1}{2}} \\
        &\leq  \left( \int_{\omega} \left( \frac{\mathcal{O}(N^{\frac{3}{2}}T)}{\dlambdagmin} \norm{ \dot L_{G_{\omega}} }_{max} \right) ^2 \right)^{\frac{1}{2}} \\
        &\leq  \left( \omega_{max} \left( \frac{\mathcal{O}(N^{\frac{3}{2}}T)}{\dlambdagmin} \norm{ \dot L_{G} }_{max} \right) ^2 \right)^{\frac{1}{2}} \\
        \norm{\ftmat_D-\ftmat_J} &\leq  \left( \frac{\mathcal{O}(N^{\frac{3}{2}}T \omega_{max}^{\frac{1}{2}})}{\dlambdagmin} \norm{ \dot L_{G} }_{max} \right) \\
    \end{align*}
    % where $T$ is the number of timesteps considered. 
    where $T$ is the time duration over which the evolution ghof the graphs is considered. 

    For the second part we can show that
    \begin{align*} 
        \norm{ \ftmat_J - \ftmat_{AD}} &= \norm{\dot \ftmat_J} \\
         &= \sqrt{\int_{i=0}^{N\omega_{max}} \norm{v_i}^2} \\
        &\leq \frac{\sqrt{N\omega_{max}}(N\omega_{max}-1)T}{\dlambdajmin} \left( \dot L_J \right)_{max} \\
        &\leq \frac{\mathcal{O}((N\omega_{max})^{\frac{3}{2}}T)}{\dlambdajmin} \left( \dot L_J \right)_{max} \\
    \end{align*}
    Also we have,
    \begin{align*}
        L_J &= L_T \oplus L_G \\
        &= L_T \otimes I_N + I_T \otimes L_G \\
        \dot L_J &= I_T \otimes \dot L_G \\
        \norm{ \dot L_J} &= \norm{ I_T \otimes \dot L_G }\\
        \norm{ \dot L_J} &= \norm{ I_T } \norm{ \dot L_G }\\
        \norm{ \dot L_J} &= \int_{\omega} \left( d \omega \right)^{\frac{1}{2}} \norm{ \dot L_G }\\
        \norm{ \dot L_J} &= \sqrt{\omega} \norm{ \dot L_G }\\
    \end{align*}
    \begin{align*} 
        \therefore \norm{ \ftmat_J - \ftmat_{AD}} &\leq \frac{\mathcal{O}(N^{\frac{3}{2}})}{\dlambdajmin} \left( \dot L_J \right)_{max} \\
        &\leq \frac{\mathcal{O}(N^{\frac{3}{2}} \omega^2 T)}{\dlambdajmin} \left( \dot L_G \right)_{max} \\
    \end{align*}

    Combining the two parts we get the result
    \begin{equation}
    \begin{aligned}
        \norm{\ftmat_D - \ftmat_{AD}} &\leq \norm{\ftmat_D - \ftmat_J}  + \norm{\ftmat_J - \ftmat_{AD} } \\
        &\leq \mathcal{O}\left(\frac{N^{\frac{3}{2}} T \omega_{max}^{\frac{1}{2}}}{(\Delta \lambda_G)_{min}} + \frac{N^{\frac{3}{2}} T \omega_{max}^{2}}{(\Delta \lambda_J)_{min}}\right) \left( \norm{\dot{L_G}} \right)_{max}
    \end{aligned}
    \end{equation}

    For the discrete case this bound becomes 
    \begin{equation}
        \norm{\ftmat_D - \ftmat_{AD}} \leq \mathcal{O}\left(\frac{(NT)^{\frac{3}{2}}}{(\Delta \lambda_G)_{min}} + \frac{(NT^2)^{\frac{3}{2}}}{(\Delta \lambda_J)_{min}}\right) \left( \norm{\dot{L_G}} \right)_{max}    
    \end{equation}
    
    % % for some scalar $c_k$ (the presence of this term reflects the freedom to multiply $v_k$ by a scalar).  Similar considerations for the adcollective give
    % \begin{equation}
    %     \dot w_k^* = \sum_{j \neq k} \frac{w_k^* \dot A v_j}{\lambda_k - \lambda_j} w_j^* - c_k w_k^*         
    % \end{equation}

\end{proof}

We thus see that as the graph evolves infinitesimally the difference between $\ftmat_D$ and $\ftmat_{AD}$ is bounded from above by the change in the graph matrix representation. This is desirable since it allows us to approximate $\ftmat_{AD}$ (formed by the eigendecomposition of \ljdmat) which has a physical interpretation using the defined $\ftmat_D$ which is simple to compute, when the graph changes in a stable manner.
In such cases, \emph{EFT} therefore characterizes signals on the dynamic graph by its proximity (projection) to the optimizers of $S_2(X)$ meaning high (collective dynamic graph) frequency components correspond to sharply varying signals and low frequency components to smoother signals.
Having derived the transform, we next state and prove the properties of the proposed transform in the next section.

% \subsection{Properties of Proposed Transform}
\section{Properties of Proposed Transform}\label{sec:properties}
% \color{asparagus}{
Having designed the Evolving Graph Fourier Transform, we now look at some of the properties of the transform. 
% $\emph{EFT}$ can be simulated by $\emph{GFT}$; 
% formally:
The below defines some properties of \emph{EFT} before learning its representations and applying to downstream tasks (proofs are in appendix section \ref{sec:proofs_apndx}).
% \begin{table}[t]
\begin{mdframed}
\begin{property}\label{prop_equivalence_dft} (Equivalence in special case)
    Consider $\ftmat_T$ to be the time Fourier transform and $\ftmat_{G_t}$ to be the Graph Fourier transform at time $t$. Let $\ftmat_{\mathcal{J_D}}$ be the Graph fourier transform of $\mathcal{J_D}$. In the special case of $G_{t_i}=G_{t_{j}} \forall i,j \in \{T\}$ we have $(\ftmat_{\mathcal{J_D}})_i^j = (\ftmat_{D})_i^j = (\ftmat_T \otimes \{\ftmat_{G_t}\})_{i}^{j \left\lfloor \frac{j}{N} \right\rfloor}$.
\end{property}

% Let us define some further properties of \emph{EFT} before applying it to the link prediction problem (proofs are in appendix section \ref{sec:proofs_apndx}).
% TOOD: ref later

\begin{property}\label{prop_invertible}
    EFT is an invertible transform and the inverse is given by $\mathbf{EFT}^{-1}(\hat{X})_{i}^{j} = \left( \ftmat_{G}^{-1} \hat{X} \right)_{i}^{k k} \left( \ftmat_T^{\top^{*}} \right)_{k}^{j}$ in matrix form and $\mathbf{EFT}^{-1}(\hat{x})_{j*N+i} = \left( \ftmat_T^{*} \otimes \ftmat_{G}^{-1} \right)_{j*N+i}^{k \left\lfloor \frac{k}{N} \right\rfloor} \hat{x}_k$ in vector form.
\end{property}

\begin{property}\label{prop_unitary}
    EFT is a unitary transform if and only if GFT is unitary at all timesteps considered i.e. $\ftmat_{D} \ftmat_{D}^{*} = I_{NT}$ iff  $\ftmat_{G_t} \ftmat_{G_t}^{*} = I_{N}, \forall t$.
\end{property}

%We know that \emph{GFT} is a unitary transform for all symmetric matrices that representat undirected graphs. Thus \emph{EFT} is also unitary for undirected graphs.

\begin{property}\label{prop_invariant_order}
    EFT is invariant to the order of application of DFT or GFT on signal X.
\end{property}
% The above property can be observed from equation \ref{jgft_einsum} using the fact that matrix multiplication is associative. 
\end{mdframed}
% \vspace{-4mm}
% \end{table}

% \textbf{Note:}
Property \ref{prop_unitary} allows us to define the stability of the proposed transform. Consider the EFT matrix $E$ and the signal vector $x$ (normalized). The transform would be given by $Ex$. Now consider the perturbed matrix $E+\epsilon$, where $\epsilon$ is the (fixed) perturbation. The relative difference between the output would be $\norm{(E+\epsilon)x-Ex}/\norm{Ex}=\norm{\epsilon x}/\norm{Ex}$. Since $E$ is orthogonal, $x$ is not in the null space of $E$ and so the relative difference is bounded by $\epsilon$. So a small change in $E$ should cause a small change in the output as desired.

As seen in property \ref{prop_equivalence_dft}, $\emph{EFT}$ can be simulated by $\emph{GFT}$ in the special case that the graph structure does not change with time.
The illustration between other transforms is in Figure \ref{fig:ctdft_motivation}. 
%The above as well as additional equivalences between various transforms as noted by \citep{loukas2016frequency} are given in the figure \ref{fig:ctdft_motivation}. 
The figure shows transforms (\emph{GFT, JFT, DFT, EFT)} in a circle, and arrows from one transform to the next indicate that the source transform can be obtained by the destination transform using the simulation annotated on the edges. 
Please note that the analysis has been performed for one-dimensional signals. However, the same holds true for higher dimensions as well by conducting the \emph{EFT} dimension-wise. 
\textcolor{asparagus}{Here dimension-wise means the feature dimension of a node. Each node may have a multidimensional signal residing on it and the EFT can be independently applied to each channel or dimension of the node signals on the dynamic graph.}
Below subsection provides proofs for  the above stated properties.
% With the representations obtained using the proposed transform, we intend to perform filtering in spectral space for dynamic graphs. Since our idea is to perform collective filtering along the vertex and temporal domain in \emph{EFT}, we need two modules to compute $\ftmat_{G_t}$ (vertex aspect) and $\ftmat_T$ (temporal aspect), respectively, in equation \ref{jgft_einsum} of \emph{EFT}. 
% We explain the task specific implementation of these modules in the appendix \ref{sec:architecture_apndx} and focus more on the representations and results in the following sections. 
% }

\subsection{Proofs of Properties}
In this section we now prove the properties stated above. We repeat the statements for completeness.
Though \eft and \aeft are not same in the general case they are equivalent when the graph structure does not change with time. Below result proves the result for the discrete case with graphs sampled at uniform timesteps
\begin{property}\label{prop_equivalence_dft_apndx} (Special Equivalence between \aeft and \eft)
    Consider $\ftmat_T$ to be the time fourier transform and $\ftmat_{G_t}$ to be the Graph fourier transform at time $t$. Let $\ftmat_{\mathcal{J_D}}$ be the Graph fourier transform of $\mathcal{J_D}$. In the special case of $G_{t_i}=G_{t_{j}} \forall i,j \in \{T\}$ we have $(\ftmat_{\mathcal{J_D}})_i^j = (\ftmat_{D})_i^j = (\ftmat_T \otimes \{\ftmat_{G_t}\})_{i}^{j \left\lfloor \frac{j}{N} \right\rfloor}$.
\end{property}
\begin{proof}
As before consider $\{L_{G_t}\} \in R^{N \times N \times T}$ to be the Laplacian of the graphs at each timestep with eigenvalues $\lambda_i^{t}$ where $i \in {N}, t \in {T}$. Let $L_T \in R^{T \times T}$ be the Laplacian of the time adjacency matrix with eigenvalues $\mu_j$ where $j \in {T}$.
The Laplacian of the collective graph $\mathcal{J_D}$ is expressed as 
\begin{equation*}
    (L_{\mathcal{J_D}})_i^j = (L_T \oplus \{L_{G_t}\})_i^{j \left\lfloor \frac{j}{N} \right\rfloor} = L_T \otimes I_N + (I_T \otimes \{ L_{G_t} \})_{i}^{j \left\lfloor \frac{j}{N} \right\rfloor}
\end{equation*}
Consider $x_1, x_2, \dots x_p$ to be the linearly independent right eigenvectors of $L_T$ and $z_1^t, z_2^t, \dots z_{q_t}^t$ to be the linearly independent right eigenvectors of $L_{G_t}$.  Consider the vector $y_j = (x_k \otimes {z_l^t})_j^{\left\lfloor \frac{j}{N} \right\rfloor}, y \in R^{NT}$. Then we have
\begin{align*}
    & (L_{\mathcal{J_D}} y)_i = (L_T \otimes I_N)_i^j y_j + (I_T \otimes \{ L_{G_t} \})_{i}^{j \left\lfloor \frac{j}{N} \right\rfloor} y_j \\
    &= (L_T \otimes \{I_N\})_i^{j \left\lfloor \frac{j}{N} \right\rfloor} (x_k \otimes {z_l^t})_j^{\left\lfloor \frac{j}{N} \right\rfloor} + (I_T \otimes \{ L_{G_t} \})_{i}^{j \left\lfloor \frac{j}{N} \right\rfloor} (x_k \otimes {z_l^t})_j^{\left\lfloor \frac{j}{N} \right\rfloor} \\
    &= (L_T \otimes \{I_N\} \square x_k \otimes {z_l^t})_i^{\left\lfloor \frac{i}{N} \right\rfloor} + (I_T \otimes \{ L_{G_t} \square x_k \otimes {z_l^t}\})_{i}^{\left\lfloor \frac{i}{N} \right\rfloor} \\
    &= (L_T x_k \otimes \{I_N\} \square {z_l^t})_i^{\left\lfloor \frac{i}{N} \right\rfloor} + (I_T x_k \otimes \{ L_{G_t} \square {z_l^t}\})_{i}^{\left\lfloor \frac{i}{N} \right\rfloor} \\
    &= (\mu_k x_k \otimes {z_l^t})_i^{\left\lfloor \frac{i}{N} \right\rfloor} + (x_k \otimes \{ \lambda_{l}^t {z_l^t}\})_{i}^{\left\lfloor \frac{i}{N} \right\rfloor} \\
    &= (\mu_k x_k \otimes {z_l^t} + x_k \otimes \{ \lambda_{l}^t {z_l^t}\})_{i}^{\left\lfloor \frac{i}{N} \right\rfloor} \\
    &= ( x_k \otimes {z_l^t} diag(\{\mu_k\}) + x_k \otimes \{ {z_l^t}\} diag(\{\lambda_{l}^t\}))_{i}^{\left\lfloor \frac{i}{N} \right\rfloor} \\
    &= ( (x_k \otimes {z_l^t}) diag(\{\mu_k + \lambda_{l}^t\}))_{i}^{\left\lfloor \frac{i}{N} \right\rfloor} \\
\end{align*}
where $\square$ indicates timestep (column) wise product and $diag(.)$ operator converts a vector to a diagonal matrix. In the special case where $G_{ti}=G_{tj} \forall i,j \in {T}$ we have  $\lambda_{l}^{ti}=\lambda_{l}^{tj}$. Thus we get
\begin{align*}
    (L_{\mathcal{J_D}} y)_i &= ( (x_k \otimes {z_l^t}) diag(\{\mu_k + \lambda_{l} I_T\}))_{i}^{\left\lfloor \frac{i}{N} \right\rfloor} \\
    &= ( \mu_k + \lambda_{l} (x_k \otimes {z_l^t}) diag(\{ I_T\}))_{i}^{\left\lfloor \frac{i}{N} \right\rfloor} \\
    &= ( \mu_k + \lambda_{l} (x_k \otimes {z_l^t}))_{i}^{\left\lfloor \frac{i}{N} \right\rfloor} \\
    &= ( \mu_k + \lambda_{l}) y_{i} \\
\end{align*}
Thus $y_j = (x_k \otimes {z_l^t})_j^{\left\lfloor \frac{j}{N} \right\rfloor}$ is the eigenvector of $L_{\mathbf{J_D}}$ with eigenvalue $\mu_k + \lambda_{l}$. But $y$ is nothing but one of the columns of $\ftmat_D^{*}$. By the rank nullity theorem, the row spaces of the transform matrices $\ftmat_D$ and \emph{GFT} of $\mathcal{J_D}$ share the same orthogonal basis. Thus the two transforms are equivalent in this case.
\end{proof}

Note the eigenvalues ($\Lambda_T \oplus \Lambda_G$) obtained in the result above are exactly the ones used for plotting the frequency response of \eft as we compress the sequence of graphs into a single dynamic graph. 

Next we prove some properties of \eft as stated in the main paper

\begin{property}\label{prop_invertible_apndx}
    EFT is an invertible transform and the inverse is given by $\mathbf{EFT}^{-1}(\hat{X})_{i}^{j} = \left( \ftmat_{G}^{-1} \hat{X} \right)_{i}^{k k} \left( \ftmat_T^{\top^{*}} \right)_{k}^{j}$ in matrix form and $\mathbf{EFT}^{-1}(\hat{x})_{j*N+i} = \left( \ftmat_T^{*} \otimes \ftmat_{G}^{-1} \right)_{j*N+i}^{k \left\lfloor \frac{k}{N} \right\rfloor} \hat{x}_k$ in vector form.
\end{property}
\begin{proof}

We begin by noting the expression for \emph{EFT} ($\ftmat_{D}$)
\begin{equation*}
    \left(  \ftmat_{D} \right)_{j}^{i} = \left( \ftmat_T \otimes {\ftmat_{G_t}} \right)_{i}^{j \left\lfloor \frac{j}{N} \right\rfloor}
\end{equation*}
where $\ftmat_{G_t} \in R^{N \times N}$ is the graph fourier transform of the graph at time $t$, $\ftmat_T \in R^{T \times T}$ is the time fourier transform. Let $\invftmat_{G_{t}} = \ftmat_{G_t}^{-1}$  be the inverse graph fourier transform of the graph at timestep $t$ and $\invftmat_{T} = \ftmat_T^{*}$ be the inverse time fourier transform.
% Let $\invftmat_{D}$ represent the inverse of $\ftmat_{D}$ if it exists.

We can write $\ftmat_{D}$ as a block matrix in the following form
\begin{align*}
    \ftmat_{D} &= \left[ CB^1, CB^2, \dots CB^T \right] \\
    CB^i &= \ftmat_T^{i} \otimes \ftmat_{G_i}
\end{align*}
where $\ftmat_T^{i}$ is the $i$-th column of $\ftmat_T$ and $CB^i \in R^{NT \times N}$.

Consider $\invftmat_{D}$ in a similar but row block format as follows
\begin{gather}
    % \invftmat_{D} &= \left[ RB_1, RB_2, \dots RB_T \right] \\
    \invftmat_{D} 
    =
    \begin{bmatrix}
    RB_1 \\
    RB_2 \\
    \vdots \\
    RB_T
    \end{bmatrix} \\
    RB_i = \invftmat_{Ti} \otimes \invftmat_{G_i}
\end{gather}
where $\invftmat_{Ti}$ is the $i$-th row of $\invftmat_T$ and $RB_i \in R^{N \times NT}$.

Now taking the matrix product of $\invftmat_D$ and $\ftmat_D$ we get
\begin{gather*}
    % \invftmat_{D} &= \left[ RB_1, RB_2, \dots RB_T \right] \\
    \invftmat_D \ftmat_D
    =
    \begin{bmatrix}
    RB_1 \\
    RB_2 \\
    \vdots \\
    RB_T
    \end{bmatrix} 
    \begin{bmatrix}
    CB^1 & CB^2 & \dots CB^T
    \end{bmatrix} \\
    = 
    \begin{bmatrix}
    RB_1 CB^1 &  RB_1 CB^2 \dots \\
    RB_2 CB^1 & RB_2 CB^2 \dots \\
    \vdots \\
    RB_T CB^1 & RB_T CB^2 \dots
    \end{bmatrix}
\end{gather*}
We can verify that $RB_i CB^j$ evaluates to the following
\begin{align}
    RB_i CB^j &= \left( \invftmat_{Ti} \otimes \invftmat_{G_i} \right) \left( \ftmat_T^j \ftmat_{G_j} \right)\\
    &= \left( \invftmat_{Ti} \ftmat_T^{j} \right) \otimes \left( \invftmat_{G_i} \ftmat_{G_j} \right) \\
\end{align}
Now the columns of $\invftmat_{T}$ form the eigenvectors of a circulant matrix ($L_T$). Also we know that if columns form basis of column space then rows form the basis of the row space. Thus we have
\begin{align}
    \invftmat_{Ti} \ftmat_T^{j} &= 
    \begin{dcases}
    1 ,& \text{if } \text{$i=j$}\\
    0,              & \text{otherwise}
    \end{dcases} \\
    \invftmat_{G_i} \ftmat_{G_i} &= I_N \\
    \therefore RB_i CB^j &= 
    \begin{dcases}
    I_N ,& \text{if } \text{$i=j$}\\
    0,              & \text{otherwise}
    \end{dcases} \\
\end{align}
Thus we have shown that $\invftmat_D \ftmat_D = I_{NT}$. Thus $\invftmat_D$ is a left inverse of $\ftmat_D$. We know that for a square matrix left inverse is also the right inverse and can be readily verified in a similar manner. 
Thus \emph{EFT} is invertible and the inverse of the transformed signal in vector form is $\left( \invftmat_T \otimes \{\invftmat_{G_t}\} \right)_{i}^{j \lfloor \frac{j}{N} \rfloor} \hat{x_j} = \left( \ftmat_T^{*} \otimes \{\ftmat_{G_t}^{-1}\} \right)_{i}^{j \lfloor \frac{j}{N} \rfloor} \hat{x_j}$. Similarly for the matrix form of the signal we have the inverse of the transform given as $\left( \{\invftmat_{G_t}\} \hat{X} \right)_{i}^{j j} \left( \invftmat_T^{\top} \right)_j^k = \left( \{\ftmat_{G_t}^{-1}\} \hat{X} \right)_{i}^{j j} \left( \ftmat_T^{\top^{*}} \right)_j^k$.
\end{proof}

\begin{property}\label{prop_unitary_apndx}
    EFT is a unitary transform if and only if GFT is unitary at all timesteps considered i.e. $\ftmat_{D} \ftmat_{D}^{*} = I_{NT}$ iff  $\ftmat_{G_t} \ftmat_{G_t}^{*} = I_{N}, \forall t$
\end{property}
\begin{proof}
This property can be proved in a similar manner as in proof of property \ref{prop_invertible_apndx}.  
The only difference here is we consider $\invftmat_D$ to be the transposed conjugate of $\ftmat_D$ rather than inverse i.e. $\invftmat_D = \ftmat_D^{*}$ and also $\invftmat_{G_i} = \ftmat_{G_i}^{*}$.
Similar to the previous proof we have the following
\begin{gather*}
    \invftmat_D \ftmat_D
    =
    \begin{bmatrix}
    RB_1 \\
    RB_2 \\
    \vdots \\
    RB_T
    \end{bmatrix} 
    \begin{bmatrix}
    CB^1 & CB^2 & \dots CB^T
    \end{bmatrix} \\
    = 
    \begin{bmatrix}
    RB_1 CB^1 &  RB_1 CB^2 \dots \\
    RB_2 CB^1 & RB_2 CB^2 \dots \\
    \vdots \\
    RB_T CB^1 & RB_T CB^2 \dots
    \end{bmatrix}
\end{gather*}   
\begin{align*}
    RB_i CB^j &= \left( \invftmat_{Ti} \otimes \invftmat_{G_i} \right) \left( \ftmat_T^j \ftmat_{G_j} \right)\\
    &= \left( \invftmat_{Ti} \ftmat_T^{j} \right) \otimes \left( \invftmat_{G_i} \ftmat_{G_j} \right) \\
\end{align*}
\begin{gather*}
    \therefore \invftmat_D \ftmat_D
    = 
    \begin{bmatrix}
    \left( \invftmat_{T1} \ftmat_T^{1} \right) \otimes \left( \invftmat_{G_1} \ftmat_{G_1} \right) &  \left( \invftmat_{T1} \ftmat_T^{2} \right) \otimes \left( \invftmat_{G_1} \ftmat_{G_2} \right) \dots \\
    \left( \invftmat_{T2} \ftmat_T^{1} \right) \otimes \left( \invftmat_{G_2} \ftmat_{G_1} \right) & \left( \invftmat_{T2} \ftmat_T^{2} \right) \otimes \left( \invftmat_{G_2} \ftmat_{G_2} \right) \dots \\
    \vdots \\
    \left( \invftmat_{TT} \ftmat_T^{1} \right) \otimes \left( \invftmat_{G_T} \ftmat_{G_1} \right) &\left( \invftmat_{TT} \ftmat_T^{2} \right) \otimes \left( \invftmat_{G_T} \ftmat_{G_2} \right) \dots
    \end{bmatrix} \\
    = 
    \begin{bmatrix}
    1 \otimes \left( \invftmat_{G_1} \ftmat_{G_1} \right) &  0 \otimes \left( \invftmat_{G_1} \ftmat_{G_2} \right) \dots \\
    0 \otimes \left( \invftmat_{G_2} \ftmat_{G_1} \right) & 1 \otimes \left( \invftmat_{G_2} \ftmat_{G_2} \right) \dots \\
    \vdots \\
    0 \otimes \left( \invftmat_{G_T} \ftmat_{G_1} \right) & 0 \otimes \left( \invftmat_{G_T} \ftmat_{G_2} \right) \dots
    \end{bmatrix} \\
    = 
    \begin{bmatrix}
    \left( \ftmat_{G_1}^{*} \ftmat_{G_1} \right) &  0  \dots \\
    0 & \left( \ftmat_{G_2}^{*} \ftmat_{G_2} \right) \dots \\
    \vdots \\
    0 & 0 \dots
    \end{bmatrix} \\
\end{gather*}   

\noindent \underline{Part 1:}
If $\ftmat_{G_1}$ is unitary then $\ftmat_{G_1}^{*} = \ftmat_{G_1}^{-1}$.
Thus in this case $\invftmat_D \ftmat_D = I_{NT}$ which implies $\invftmat_{D} = \ftmat_D^{*} = \ftmat_D^{-1}$ implying $\ftmat_D$ is unitary.

\noindent \underline{Part 2:}
Considering $\ftmat_D$ is unitary whic means $\invftmat_{D} = \ftmat_D^{*} = \ftmat_D^{-1}$.
Thus $\invftmat_D \ftmat_D = I_{NT}$ and so $\ftmat_{G_i}^{*} \ftmat_{G_i} = I_N \xrightarrow{} \ftmat_{G_i}^{-1} = \ftmat_{G_i}^{*}$. $\therefore \ftmat_{G_i}$ is unitary proving the 2nd part and completing the proof.
\end{proof}

\begin{property}\label{prop_invariant_order1}
    EFT is invariant to the order of application of DFT or GFT on signal X.
\end{property}
The above property can be observed from equation \ref{jgft_einsum} using the fact that matrix multiplication is associative. 
\\

\newpage

\section{Datasets}
% talk of 6 datasets

\begin{wraptable}{l}{0.42\textwidth}
% \begin{table}[t] 
	\centering
	\caption{The statistics of the Large scale Dynamic graph datasets for link prediction.}
	%\setlength{\tabcolsep}{0.8mm}
	% \resizebox{0.9\columnwidth}{!}
        \resizebox{0.4\textwidth}{!}
	{
         \begin{tabular}{ccccc}
			\toprule
			\bf SR Datasets & \bf Beauty & \bf Games & \bf CDs \\
			\midrule
                \bf $\#$ of Users	& 52,024	& 31,013	& 17,052 \\
                \bf $\#$ of Items	& 57,289	& 23,715	& 35,118 \\
                \bf $\#$ of Interactions	& 394,908	& 287,107	& 472,265 \\
                \bf Average length	& 7.6	& 9.3	& 27.6 \\
                \bf Density	& 0.01\%	& 0.04\%	& 0.08\% \\
   %              \midrule
			% \bf SBR Datasets & \bf Yoochoose1\_64 & \bf Yoochoose1\_4 & \bf Diginetica \\
			% \midrule
   %              \bf $\#$ of Sessions	& 425,757	& 5,973,643	& 780,328 \\
   %              \bf $\#$ of Items	& 16,766	& 29,618	& 43,097 \\
   %              \bf $\#$ of Interactions	& 557,248	& 8,326,407	& 982,961 \\
   %              \bf Average length	& 6.16	& 5.71	& 5.12 \\
   %              \bf Density	& 0.008\%	& 0.005\%	& 0.003\% \\
			\bottomrule
	\end{tabular}
        }
	\label{tab:dataset_stats}
% \end{table}
\end{wraptable} 

\textbf{Continuous Time Dynamic Graph link prediction dataset in sequential recommendation setting:} 
%For a fair comparison, we inherit datasets, baselines, and settings from the previous best dynamic graph method DGSR \citep{DGSR}. 
For showing the efficacy of our method on large dynamic graphs, we perform experiments on three real-world e-commerce datasets (cf., Table \ref{tab:dataset_stats}) for sequential recommendation.  Specifically, we pose the sequential recommendation as a link prediction problem on temporal graphs.
%These datasets have the product reviews and metadata of the interactions along with the associated timestamp. We only use the interaction graph and discard the review and metadata information for fair comparison with baselines. 
%Similar to previous works \citep{DGSR}, we filter users that interacted with less than five items. 
The penultimate and last interactions are used for validation and testing, respectively. The graphs at each interaction timestamp is constructed as detailed in \citep{DGSR}  i.e., at time $t$, the subgraph ($G_t$) containing all interactions till $t$ is considered. Then the $m$-hop neighborhood $G^m_t(u)$ around the user $u$ is sampled from it. The next item to predict is the item ($i_{t+1}$) interacted with at time $t+1$. Thus the training set would contain $(G^m_1(u), i_{2}), (G^m_2(u), i_{3}) \dots (G^m_{T-2}(u), i_{T-1})$ and the test set would have $(G^m_{T-1}(u), i_{T})$. The graph construction is done in the preprocessing phase to speed-up training and testing.

\begin{wraptable}{l}{0.42\textwidth}
% \begin{table}[ht] 
% \small
% \vspace{-24pt}
\centering
\caption{Statistics and details for link prediction on the benchmark dynamic graph datasets. LP is the abbreviation for Link Prediction and NC is for Node Classification.}
\label{tab:dataset}
% \vspace{-5pt}
\resizebox{0.4\textwidth}{!}{
\begin{tabular}{cccccc}
\hline
& \# Nodes & \# Edges & \# Time Steps & Task\\
&          &          & (Train / Val / Test) & &   \\
\hline

% BC-OTC   & 5,881   & 35,588  & 95 / 14 / 28\\
% BC-Alpha & 3,777   & 24,173  & 95 / 13 / 28 \\
% Reddit   & 55,863  & 858,490 & 122 / 18 / 34 \\
SBM      & 1,000   & 4,870,863 & 35 / 5 / 10 & LP  \\
UCI      & 1,899   & 59,835  & 62 / 9 / 17 & LP\\
AS       & 6,474   & 13,895  & 70 / 10 / 20 & LP \\ 
Elliptic & 203,769 & 234,355 & 31 / 5 / 13 & NC \\
Brain   & 5,000  & 1,955,488 & 10 / 1 / 1 & NC \\
\hline
% \vspace{-20pt}
\end{tabular}
}
% \end{table}
\end{wraptable}

\noindent \textbf{Benchmark Dynamic Graph Datasets}:
Table \ref{tab:dataset} summarizes datasets for link prediction on benchmark dynamic graph datasets.
Each dataset contains a sequence of time-ordered graphs. 
SBM is a synthetic dataset to simulate evolving community structures. UCI dataset is a student community network where nodes represent the students, and the edges represent the messages exchanged between them. AS dataset summarizes a temporal communication network indicating traffic flow between routers. 
The Elliptic (Ell) dataset delineates legitimate versus unlawful transactions within the elliptic network of Bitcoin transactions. 
In this context, nodes symbolize individual transactions, while edges correspond to the pathways of monetary transfers. 
The Brain (Brn) dataset focuses on nodes representing minuscule cerebral regions or cubes, with the edges signifying their interconnections.

\textbf{Synthetic Dataset}\label{sec:synthetic_dataset_apndx}
Consider the dynamic graph over T timesteps. Thus we have T graph snapshots. We compute the eigenvectors at each snapshot and place them over the graph's nodes. Moreover, for each node (spread over T timesteps) we compute a periodic signal that is added to the eigenvector component $Evec(G_t)$. So the expression for the noise added to the signal would be: $X(i,t) = \sum_k \alpha_k Evec(G_t)[i,k] + \sum_f \beta_f * e^{i\omega_t f}[t]$. In our experiments, we have used only one randomly chosen eigenvector. Also we consider only a single sinusoid frequency $\omega$. $\alpha_k, \beta_f$ are parameters and are set to $\frac{1}{2}$ in our experiments. For noise, we add to $X_(i,t)$ a signal taken from a Gaussian distribution with 0 mean i.e. $X(i,t) = X(i,t) + \mathcal{N}(0,\delta)$, where $\delta$ is the standard deviation.

\subsection{Experimental Setup}
 We implement our models using the DGL framework \citep{wang2019dgl} in the pytorch library \citep{paszke2017pytorch}. 
The hyperparameters are selected from the following search space: learning rate $\in [0.01,0.0003]$, $l_2$ regularization parameter $\alpha \in [0.01,0.00001]$, embedding and hidden layer dimensions $\in \{32,64,128\}$, filter order $\in \{2,4,8,16\}$, subgraph size $\in \{1,2,3,4\}$. 
The experiments are run on a single Tesla P100 GPU. We run our method for 5 runs per dataset and report the mean of the results. For the baselines we report the best results that have been reported unless mentioned otherwise. If results are not available we run baselines by using the implementation provided with default parameters and optimizing the hidden size (width) and layer number (depth) of the network.
Regarding graph construction, for the Sequential Recommendation (SR) datasets we use similar to \citep{DGSR}. For the Session Based Recommendation (SBR) setting we use the transition graph of the items in the sequence as in \citep{SRGNN}. We also try with higher order graphs, albeit without any gains, as reported in \citep{SRGNN}. 
Moreover, since for SBR the last item is of more significance to the prediction task and the datasets suffer from overfitting we modify the prediction layer accordingly and incorporate appropriate changes from baselines.

\subsection{Implementation} \label{sec:architecture_apndx}

We intend to perform filtering in spectral space for dynamic graphs using \emph{EFT}. Since our idea is to perform collective filtering along the vertex and temporal domain in \emph{EFT}, we need two modules to compute $\ftmat_{G_t}$ (vertex aspect) and $\ftmat_T$ (temporal aspect), respectively, in equation \ref{jgft_einsum} of \emph{EFT}. We now explain these modules in detail. 

% \subsubsection{Filtering in spectral space of the vertex domain}
\noindent \textbf{Filtering along the vertex domain:}
This module computes the convolution matrix $\ftmat_{G_t}$ in equation \ref{jgft_einsum}.
Consider the filter response $\hat{\Lambda}_l$ which is a diagonal matrix with diagonal values representing the magnitude of the corresponding frequency(eigenvalue). 
% Thus we have that $\ftmat_{G_t} = U \Lambda_{l_t} U^{*}$. 
% Thus we have that $\ftmat_{G_t} = U^{*}$. 
% Note that convolution by $\ftmat_{G_t}$ in the vertex domain is equivalent to filtering in the frequency domain.
% The equation \ref{graph_conv} requires the computation of the eigendecomposition and takes $\mathcal{O}(N^3)$ complexity where $N$ is the number of user/item nodes. This would increase by a factor $T$ of the number of snapshots considered to $\mathcal{O}(N^3 T)$. 
In order to avoid the computational cost of the eigendecomposition, we choose to approximate the it using polynomials. In this work, we use the Chebyshev polynomials \citep{defferrard2016convolutional}. 
Specifically, the frequency response of the desired filter is approximated as $\hat{\Lambda}_l = \sum_{k=0}^{\textcolor{black}{O_f}} c_k T_k(\Tilde{\Lambda})$,
% \begin{equation}\label{poly_approx}
% \small
%     \hat{\Lambda}_l = \sum_{k=0}^{\textcolor{black}{O_f}} c_k T_k(\Tilde{\Lambda})
% \end{equation}
where $O_f$ is the polynomial/filter order, $T_k$ is the Chebyshev polynomial basis, $\Tilde{\Lambda} = \frac{2\Lambda}{\lambda_{max}} - I$, $\lambda_{max}$ is the maximum eigenvalue and $c_k$ is the corresponding \textit{filter coefficients}. 
% By definition, the Chebyshev polynomials have a recursive formulation with basis $T_0(\Tilde{\Lambda})=I$, $T_1(\Tilde{\Lambda})=\Tilde{\Lambda}$ and beyond that $T_k(\Tilde{\Lambda}) = 2 \Tilde{\Lambda} T_{k-1}(\Tilde{\Lambda}) - T_{k-2}(\Tilde{\Lambda})$. 
% Thus, we can approximate the filtering operation in Equation \ref{graph_conv} as:
Thus, we can approximate the filtering operation as:
$X \ast \Lambda_l \approx U \left( \sum_{k=0}^{\textcolor{black}{O_f}} c_k T_k(\Tilde{\Lambda_l}) \right) U^*X = \sum_{k=0}^{\textcolor{black}{O_f}} c_k T_k(U \Tilde{\Lambda_l} U^*) X = \sum_{k=0}^{\textcolor{black}{O_f}} c_k T_k(\Tilde{L}) X$.
% % \begin{equation}\label{eq_conv_approx}
% % \begin{aligned}
% %     X \ast \Lambda_l &\approx U \left( \sum_{k=0}^{\textcolor{black}{O_f}} c_k T_k(\Tilde{\Lambda_l}) \right) U^*X = \sum_{k=0}^{\textcolor{black}{O_f}} c_k T_k(U \Tilde{\Lambda_l} U^*) X \\
% %    & = \sum_{k=0}^{\textcolor{black}{O_f}} c_k T_k(\Tilde{L}) X
% % \end{aligned}
% % \end{equation}
% \begin{equation*}\label{eq_conv_approx}
% \small
%     X \ast \Lambda_l \approx U \left( \sum_{k=0}^{\textcolor{black}{O_f}} c_k T_k(\Tilde{\Lambda_l}) \right) U^*X = \sum_{k=0}^{\textcolor{black}{O_f}} c_k T_k(U \Tilde{\Lambda_l} U^*) X = \sum_{k=0}^{\textcolor{black}{O_f}} c_k T_k(\Tilde{L}) X
% \end{equation*}
% This corresponds to an FIR filter of order $K$ \citep{smith1997scientist}. 
Having the filter coefficients $c_k$ as learnable parameters enables learning of filter for the task. 
% Here, the adjacency matrix is weighted with weights from the Dynamic Graph Encoding layer (equation \ref{eq_attn_ua}).
 % Before passing the embedding signals to the filter layer we project them appropriately using a feed forward layer. 
 The convolution $X \ast \Lambda_l$ gives the desired filtered response. 
 %However previous works have shown that repeated application of spectral graph filters could cause an "unsmooth spectrum" \citep{yang2022new} where the low frequency signals could get diminished. 
 %his could be detrimental to the task of sequential recommendation as low  frequency signals (neighborhood information) help to capture the interactions that benefit the task [cite fmlp]. 
 %We induce low frequency information in the filtered output as well by adding the convolved signals in the form of residual connections. 
 % The output of this module is given further for filtering along the temporal aspect. We can see from equation \ref{jgft_einsum} that this order could be interchanged as matrix multiplication is associative.

% Overall, this module performs convolution of input dynamic graph and computes embeddings of the nodes at respective timesteps. 
% % Up to this point, we can capture \emph{distant collaborative signals} among graph node interactions at a given timestep. 
% Now, as the dynamic graph evolves, we need to capture evolving signals along the temporal dimension done by the next module. We can see from property \ref{prop_invariant_order} that this order of filtering could be interchanged.
% % explain what it is we do
% % explain why it will help....collaborative signals etc.
% % how: explain chebyshev approximation, our equations etc.
% % 

% \subsubsection{Filtering in spectral space of the temporal domain}
\noindent \textbf{Filtering along the temporal Domain:}
After performing filtering in the vertex domain, we aim to filter over the temporal signals using  $\ftmat_T$ as in equation \ref{jgft_einsum}.
% (i.e. now unifying GFT with DFT).
%By our definition in \ref{jgft_einsum}, we would need to perform filtering operation over the node(user) embeddings at each timestamp. 
% % Considering we have compressed the dynamic graphs as a weighted graph with timestamp information encoded on the edges (cf., section \ref{subsubsec:dynamic_graph_construction}), we consider the immediate neighbors as the time-varying signal for a node. Specifically, we have the interacted items over time as a sequence of signals for a single user. For the sequential recommendation task, the formulation comes naturally as a user would be known by the items interacted with, and the changing preferences would be reflected in the time-varying signals. Hence, 
% We apply the Fourier transform dimension-wise to these signals, perform filtering in the frequency domain and convert back to the temporal domain by using the inverse Fourier transform. 
% % Such transformation helps in capturing the long-term content signals from user preferences. 
To apply the $\ftmat_T$ (Fourier transform), we must first ensure that the signals in sequences are sampled at uniform intervals. 
% However, it is not generally valid for the sequential recommendation problem, where users can interact with items anytime. 
In the continuous time setting, interactions between nodes could occur at anytime or the sampling could be non-uniform, 
%However, it is not generally valid for the link prediction applications in dynamic graphs, where interactions between nodes could occur at anytime or the sampling could be non-uniform. 
Thus, we perform a mapping from $R^{T \times d} \xrightarrow{} R^{T \times d}$ that aims to map the input space to a uniformly sampled space. For computational reasons, we select the current and next embeddings (with positional information) along with the timestamp information ($E_t(t) \in R^{d}$) for getting the mapped embedding akin to interpolation. 
% learnt interpolation of the signals between the interactions in order to have them at equal intervals.In an ideal case, we need inputs from all timesteps for the mapping, however 
Formally, let $X_{t}^i \in R^{d}$ be the embeddings of the node at time $t$. This is first mapped to the interpolated space using a universal approximator: $X_{t} = W^i_2 \sigma^i (W^i_1 [X_{t}^i;X_{t+1}^i;E_t(t)] + b^i_1) + b^i_2$, 
% \begin{equation*}
%     X_{t} = W^i_2 \sigma^i (W^i_1 [X_{t}^i;X_{t+1}^i;E_t(t)] + b^i_1) + b^i_2
% \end{equation*}
where $W^i_1$, $W^i_2$, $b^i_1$, $b^i_1$ are learnable parameters and $\sigma^i$ is a non-linearity. We call this module the \emph{time encoding layer} 
% in Figure \ref{fig:ctdft_arch}
, which is essential for applying Fourier transform along the temporal dimension. 
Let $X = X_{t} \in R^{T \times d}$ be the interpolated sequence of embeddings of the node. This is converted to the frequency domain ($\hat{X} \in R^{T \times d}$) using the DFT matrix $\ftmat_T$ as $\hat{X} = \ftmat_T X_t$
% \begin{equation*}
%     \hat{X} = \ftmat_T X_t
% \end{equation*}
% Then we multiply the above equation by a temporal filter $F_T \in R^{T \times d}$ that represents the magnitude and phase of the frequency components of every dimension up to the Nyquist's frequency ($\omega_{max}=\frac{\text{Sampling Frequency}}{2}$) \citep{leis2011digital}. The multiplication is performed element-wise for filtering in the frequency domain to obtain $\hat{X}_f = F_T \odot \hat{X}$.
Then we multiply $\hat{X}$ element-wise by a temporal filter $F_T \in R^{T \times d}$ to obtain the filtered signal $\hat{X}_f = F_T \odot \hat{X}$ which is then converted back to the temporal domain by using the inverse transform $\ftmat_T^*$ to get $X_f = \ftmat_T^* \hat{X}_f$.
% \begin{equation*}
%     \hat{X}_f = F_T \odot \hat{X}
% \end{equation*}
% The filtered signal is then converted back to the temporal domain by using the inverse transform $\ftmat_T^*$ to get $X_f = \ftmat_T^* \hat{X}_f$.
% \begin{equation*}
%     X_f = \ftmat_T^* \hat{X}_f
% \end{equation*}
$X_f$ is the equivalent of $\hat{X}_G$ in equation \ref{jgft_einsum} that is the output of \emph{EFT}. 
In practice, the fast Fourier transform is used that can perform the computations in order $\mathcal{O}(T log(T))$. 
% Hence, overall time complexity of the architecture is $O(N + E + NTlogT)$. 
Hence, overall time complexity of the architecture is $O((N+E)T + NTlogT)$.
% TODO: talk of add norm and ff layers; why we need them? for mapping back to some space. prevent overfitting etc.
To map the output back to the original space from the interpolated space we would need further mapping layers. Similar to \citep{FMLPRec}, we use the standard layer normalization (LN) and feedforward (FFN) layers: $X_{F} = \text{LN}\left(\text{LN}\left(X_t + \text{D}(X_f)\right) + \text{D}\left(\text{FFN}\left(\text{LN}\left(X_t + \text{D}(X_f)\right)\right)\right)\right)$, where 
% % \begin{align*}
% %     \Tilde{X_f} &= \text{LayerNorm}(X_t + \text{Dropout}(X_f)) \\
% %     \Tilde{X_{f2}} &= W^f_2 \sigma^f (W^f_1 \Tilde{X_f} + b^f_1) + b^f_2 \\
% %     X_{F} &= \text{LayerNorm}(\Tilde{X_{f}} + \text{Dropout}(\Tilde{X_{f2}})) 
% % \end{align*}
% \begin{align*}
% \small
%     X_{F} &= \text{LN}\left(\text{LN}\left(X_t + \text{D}(X_f)\right) + \text{D}\left(\text{FFN}\left(\text{LN}\left(X_t + \text{D}(X_f)\right)\right)\right)\right) 
% \end{align*}
$W^f_2, W^f_1, b^f_1, b^f_2$ are learnable parameters and D(.) represents dropout. We could stack filter layers with the node embeddings obtained from previous layers as inputs. $X_{F}$ is the final filtered signal that is used in the downstream prediction. For the concerned node $n$ we denote this as $X_{F}^n$. 

% explain what it is we do
% explain why it will help....long term content signals etc.
% how: explain fourier transform, our equations etc.

\subsection{Computational Complexity}
Considering the spectral transform, the exact eigendecomposition of the joint laplacian would take order $\mathcal{O}((NT)^3)$ whereas our method of \eft would take $\mathcal{O}(N^3T + NT \log(T))$. Thus we reduce the complexity from a factor of $T^3$ to $T\log(T)$. This would be beneficial in cases where there are many timesteps considered. For the model at the implementation level 
% , as shown in the architecture of the previous section,
since we have made use of a function approximator that runs in time linear to the number of edges ($\varepsilon$), the time complexity is $\mathcal{O}(\varepsilon + NT\log(T))$.
We have performed a wall clock run time analysis for the training of our method and the results in table \ref{tab:wall_clock_eft} shows that it is comparable to a dynamic graph based baseline (that doesn't use any spectral transform):

\begin{table*}[!ht] 
\centering
\caption{Wall clock running (sec/epoch) time of our and baseline method on SR datasets}\label{tab:wall_clock_eft}
\begin{tabular}{ccc}
\hline
Dataset	& Method	&	Wall clock time (sec/epoch) \\
\hline
Beauty	&	DGSR	&	565 \\
Beauty	&	EFT	&	753 \\
Games	&	DGSR	&	1719 \\
Games	&	EFT	&	2535 \\
CD	&	DGSR	&	5415 \\
CD	&	EFT	&	12637 \\
\hline
% \vspace{-20pt}
\end{tabular}
\end{table*}

% SBR setting graph

%\section{Error bounds between \eft and \aeft}
%Figure \ref{fig:error_bounds} shows the plots of the norm of the difference between \eft and \aeft along with the predicted upper bounds in theorem \ref{theorem_bounds}.
%We show the plots with increasing number of nodes, timesteps and rate of graph evolution. We see that the errors are well below the bounds and are asymptotically diverging. This verifies the theoretical bounds.Note that while the factor of $N^{\frac{3}{2}}T$ may seem large, as the number of nodes and timestamps considered increases, the norm of the difference between the CTDFT and A-CTDFT is bound to increase considering there could be more structural changes in the graph with time and thus in the joint Laplacian. The aim of the bound is to inform that if the rate of evolution of the graph is small the two matrices would be close to each other and in the special case that the graph structure doesn’t change with time the two are the same. Thus the essence of the bound is theoretical and may help a practitioner decide whether the proposed CTDFT is a good approximation of the eigendecomposition of the variational form for a specific application.

% \pagebreak
\newpage

\section{Ablation Study}

\begin{wraptable}{l}{0.60\textwidth}
% \begin{table}[t]
    %\small
	\caption{Ablation study of our model. We report Recall@10 (R@10) and NDCG@10 on Beauty and Games datasets.}
	\label{tab:filter}
 \resizebox{0.58\textwidth}{!}{
	\setlength{\tabcolsep}{1.3mm}{
		\begin{tabular}{l|cc|cc}
			\hline
			 &  \multicolumn{2}{c|}{Beauty} &
			\multicolumn{2}{c}{Games} \\
			\hline
			& R@10 & NDCG@10 & R@10 &NDCG@10 \\
			\hline
			\hline
			EFT & \textbf{53.23} & \textbf{37.10} & \textbf{77.78} & \textbf{58.75}\\
			\hline
			w/o Temporal filter & 52.42 & 36.12 & 76.55 & 56.95\\
			w/o Graph filter & 38.27 & 24.39 & 58.36 & 40.06\\
			\hline
			+High Pass Filter & 47.53 & 31.10 & 76.88 & 57.24\\
			+Low Pass Filter & 52.71 & 36.76 & 77.74 & 58.49\\
                +Band Pass Filter & 52.27 & 36.09 & 76.67 & 56.98\\			
                +Band Stop Filter & 45.34 & 29.09 & 77.63 & 58.42\\
			\hline
		\end{tabular}
  \label{tab:ablation_results}
	}
 }
% \end{table}
\end{wraptable}

\textbf{Component ablation:} In our first ablation study, we study the effect of various modules of the 
\eft architecture by systematically removing model components. Table \ref{tab:ablation_results} summarizes our findings. For example, performance declines when we remove the graph filtering module ("w/o graph filter" in Table \ref{tab:ablation_results}). 
% On the other hand, 
It confirms that the graph filters help to reduce long-range noise and positively impact performance.
%The observed behavior is in line with the observation made by \citep{SRGNN} that 
% a vast number of hops in the graph negatively impact the overall performance. reports overfitting when larger subgraphs are used.
Next, we replace learnable filters of \eft with several static filters such as highpass, bandpass, lowpass, etc. Performance with static filters is less than that of dynamic filters, supporting our choice of having learnable filters in \eft. 
%We further observe that among the static filters, the highest results are obtained in the lowpass setting, indicating the presence of predominantly low-frequency signals on the specific datasets.
% which will be visualized in a later section.

\noindent \textbf{Parameter Selection:} In this experiment, we study the effect of filter and graph construction parameters that will help select optimal parameters for the model. Specifically, we run experiments 
% on two datasets (Games and Beauty) 
for 1) the order of the graph filter and 2) subgraph size, which is the number of hops considered around the given user node for constructing the graph. 
% The results are in Figure \ref{fig:param_selection} (with apt scaling). 
The results are in Figure 7 with apt transform of the 2 metric scales for comprehension. 
% Figure \ref{fig:param_selection} shows the results with appropriate scalings/separation for/between the two metrics for better comprehension. 
% Figure \ref{fig:param_selection} shows the results with appropriate scaling for the two metrics for better comprehension. 
For the filter order, we observe that for both datasets, there exists an optimal filter order at which the best performance is achieved. We observe that increasing filter order further causes overfitting on these datasets. For the subgraph size, we observe an increasing trend in the results, indicating that higher subgraph sizes ($> 1$) benefit the performance over a single hop (which is the sequence itself). 
% This shows that modeling the sequential recommendation as a graph learning problem is helpful over considering only the sequence. 
This shows that modeling the SR as a graph learning problem is helpful over considering only the sequence. 
We conclude that beyond the subgraph size of two, the results saturate for these datasets.
%which could be due to the nature of the datasets containing predominantly low-frequency signals.

\begin{figure}
     \centering
     \begin{subfigure}[b]{0.21\textwidth}
         \centering
         \includegraphics[width=\textwidth]{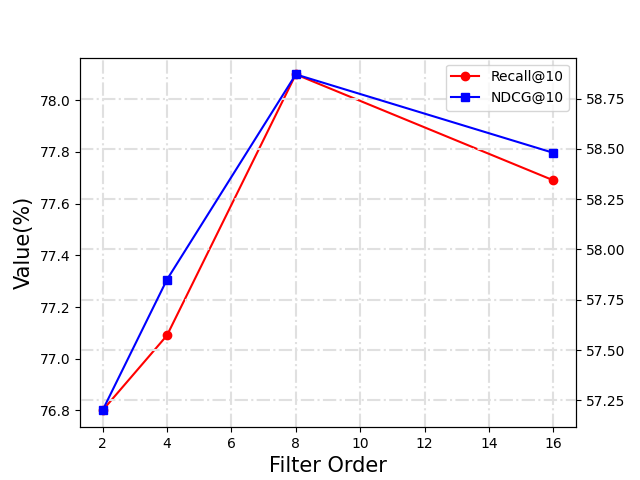}
         \caption{$Games$}
         \label{fig:fo_games}
     \end{subfigure}
     \hfill
     \begin{subfigure}[b]{0.21\textwidth}
         \centering
         \includegraphics[width=\textwidth]{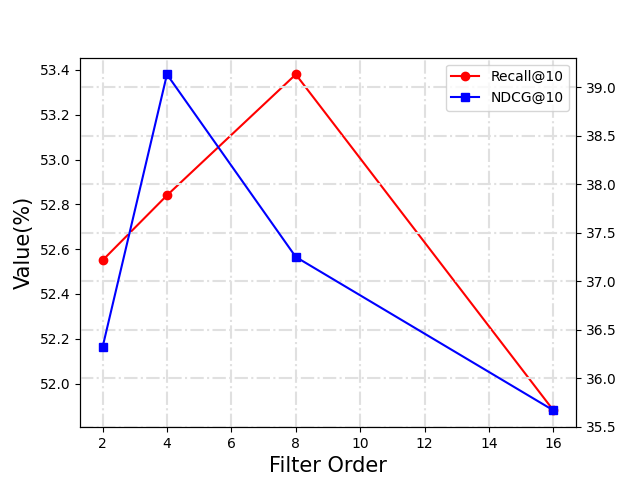}
         \caption{$Beauty$}
         \label{fig:fo_beauty}
     \end{subfigure}
     \hfill
     \begin{subfigure}[b]{0.21\textwidth}
         \centering
         \includegraphics[width=\textwidth]{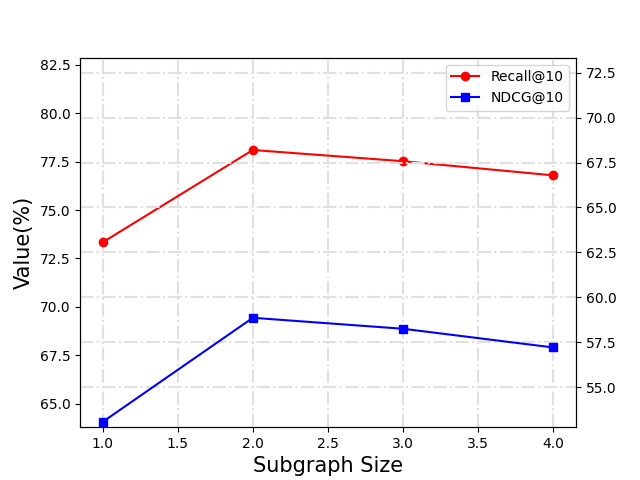}
         \caption{$Games$}
         \label{fig:khop_games}
     \end{subfigure}
     \hfill
     \begin{subfigure}[b]{0.21\textwidth}
         \centering
         \includegraphics[width=\textwidth]{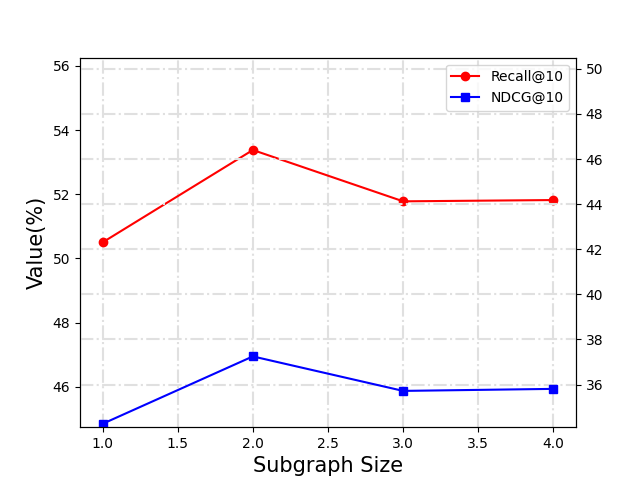}
         \caption{$Beauty$}
         \label{fig:khop_beauty}
     \end{subfigure}
        \caption{Effect of the parameters (filter order and subgraph size) on \eft performance.}
        \label{fig:param_selection}
\end{figure}

\noindent 

% \pagebreak
\newpage

\end{appendix}

\end{document}